\documentclass[12pt]{article} %

\usepackage{etoolbox}
\newtoggle{colt}
\togglefalse{colt}

\iftoggle{colt}{}{
\usepackage[sort&compress,numbers]{natbib}
}

\usepackage{boxedminipage}
\usepackage{multirow,nicefrac}
\usepackage{makecell,upgreek}
\usepackage{footnote}
\usepackage{longtable}
\usepackage{tablefootnote}
\usepackage[T1]{fontenc}
\usepackage{verbatim} 
\usepackage[utf8]{inputenc}

\usepackage{booktabs}   
\usepackage{adjustbox}

\iftoggle{colt}{}{
\usepackage[margin=0.75in]{geometry}
}
\usepackage{algorithm}
\usepackage{algpseudocode}

\usepackage[dvipsnames]{xcolor}

\usepackage{colortbl}
\definecolor{lightgray}{gray}{0.9}  %

\iftoggle{colt}{}{
\usepackage{amsmath}
\usepackage{amsthm}

\usepackage[breaklinks=true]{hyperref}
\hypersetup{
	colorlinks=true,
	linkcolor=blue,
	citecolor=blue,
	urlcolor=blue
}
}

\iftoggle{colt}{}{
  \usepackage[capitalise,nameinlink]{cleveref}
}
\usepackage{amsfonts,amssymb,mathrsfs}
\usepackage{mathtools}

\usepackage{appendix}
\newcommand{\defeq}{:=}

\iftoggle{colt}{
  
}{
\theoremstyle{plain}
	\newtheorem{theorem}{Theorem}

	\newtheorem{lemma}{Lemma}

	\newtheorem{proposition}{Proposition}

	\theoremstyle{definition}

	\newtheorem{definition}{Definition}

	\newtheorem{remark}{Remark}

}

\Crefname{appendix}{Appendix}{Appendices}

\usepackage{autonum}

\usepackage[T1]{fontenc}
\usepackage[scaled]{beramono}
\usepackage{listings}
\definecolor{varorange}{RGB}{230,126,34}     %
\definecolor{opblue}{RGB}{52,152,219}        %
\definecolor{ifred}{RGB}{192,57,43}          %
\definecolor{commentgray}{RGB}{150,150,150}  %
\definecolor{codebg}{rgb}{0.98,0.98,0.98}    %

\newsavebox\codebox

\usepackage{authblk}

\newcommand{\rr}{\mathbb{R}}
\newcommand{\ee}{\mathbb{E}}
\newcommand{\pp}{\mathbb{P}}

\newcommand{\divf}[2]{\mathrm{D}_{f}\left( #1 \middle\| #2 \right)}
\newcommand{\chisq}[2]{\mathrm{D}_{\chi^2}\left( #1 \middle\| #2 \right)}

\DeclareMathOperator*{\argmin}{argmin}
\DeclareMathOperator*{\argmax}{argmax}

\DeclareMathOperator{\reg}{Reg}
\DeclareMathOperator{\Var}{Var}
\DeclareMathOperator{\Cov}{Cov}

\newcommand{\rstar}{r^{\star}}
\newcommand{\TV}{\mathrm{TV}}
\newcommand{\iid}{\overset{\mathrm{iid}}{\sim}}

\newcommand{\epw}{\varepsilon_{\mathrm{pw}}}
\newcommand{\epwnaught}{\varepsilon_{\mathrm{pw}_{0}}}
\newcommand{\ermsq}{\varepsilon^2_{\scriptscriptstyle\mathrm{RM}}}
\newcommand{\indic}{\mathbf{1}}
\newcommand{\sign}{\mathrm{sign}}
\newcommand{\piref}{\pi_{\mathrm{ref}}}
\newcommand{\piEM}{\pi_{M}}
\newcommand{\piEMhat}{\widehat{\pi}_{M}}

\newcommand{\yhat}{\hat{y}}
\newcommand{\rhat}{\widehat{r}}
\newcommand{\pistar}{\pi^\star}
\newcommand{\Rmax}{R_{\mathrm{max}}}
\newcommand{\R}{R}
\newcommand{\Reg}{\mathrm{Reg}}

\newcommand{\pihatbon}{\pihat_{\mathrm{BoN}}}

\newcommand{\pihat}{\widehat\pi}

\newcommand{\YhatN}{\widehat{\cY}_N}

\renewcommand{\Cov}{\mathrm{Cov}}

\newcommand{\epspw}{\epw}
\newcommand{\epssq}{\ermsq}

\renewcommand{\epsilon}{\varepsilon}

\def\ddefloop#1{\ifx\ddefloop#1\else\ddef{#1}\expandafter\ddefloop\fi}
\def\ddef#1{\expandafter\def\csname bb#1\endcsname{\ensuremath{\mathbb{#1}}}}
\ddefloop ABCDEFGHIJKLMNOPQRSTUVWXYZ\ddefloop

\def\ddefloop#1{\ifx\ddefloop#1\else\ddef{#1}\expandafter\ddefloop\fi}
\def\ddef#1{\expandafter\def\csname frak#1\endcsname{\ensuremath{\mathfrak{#1}}}}
\ddefloop ABCDEFGHIJKLMNOPQRSTUVWXYZ\ddefloop

\def\ddefloop#1{\ifx\ddefloop#1\else\ddef{#1}\expandafter\ddefloop\fi}
\def\ddef#1{\expandafter\def\csname fr#1\endcsname{\ensuremath{\mathfrak{#1}}}}
\ddefloop ABCDEFGHIJKLMNOPQRSTUVWXYZ\ddefloop

\def\ddefloop#1{\ifx\ddefloop#1\else\ddef{#1}\expandafter\ddefloop\fi}
\def\ddef#1{\expandafter\def\csname eul#1\endcsname{\ensuremath{\EuScript{#1}}}}
\ddefloop ABCDEFGHIJKLMNOPQRSTUVWXYZ\ddefloop

\def\ddefloop#1{\ifx\ddefloop#1\else\ddef{#1}\expandafter\ddefloop\fi}
\def\ddef#1{\expandafter\def\csname scr#1\endcsname{\ensuremath{\mathscr{#1}}}}
\ddefloop ABCDEFGHIJKLMNOPQRSTUVWXYZ\ddefloop

\def\ddefloop#1{\ifx\ddefloop#1\else\ddef{#1}\expandafter\ddefloop\fi}
\def\ddef#1{\expandafter\def\csname b#1\endcsname{\ensuremath{\mathbf{#1}}}}
\ddefloop ABCDEGHIJKLMNOPQRSTUVWXYZ\ddefloop

\def\ddefloop#1{\ifx\ddefloop#1\else\ddef{#1}\expandafter\ddefloop\fi}
\def\ddef#1{\expandafter\def\csname bhat#1\endcsname{\ensuremath{\hat{\mathbf{#1}}}}}
\ddefloop ABCDEFGHIJKLMNOPQRSTUVWXYZ\ddefloop

\def\ddefloop#1{\ifx\ddefloop#1\else\ddef{#1}\expandafter\ddefloop\fi}
\def\ddef#1{\expandafter\def\csname btil#1\endcsname{\ensuremath{\tilde{\mathbf{#1}}}}}
\ddefloop ABCDEFGHIJKLMNOPQRSTUVWXYZ\ddefloop

\def\ddefloop#1{\ifx\ddefloop#1\else\ddef{#1}\expandafter\ddefloop\fi}
\def\ddef#1{\expandafter\def\csname bst#1\endcsname{\ensuremath{\mathbf{#1}^\star}}}
\ddefloop ABCDEFGHIJKLMNOPQRSTUVWXYZ\ddefloop

\def\ddefloop#1{\ifx\ddefloop#1\else\ddef{#1}\expandafter\ddefloop\fi}
\def\ddef#1{\expandafter\def\csname bst#1\endcsname{\ensuremath{\mathbf{#1}^\star}}}
\ddefloop abcdeghijklmnopqrstuvwxyz\ddefloop

\def\ddefloop#1{\ifx\ddefloop#1\else\ddef{#1}\expandafter\ddefloop\fi}
\def\ddef#1{\expandafter\def\csname bhat#1\endcsname{\ensuremath{\hat{\mathbf{#1}}}}}
\ddefloop abcdefghijklmnopqrstuvwxyz\ddefloop

\def\ddefloop#1{\ifx\ddefloop#1\else\ddef{#1}\expandafter\ddefloop\fi}
\def\ddef#1{\expandafter\def\csname b#1\endcsname{\ensuremath{\mathbf{#1}}}}
\ddefloop abcdeghijklnopqrstuvwxyz\ddefloop

\def\ddefloop#1{\ifx\ddefloop#1\else\ddef{#1}\expandafter\ddefloop\fi}
\def\ddef#1{\expandafter\def\csname barb#1\endcsname{\ensuremath{\bar{\mathbf{#1}}}}}
\ddefloop abcdefghijklmnopqrstuvwxyz\ddefloop

\def\ddef#1{\expandafter\def\csname c#1\endcsname{\ensuremath{\mathcal{#1}}}}
\ddefloop ABCDEFGHIJKLMNOPQRSTUVWXYZ\ddefloop
\def\ddef#1{\expandafter\def\csname h#1\endcsname{\ensuremath{\widehat{#1}}}}
\ddefloop ABCDEFGHIJKLMNOPQRSTUVWXYZ\ddefloop
\def\ddef#1{\expandafter\def\csname hc#1\endcsname{\ensuremath{\widehat{\mathcal{#1}}}}}
\ddefloop ABCDEFGHIJKLMNOPQRSTUVWXYZ\ddefloop
\def\ddef#1{\expandafter\def\csname t#1\endcsname{\ensuremath{\widetilde{#1}}}}
\ddefloop ABCDEFGHIJKLMNOPQRSTUVWXYZ\ddefloop
\def\ddef#1{\expandafter\def\csname tc#1\endcsname{\ensuremath{\widetilde{\mathcal{#1}}}}}
\ddefloop ABCDEFGHIJKLMNOPQRSTUVWXYZ\ddefloop

\usepackage{color-edits}
\addauthor{ab}{red}
\addauthor{vs}{teal}

\title{Revisiting the (Sub)Optimality of Best-of-N for Inference-Time Alignment}
\usepackage{times}

\author[1]{Ved Sriraman\thanks{ved@cs.columbia.edu}}
\author[1,2]{Adam Block\thanks{adam.block@columbia.edu}}

\affil[1]{Department of Computer Science, Columbia University}
\affil[2]{Department of Electrical Engineering, Columbia University}

\date{}

\begin{document}

\maketitle

\begin{abstract}%
    
Best-of-$N$ (BoN) sampling is a widely used inference-time alignment method for language models, whereby $N$ candidate responses are sampled from a reference model and the one with the highest predicted reward according to a learned reward model is selected.  Despite its widespread practical use, recent theoretical work has suggested that it is statistically suboptimal and vulnerable to reward hacking, the process by which models exploit weaknesses in the learned reward model to achieve high estimated reward without genuinely improving performance.
We revisit this question under assumptions that more closely reflect practice than that of prior work.  In particular, in contradistinction to earlier analyses that focused on expected true reward, which may not be meaningful in many practical settings, we investigate how inference-time alignment affects the \emph{win-rate}, a pairwise comparison-based metric more closely aligned with how reward models are trained and evaluated in practice.  We demonstrate that, under minimal conditions on the quality of the reference model and learned reward model, properly tuned BoN is both computationally and statistically optimal in achieving high win-rate, partially explaining its widespread practical success.  Because BoN remains susceptible to reward-hacking in this setting, we propose a simple and practical variant that provably eliminates reward-hacking while maintaining optimal statistical performance.  Finally, we show that prior approaches are provably \emph{suboptimal} when considering win-rate, highlighting the importance of choosing appropriate objectives when analyzing inference-time alignment methods.

\end{abstract}

\section{Introduction}
\label{sec:intro}

\looseness-1
Inference-time compute has emerged as a critical new scaling axis in Language Models (LMs) over the past few years, resulting in major advances in mathematical reasoning, coding ability, and reliability \citep{brown2024large,guo2025deepseek,jaech2024openai}.  Previous works have identified two primary approaches to scaling inference-time compute: (i) \emph{parallel scaling}, where a model samples multiple outputs in parallel and selects the one to return to a user; and (ii) \emph{serial scaling}, where a model generates a much longer output that allows it access to computationally more involved solutions \citep{mirtaheri2025let}.  While both approaches have demonstrated significant empirical success, parallel scaling is particularly appealing due to its simplicity, ease of implementation, and ability to leverage existing models without additional training, with a number of empirical results demonstrating strong performance improvements from parallel scaling alone \citep{brown2024large,gui2024bonbon,huang2025self}.  Additionally, the current approach to training models using Reinforcement Learning to take advantage of serial scaling relies critically on parallel scaling in order to estimate reward signals during the training process \citep{guo2025deepseek}.  Thus, understanding the theoretical underpinnings of parallel scaling is a critical step towards understanding the capabilities of modern LMs.

The critical question that arises in the parallel scaling regime is how best to select the final returned output from a set of $N$ independently sampled candidates.  Initial empirical approaches focused on verifiable settings, where the model's output could be directly scored against a known ground-truth reward $\rstar(x,y)$ depending on a prompt $x$ and candidate response $y$; in this setting, the highest-scoring candidate could be selected \citep{brown2024large}.  While such an approach leads to a strong skyline performance, it is not applicable in the vast majority of real-world settings where no such verification is possible.  Instead, it has become the standard pracatice to use some learned proxy reward $\rhat(x,y)$, again depending on the prompt and candidate response, to select the best candidate \citep{stiennon2020learning,stroebl2024inference,nakano2021webgpt,wang2022self,li2022competition}; such an approach leads to the well-known \emph{Best-of-N} (BoN) algorithm that returns $\yhat_N \in \argmax_{y_i} \rhat(x, y_i)$.  While the BoN algorithm has demonstrated strong empirical performance across a variety of settings and has seen some prior theoretical work \citep{huang2025self,huang2025best}, there remain critical lacunae in our understanding of its behavior.  

In \citet{huang2025best}, the authors provide a theoretical analysis of the BoN algorithm in the setting where the learned reward $\rhat$ is close to the true reward $\rstar$ in terms of mean-squared error under the sampling distribution $\piref(\cdot | x)$ used to generate candidates $y_i$ and demonstrate reasonable performance compared to some `good' language model $\pistar$ as long as $\pistar$ is close to $\piref$ in an appropriate sense.  Critically, however, they demonstrate that the BoN algorithm is \emph{suboptimal} in this setting and that in order to achieve optimal performance, one must instead use a more sophisticated (and complicated) $\chi^2$\emph{-regularized} variant of the BoN algorithm introduced therein.  This suboptimality arises due to another drawback of BoN: when the number of candidates $N$ is large, the algorithm engages in \emph{reward-hacking}, a manifestation of Goodhart's law \citep{goodhart1984monetary} whereby the algorithm selects candidates that achieve high reward $\rhat$ but are in fact low-quality according to the true reward $\rstar$ \citep{gao2023scaling}.  The $\chi^2$-regularized BoN algorithm is able to mitigate this reward-hacking behavior by penalizing candidates that are unlikely under the sampling distribution $\piref$, leading to improved performance.  While an attractive theoretical result, the empirical relevance of this suboptimality is unclear, as optimally tuned BoN algorithms mostly outperformed the more complicated $\chi^2$-regularized variant in practice, either statistically (leading to better performance) or computationally (requiring fewer samples to achieve the same performance).

The gap between theory and practice in \citet{huang2025best} is partially explained by the fact that the authors in that paper recognize that in the special case where rewards are binary, as is often the case in reasoning tasks where outputs are either correct or incorrect, the BoN algorithm is in fact optimal when the number of candidate responses $N$ is properly tuned.\footnote{Technically, in that work the authors allow for comparator policies with \emph{uniformly bounded likelihood ratios}, but such an assumption is unlikely to hold for reasonable $\pistar$ in the absence of binary rewards due to the increasingly long generations from modern LMs.}  On the other hand, in many situations, such as open-ended generation tasks like dialog or creative writing, rewards are not binary but instead take on a wide range of values depending on the quality of the response.  In such settings, it remains unclear whether the BoN algorithm is optimal or suboptimal, and if suboptimal, whether more complicated variants can outperform it in practice.  Moreover, in any setting, it remains unclear whether the mean-squared error metric used in \citet{huang2025best} is the most appropriate way to measure the quality of learned reward models $\rhat$ and what the appropriate notion of `closeness' between the sampling distribution $\piref$ and the target language model $\pistar$ should be.  In addition, the performance metric studied in that work is the expected true reward, which is often not very meaningful in many practical situations, where \emph{win-rate with respect to some comparator policy} (cf. \eqref{eq:win-rate-definition}) is a more natural evaluation \citep{alpaca_eval}.

In this work, we aim to address these open questions by providing a more complete theoretical understanding of the BoN algorithm that we believe better fits the empirical realities in which it is used, as well as present a new algorithm that is both theoretically optimal and provably mitigates reward-hacking.  Our starting point is to observe that neither the mean-squared error metric for reward models nor the metric of expected true reward under the target language model are fully appropriate for understanding the performance of the BoN algorithm.  Indeed, in most situations where rewards for language models are nonbinary, the `ground truth' notion of reward is the \emph{win-rate}, i.e., the proportion of times that some oracular alignment source (either human labellers or a strong LM) prefers the returned response to an offline reference response \citep{alpaca_eval,dubois2023alpacafarm,dubois2024length,ouyang2022training}.  Moreover, most reward models are trained in a similar manner, via pairwise comparisons between candidate responses rather than direct regression against some ground-truth reward value, which is rarely available in practice \citep{rafailov2023direct}.  Indeed, there is a philosophical question to be asked as to the sense in which a ground-truth reward $\rstar$ even exists when there is never direct supervision thereof; most empirical work asserts the Bradley-Terry model \citep{bradley1952rank,ouyang2022training} for practical purposes, but the expected reward remains a somewhat artificial construct.

Furthermore, while natural in the context of learning theory, the use of squared error is somewhat contrived in the context of the BoN algorithm in that it, unlike the algorithm itself, is not \emph{scale-invariant}: multiplying the reward model $\rhat$ by a positive constant does not change the behavior of the BoN algorithm, but it does change the mean-squared error. Instead, a more natural notion of reward mismatch respects the way in which $\rhat$ is used in BoN, i.e., monotone transformations should not affect our quality metric.  Finally, we address the ad hoc nature of the $\chi^2$-regularized BoN algorithm by considering a substantially more general notion of quality, that of $\cE_M$-divergence \citep{polyanskiy2010channel}, which was extensively studied in \citet{block2023sample} and intimately connected to the  recently introduced notion of \emph{coverage} \citep{huang2025self,chen2025coverage}, which  has been fruitfully applied to understand LMs in a variety of settings.\footnote{Formally, coverage is defined as $\Cov_M(\pistar \| \piref) = \pp_{y \sim \pistar}\left( \frac{d\pistar}{d \piref}(y) \geq M \right)$.  It is easy to see that $\cE_M(\pistar \| \piref) = \Cov_M(\pistar \| \piref) - M \cdot \pp_{y \sim \piref}\left( \frac{d \pistar}{d \piref}(y) \geq M \right)$ is bounded above by coverage. }

\looseness-1 We have two main results in this work, the first of which is as follows.
\begin{theorem}[Informal Statement of \Cref{thm:bon-win-rate-upper-bound,thm:skyline-lb}]\label{thm:informal_bon}
    If $\rhat$ is close to $\rstar$ in terms of pairwise win-rate error (cf. \eqref{eq:pairwise-win-rate-error}), then the BoN algorithm using $\rhat$ with tuned $N$ achieves optimal performance.
\end{theorem}
This result demonstrates that when using an appropriate notion of reward model error that is commonly used in practice, the BoN algorithm is in fact optimal, somewhat opposing the practical intuition that arises from \citet{huang2025best}.  Indeed, while not a formal contradiction to the difference in formal settings, our result suggests that the suboptimality identified in \citet{huang2025best} is an artifact of the assumptions therein rather than a fundamental limitation of the BoN algorithm.

On the other hand, the fundamental insight that reward-hacking can occur when $N$ is large remains valid, and thus we further introduce a new \emph{$\cE_M$-regularized} BoN algorithm that directly penalizes candidates that are unlikely under the sampling distribution $\piref$ in terms of the $\cE_M$-divergence.  Our second main result is as follows:
\begin{theorem}[Informal Statement of \Cref{thm:topk-winrate-upper-bound}]\label{thm:informal_coverage_bon}
    If $\rhat$ is close to $\rstar$ in terms of pairwise win-rate error, then the $\cE_M$-regularized BoN algorithm (cf. \eqref{eq:cov-reg-var-prob}) using $\rhat$ with appropriately tuned regularization is both optimal and achieves performance that provably does not decay with $N$.
\end{theorem}
This result demonstrates that by using $\cE_M$-regularization, we can mitigate reward-hacking.  A key advantage of our proposed algorithm over prior approaches to combatting reward overoptimization is the simplicity of implementation, which does not require online estimation, extensive additional training, or significant computational overhead beyond that of the BoN algorithm itself \citep{huang2025best,khalaf2025inference,jinnai2025regularized}.  Moreover, we show in \Cref{prop:chi2-vs-cov-arbitrary-c} that the $\chi^2$-regularized BoN algorithm of \citet{huang2025best} can be arbitrarily worse than our $\cE_M$-regularized BoN algorithm in terms of win-rate regret, demonstrating the insufficiency of prior approaches in this setting.

We now proceed in \Cref{sec:preliminaries} by introducing and motivating our formal setting as well as clarifying how it differs from that considered in prior work, before stating our results for the BoN algorithm in \Cref{sec:BoN}.  We then introduce our coverage-regularized BoN algorithm in \Cref{sec:monotone_algos} before discussing some key proof techniques in \Cref{sec:proof_techniques}.  While in much of the paper we consider a simplified setting where the comparator policy defining win-rate is equal to $\piref$, we extend our results to apply to settings where the win-rate is defined with respect to an arbitrary comparator in \Cref{sec:win_rate_arbit}. We conclude with a discussion of related work and future directions in \Cref{sec:discussion}.

\section{Preliminaries and Problem Setup}
\label{sec:preliminaries}

In this section, we introduce our inference-time alignment framework and formally state the problems we consider.  We first describe the computational model through which we interact with language and reward models, before discussing the appropriate notions of model quality. 
\subsection{Inference-time alignment framework}
\label{sec:inference_time_alignment_framework}
\looseness-1 Our framework is inspired by that of \citet{huang2025self,huang2025best}, although we consider different desiderata.  Throughout, we let $\cY$ denote a space of possible responses and $\cX$ denote a space of possible prompts or contexts.  A \emph{language model} is any map $\pi: \cX \to \Delta(\cY)$ from prompts to distributions over responses; due to the connections with RL, we use the terms language model and policy interchangeably.  For the purposes of analysis, we will suppose that there exists a true reward function $\rstar: \cX \times \cY \to \rr$ that measures the quality of a response $y$ given a prompt $x$; in practice, this reward function will be unknown and only accessible through some learned reward model $\rhat: \cX \times \cY \to \rr$.  For the sake of simplicity, we will always assume that $0 \leq \rstar(x,y), \rhat(x,y) \leq \Rmax$ for some known $\Rmax > 0$.
As in \citet{huang2025best}, we will consider a simplified setting where the prompt $x$ is fixed throughout, and thus we will suppress it from notation for clarity, i.e., we will write $\piref(\cdot)$ instead of $\piref(\cdot | x)$ and similarly for $\rstar$ and $\rhat$.  Note that there is no loss of generality in doing so, as we can take expectations over the marginal distribution of prompts at the end of our analysis if desired. 

We suppose that the learner has access to the following computational model, motivated by the fact that empirical LM inference engines like vLLM \citep{kwon2023efficient} can quickly sample from large language models and evaluate reward models in parallel.
\begin{definition}[Sample-and-Evaluate]\label{def:sample-and-evaluate-framework}
    The learner has access to a \emph{reference model} $\piref$ from which he can sample responses $y \sim \piref$ independently.  In addition, for any sampled response $y$, the learner can query the reward model $\rhat$ to obtain the reward value $\rhat(y)$.  The efficiency of the learner is measured by the total number of samples, $N$, drawn from $\piref$.
\end{definition}
The goal of the learner is to use this sample-and-evaluate access to produce a response $\yhat$ that achieves high true reward $\rstar(\yhat)$ while using as few samples from $\piref$ as possible.  In order to evaluate candidate policies $\pihat$ producing responses $\yhat \sim \pihat$, it is natural to compare them to some \emph{comparator policy} $\pistar$ that achieves high true reward through regret.  While \citet{huang2025best} consider regret in terms of expected true reward, i.e., trying to minimize $\ee_{y \sim \pistar}[\rstar(y)] - \ee_{y \sim\pihat}[\rstar(y)]$, we note that in many practical settings, the true reward is either (i) binary or (ii) accessible only through (potentially noisy) pairwise comparisons.  In the former case, expected true reward drastically simplifies to a search problem and is tightly analyzed in \citet{huang2025best}.  In the latter case, however, the expected numerical value of the reward is less meaningful as the learner will never have access to $\rstar$ itself.  Instead, we consider the \emph{win-rate} induced by $\rstar$ of the learned policy against the comparator policy $\piref$, defined to be:
\begin{align}\label{eq:win-rate-definition}
    R^{\rstar}(\pi) = \pp_{y \sim \pi, y' \sim \piref}(\rstar(y) > \rstar(y')) + \frac 12 \cdot \pp_{y \sim \pi, y' \sim \piref}\left( \rstar(y) = \rstar(y') \right).
\end{align}
More generally, $\rstar$ could be replaced with another reward function, such as $\rhat$, and \eqref{eq:win-rate-definition} would define the win-rate under that reward function. We remark that while \eqref{eq:win-rate-definition} uses the base model from which we draw candidate samples, $\piref$, as the comparator policy, in principle one could evaluate win-rate against an arbitrary comparison distribution $q$.  This latter setting is relevant in practice, as it corresponds to AlpacaEval-style pipelines \citep{dubois2023alpacafarm,dubois2024length} where $q$ corresponds to a strong comparator like GPT-4.  For the sake of simplicity and clarity of presentation, we will first present our results for the special case where $q = \piref$, and then show that for general $q$, under natural closeness assumptions between $q$ and $\piref$, the same arguments yield qualitatively similar results, which we detail in \Cref{sec:win_rate_arbit}. It is immediate that expected rewards coincide (up to a constant factor) with the notion of win rate in the special case where all reward functions are binary. Given a comparator policy $\pistar$, we wish to minimize the \emph{regret} of a candidate policy $\pihat$, i.e., $\reg(\pihat) = R^{\rstar}(\pistar) - R^{\rstar}(\pihat)$.  The fundamental question we ask is:
\begin{quotation}
    \emph{Within the sample-and-evaluate framework, how many samples from $\piref$ are necessary and sufficient to achieve low regret with respect to a given comparator policy $\pistar$ and how should we select which output to return?}
\end{quotation}
In order to answer this question, we have to better understand in what sense the reward model $\rhat$ is close to the true reward $\rstar$ and in what sense the reference policy $\piref$ is a good proxy for the comparator policy $\pistar$.  We discuss these two questions in turn below.

\subsection{In what sense is $\rhat$ close to $\rstar$?}
\label{sec:error-notions}

Our first goal is to formalize the sense in which the learned reward model $\rhat$ is close to the true reward $\rstar$.
Critically, the reward model $\rhat$ is trained on `typical' outputs, which we formalize by assuming that $\rhat$ is close to $\rstar$ on samples drawn from $\piref$.  While \citet{huang2025best} supposed that the \emph{squared error}, $\epssq(\rhat) = \ee_{y \sim \piref}\left[ \left( \rhat(y) - \rstar(y) \right)^2 \right]$ is small, we argue that this is not the most appropriate metric in our setting for two reasons.  First, squared error is not \emph{scale-invariant}, i.e., multiplying $\rhat$ by a positive constant does not change the behavior of any selection algorithm that uses $\rhat$ to select outputs, but it does change the squared error.  Second, assuming small squared error is somewhat unnatural when rewards are only accessible through pairwise comparisons.  Indeed, there is a philosophical point to be made here: if rewards are only accessible through pairwise comparisons, then it is not clear that there is any meaningful ground-truth numerical reward function $\rstar$ at all beyond the ordinal information it induces.
Numerical reward functions are often trained by running Maximum Likelihood Estimation (MLE) on pairwise comparison data assuming the Bradley-Terry model \citep{bradley1952rank,rafailov2023direct}, which is known to suffer sample complexity scaling exponentially in the maximum reward value\footnote{Indeed, it arises naturally as a simple calculation reveals that the maximum inverse Fisher information scales similarly.} \citep{rosset2024direct,xie2024exploratory}. Instead, we propose to measure the quality of $\rhat$ through its \emph{pairwise win-rate error} with respect to $\rstar$:
\iftoggle{colt}{
\begin{align}
\label{eq:pairwise-win-rate-error}
\epspw(\rhat)
\defeq
\ee_{y,y'\sim\piref}\!\left[|
\phi_{\rhat}(y, y')-
\phi_{\rstar}(y, y')|
\right],
\end{align}
where $\phi_r(y, y')\defeq\indic\{r(y)>r(y')\}+\tfrac12\,\indic\{r(y)=r(y')\}$ encodes the outcome of a single pairwise comparison under reward function $r$.
}{
    \begin{align}
\label{eq:pairwise-win-rate-error}
\epspw(\rhat)
\defeq
\ee_{y,y'\sim\piref}\!\left[|
\phi_{\rhat}(y, y')-
\phi_{\rstar}(y, y')|
\right] \quad \text{and} \quad \phi_r(y, y')\defeq\indic\{r(y)>r(y')\}+\tfrac12\,\indic\{r(y)=r(y')\}
\end{align}
}
 Here, $\epspw(\rhat)$ quantifies the amount by which $\rhat$ disagrees with $r^\star$ on the outcome of two `typical' responses sampled  from $\piref$.  In particular, the absolute difference $|
\phi_{\rhat}(y, y')-
\phi_{\rstar}(y, y')|$ vanishes when $\rhat$ and $\rstar$ agree on the pairwise comparison outcome between $y$ and $y'$, is $\nicefrac{1}{2}$ when one of the reward functions has a tie and the other does not, and penalizes disagreements between $\rhat$ and $\rstar$ that flip the win-versus-loss outcome with a full penalty of $1$.    This definition of $\epspw(\rhat)$ is thus closely aligned wth the notion of win rate, which rewards strict wins more than ties. 
This alignment with the win-rate functional (which also half-counts ties and ranges between $0$ and $1$) hearkens back to our discussion on win-rate in \Cref{sec:intro} and will prove to be important in our later analyses.

This metric is clearly scale-invariant and more closely matches the way reward models are trained in practice \citep{ouyang2022training,rafailov2023direct,lambert2024tulu}, suggesting that \eqref{eq:pairwise-win-rate-error} may be a more reasonable quantity to expect controlled than $\epssq$.  In the Bradley-Terry model, an elementary analysis shows that the tightest relationship between the two quantities scales exponentially in the maximum reward value, suggesting that when learning with finitely many samples from $\piref$, we may expect $\epspw(\rhat) \ll \epssq(\rhat)$.  In general, however, $\epspw(\rhat)$ is \emph{incomparable} to $\epssq(\rhat)$ as we show in \Cref{sec:appendix_diff_notions_of_coverage}.  Indeed, the scale invariance of $\epspw$ immediately implies that there are examples where $\epspw = 0$ but $\epssq$ is arbitrary.  Another sense in which requiring $\epssq(\rhat)$ to be small is a stronger assumption than requiring $\epspw(\rhat)$ to be small is that we cannot recover small regret in terms of expected true reward from small $\epspw(\rhat)$ (cf. \Cref{prop:forallA-EM-impossibility} in \Cref{sec:appendix_diff_notions_of_coverage}), in marked contrast to small $\epspw(\rhat)$ \citep{huang2025best}; as we discussed above, in many situations, winrate is the more natural notion of reward.

\subsection{In what sense is $\piref$ a good reference model?}
\label{sec:pw-error}

In addition to the quality of $\rhat$ as a proxy for $\rstar$, we must also consider the quality of the $\piref$.  Intuitively, if $\piref$ places much less mass on high-reward outputs than the comparator policy $\pistar$, then $N$ must be very large in order to obtain comparable performance to the comparator.  Thus, any notion of sample complexity must depend on some measure of discrepancy between $\pistar$ and $\piref$.  
While prior work has supposed that $\chisq{\pistar}{\piref}$ is small, where we recall that $\chisq{\nu}{\mu} = \ee\left[ \left( \nicefrac{d\nu}{d\mu} - 1 \right)^2 \right]$ is the $\chi^2$-divergence between two distributions $\nu$ and $\mu$, this assumption is somewhat ad hoc.  Indeed, the requirement of small $\chisq{\pistar}{\piref}$ falls directly out of the assumption that $\epssq(\rhat)$ is small due to the duality between $\chi^2$-divergence and squared error; thus while natural in that context, there is no \emph{a priori} reason to prefer $\chi^2$-divergence over any other discrepancy measure.  We instead propose to measure this discrepancy through the fundamental notion of $\cE_M$-divergence \citep{polyanskiy2010channel,block2023sample}, where for any $M \geq 0$, we let
\begin{align}\label{eq:em-divergence-definition}
    \cE_M(\pistar \| \piref) = \ee_{y \sim \piref}\left[ \left( \nicefrac{d\pistar}{d\piref}(y) - M \right)_+ \right]. %
\end{align}
While the $\cE_M$-divergence has many applications \citep{polyanskiy2010channel,polyanskiy2025information}, its relevance here stems from the fact that \citet{block2023sample} demonstrated its intimate connection to the problem of approximate sampling: in order to generate samples from a target distribution $\pistar$ using only samples from a reference distribution $\piref$ within tolerance (in total variation distance) $\epsilon$, any algorithm requires at least $M$ samples from $\piref$, where $\cE_M(\pistar \| \piref) \leq \epsilon$.  Such a characterization is natural in the sample-and-evaluate framework of \Cref{def:sample-and-evaluate-framework} because a natural skyline for our setting would be to approximately sample from $\pistar$ using samples from $\piref$ and then return the sampled output, if the learner had access to knowledge of $\pistar$ beyond that furnished by $\rhat$.  Moreover, the $\cE_M$-divergence is intimately related to the recently introduced notion of \emph{coverage} \citep{chen2025coverage}, which has been fruitfully applied to understand LMs in a variety of settings.  Coverage measures the tail discrepancy between two distributions and is defined for $M \geq 0$ to be \iftoggle{colt}{
$\Cov_M(\pistar \| \piref) = \pp_{y \sim \pistar} \left( \nicefrac{d \pistar}{d \piref}(y) \geq M \right)$; it
}{
\begin{align}
\Cov_M(\pistar \| \piref) = \pp_{y \sim \pistar} \left( \nicefrac{d \pistar}{d \piref}(y) \geq M \right).
\end{align}
It} is straightforward to see that $\cE_M(\pistar \| \piref) \leq \Cov_M(\pistar \| \piref)$ for all $M \geq 0$.  More generally, $\cE_M$-divergence is a special case of $f$-divergences \citep{polyanskiy2025information}, defined for convex $f: \rr_{\geq 0} \to \rr_{\geq 0}$ with $f(1) = f'(1) = 0$ to be 
\iftoggle{colt}{
$\divf{\nu}{\mu} = \ee_{y \sim \mu}\left[ f\left( \nicefrac{d\nu}{d\mu}(y) \right) \right]$.
}{
\begin{align}
\divf{\nu}{\mu} = \ee_{y \sim \mu}\left[ f\left( \nicefrac{d\nu}{d\mu}(y) \right) \right].
\end{align}
} This family includes many well-known divergences, such as KL-divergence, Total Variation distance, and $\chi^2$-divergence and a key result of \citet{block2023sample} shows that for any such $f$-divergence, we have 
$\cE_M(\pistar \| \piref) \leq \nicefrac{M \cdot \divf{\pistar}{\piref}}{f(M)}$.

With these notions of reward model and reference model quality in place, we may now refine our central question to be: 
\iftoggle{colt}{
}{
\begin{quotation}
\emph{Within the sample-and-evaluate framework, how many samples from $\piref$ are necessary and sufficient to achieve low regret with respect to a given comparator policy $\pistar$ as a function of the pairwise win-rate error $\epspw(\rhat)$ and $\cE_M(\pistar \| \piref)$?}
\end{quotation}
}

\section{Optimal Performance with Best-of-$N$}
\label{sec:BoN}

In this section we begin answering the question posed at the end of \Cref{sec:preliminaries} by providing upper bounds on the win-rate regret achieved by the most na{\"i}ve algorithm in the sample-and-evaluate framework, Best-of-$N$ (BoN):
\begin{align}\label{eq:bon}
    \pihatbon = \argmax_{y \in \{y_1, \ldots, y_N\}} \rhat(y), \quad y_i \sim \piref.
\end{align}
The BoN algorithm in \eqref{eq:bon}, introduced in \citep{stiennon2020learning}, simply pretends that $\rstar = \rhat$ and returns the sample with the highest predicted reward among $N$ i.i.d.\ draws from the reference policy $\piref$.  While BoN is susceptible to overoptimization, due to the fact that $\rhat$ may not approximate $\rstar$ on typical outputs drawn from $\pihatbon$ even if it does so on typical outputs drawn from $\piref$, the algorithm has seen widespread practical use. We have the following upper bound on the performance.
\begin{theorem}
\label{thm:bon-win-rate-upper-bound}
For any $N \geq 1$, reference model $\piref$, comparator $\pistar\ll\piref$, ground truth reward $\rstar$, and learned reward model $\rhat$, it holds that 
\begin{align}\label{eq:bon_ub}
\R^{\rstar}(\pistar)-\R^{\rstar}(\pihatbon)
\lesssim
N \cdot \epspw(\rhat) \cdot \log\left( \nicefrac{1}{\epspw(\rhat)} \right) + \cE_{N / \log(1/\epspw(\rhat))}\left( \pistar \| \piref \right).
\end{align}
\end{theorem}
There are two terms in the upper bound \eqref{eq:bon_ub}.  The first term quantifies the degree to which reward hacking occurs: as $N$ increases, the algorithm is more likely to select outputs on which $\rhat$ and $\rstar$ disagree.  The second term quantifies the extent to which, even if $\rhat = \rstar$, the BoN algorithm is able to find high-reward outputs comparable to those of the comparator $\pistar$ using only samples from $\piref$; here the $\cE_N$-divergence naturally arises as the relevant measure of discrepancy between $\pistar$ and $\piref$ due to the connection between $\cE_N$ and approximate rejection sampling outlined in \citet{block2023sample}.  Applying a key lemma from that work relating $\cE_M$-divergence to the $f$-divergence (cf. \Cref{sec:preliminaries}), we may further bound \eqref{eq:bon_ub} by 
$N \cdot \epspw(\rhat) + \nicefrac{N \cdot \divf{\pistar}{\piref}}{f(N)}$; in the special case of $\chi^2$-divergence, tuning $N$, we get (up to a log factor) an optimal rate of $\sqrt{\epspw(\rhat) \cdot \chi^2(\pistar\|\piref)}$ of the regret, which bears at least cosmetic similarity to the results of \citet{huang2025best}, although we emphasize that here we are considering win rate as opposed to expected reward.  In the special case of binary rewards, however, expected reward becomes a special case of win rate (cf. \Cref{sec:preliminaries}), and thus the above upper bound recovers the result of \citet{huang2025best} in this setting.

In contradistinction to the case of expected reward, where BoN's performance is \emph{suboptimal} statistically, we find that BoN is in fact \emph{optimal} for win-rate regret up to constant factors, as established by the following lower bound.
\begin{theorem}
\label{thm:skyline-lb}
For any non-atomic $\piref$, any $\pihat_\cA$ in the sample-and-evaluate framework, and any $\epsilon < \nicefrac 12$, there exist reward functions $\rhat$, $\rstar$ satisfying $\epspw(\rhat) \leq \epsilon$ and a comparator policy $\pistar\ll\piref$ for which
\begin{align}
\label{eq:skyline-lb-statement}
    R^{\rstar}(\pistar) - R^{\rstar}(\pihat_\cA) \gtrsim \inf_{M} \left\{ M \cdot \epsilon + \cE_M(\pistar \| \piref) \right\}.
\end{align}
\end{theorem}

This result shows that unlike in the expected-reward setting, BoN is in fact optimal up to a logarithmic factor for win-rate regret within the sample-and-evaluate framework, which helps explain its widespread practical use despite its simplicity.  This observation gently pushes back against the conclusion reached by \citet{huang2025best} that observed that, when optimizing for expected reward, BoN is statistically suboptimal and more sophisticated algorithms are necessary to achieve optimal performance.  As we have emphasized repeatedly, however, we believe that win rate is often a more natural objective than expected reward in many practical settings, and thus the suggestion that BoN should be avoided for being statistically suboptimal may be overstated.  In addition to the statistical lower bound above, we also provide a complementary computational lower bound.
\begin{proposition}
\label{prop:query-lb-hit-topset}
For any $M \geq 2$ and $\delta < \nicefrac{1}{2}$, there exist reward functions $\rstar =\rhat$ and policies $\pistar, \piref$ with $\cE_M(\pistar \| \piref) = 0$ such that within the sample-and-evaluate framework, any policy $\pihat$ achieving regret $\R^{\rstar}(\pistar) - \R^{\rstar}(\pihat) \leq \delta$
must use at least $N = \widetilde{\Omega}(M)$ samples from $\piref$.
\end{proposition}
This bound is \emph{computational} in the sense that if we did not care about computational efficiency (quantified by $N$ in the sample-and-evaluate framework), then it is trivial to achieve small regret with $\pihatbon$, as $\rhat = \rstar$; on the other hand, if we want to achieve small regret, then we have a strong lower bound on the number of samples required.  We emphasize that this lower bound matches the upper bound of \Cref{thm:bon-win-rate-upper-bound} and thus BoN is computationally and statistically optimal for win-rate regret within the sample-and-evaluate framework.

\section{Achieving Optimal Performance without Reward-Hacking}
\label{sec:monotone_algos}

While in the previous section we saw that the na{\"i}ve Best-of-$N$ algorithm given in \eqref{eq:bon} is computationally and statistically optimal for win-rate regret with optimally tuned $N$, it is still susceptible to reward-hacking. That is, as $N$ increases, the algorithm is more likely to select outputs on which $\rhat$ and $\rstar$ disagree, leading to performance that scales non-monotonically in $N$.  While understandable, this is an unfortunate property, as in practice one would like to be able to increase $N$ to improve performance without worrying about overfitting to the reward proxy $\rhat$.  In this section, we present a new algorithm, \emph{$\cE_M$-regularized Best-of-$N$}, that is able to mitigate reward-hacking and achieve performance that is monotone in $N$ while still being statistically and computationally optimal.

Intuitively, we would like a policy that returns high-reward outputs without straying too far from the reference policy $\piref$, as outputs that are unlikely under $\piref$ are more likely to be subject to reward-hacking.  Our analysis and results from the previous section suggest that the appropriate way to measure discrepancy from $\piref$ is via the $\cE_M$-divergence for an appropriate choice of $M$, defined in \eqref{eq:em-divergence-definition}.  Motivated by this observation, we consider the following variational problem:
\begin{align}\label{eq:cov-reg-var-prob}
    \pi_M \in \argmax_{\pi} \ee_{\pi}\left[ \rhat(y) \right] - \Rmax \cdot \cE_M(\pi \| \piref),
\end{align}
whose solution we refer to as the \emph{$\cE_M$-regularized Best-of-$N$ policy}.  Note that it is perhaps not immediate that \eqref{eq:cov-reg-var-prob} is a desirable objective to consider; indeed, while a similar objective, with $\cE_M$ replaced by $\chi^2$, was considered by \citet{huang2025best} in the context of expected reward maximization, in our setting we do not care about the expected reward per se, but rather the win-rate.  Another potential concern is that the variational problem in \eqref{eq:cov-reg-var-prob} is not easily sampled from within the sample-and-evaluate framework, as it is not immediately clear how to draw samples from the optimal policy $\pi_M$.  The second concern is allayed, however, by the following result, which shows that the optimal policy for \eqref{eq:cov-reg-var-prob} takes a particularly simple form.
\begin{lemma}\label{lem:cov-reg-top-quantile}
    Let $\lambda$ be the $(1 - \nicefrac{1}{M})$-quantile of $\rhat$ under $\piref$, i.e. any value satisfying $\pp\left( \rhat(y) \geq \lambda \right) = \nicefrac 1M$.  Then the optimal policy $\pi_M$ solving \eqref{eq:cov-reg-var-prob} is given by $\pi_M(\cdot) = \piref(\cdot \mid \rhat(\cdot) \geq \lambda)$.  
\end{lemma}
\begin{remark}
One might instead consider a different variation of \eqref{eq:cov-reg-var-prob}, where $\Rmax$ is replaced by a general scaling parameter $\beta > 0$.  If $\beta \geq \Rmax$ then this will end up having the same solution, but if $\beta < \Rmax$ then the solution can grow more complicated.
\end{remark} 
\Cref{lem:cov-reg-top-quantile} is proved in \Cref{app:kkt-derivation}, where we use Lagrange multipliers and convexity to derive a remarkably simple characterization of the optimal policy.  Note that $\pi_M$ can be approximately sampled from within the sample-and-evaluate framework by drawing $N$ samples from $\piref$ and returning a sample uniform from the top-$\lceil\nicefrac NM \rceil$ samples according to $\rhat$ with randomized tie-breaking (cf. \Cref{app:monotone_algos}); indeed, the total variation distance between this sampling procedure and $\pi_M$ is bounded by $\nicefrac 1N$ and $\piEMhat$ can be interpreted as doing BoN for $N \lesssim M$, but maintaining monotonicity as $N$ grows. We should emphasize that, in contrast to the procedure of \citet{huang2025best}, $\piEMhat$ is much simpler and does not require any additional estimation or rejection sampling. We have the following bound on the performance of this variant, which we call $\pihat_M$. 
\begin{theorem}
\label{thm:topk-winrate-upper-bound}
For any $M, N \geq 1$, reference model $\piref$, comparator $\pistar\ll\piref$, ground truth reward $\rstar$, and learned reward model $\rhat$, it holds that
\begin{align}\label{eq:cov-reg-winrate-ub}
    R^{\rstar}(\pistar) - R^{\rstar}(\piEMhat) \lesssim \cE_M(\pistar \| \piref) + M \cdot \epw + \frac{1}{N},
\end{align}
\end{theorem}
Note that the upper bound in \eqref{eq:cov-reg-winrate-ub} is very similar to that of BoN in \Cref{thm:bon-win-rate-upper-bound}, with the key difference being that the scaling parameter $N$ is now decoupled from the regularization parameter $M$.  As a result, by tuning $M$ appropriately, we can achieve the performance skyline established in \Cref{thm:bon-win-rate-upper-bound} while maintaining monotonicity in $N$.  Other than monotonicity, the monotone algorithm performs similarly to BoN and we note that it still requires tuning of the regularization parameter $M$, so we expect BoN to remain the dominant alignment algorithm when monotonicity is not required.

We conclude this section by observing that the only prior provably monotone algorithm, the $\chi^2$-regularized Best-of-$N$ algorithm proposed by \citet{huang2025best}, can fail to achieve optimal performance when win-rate regret is the goal.  Specifically, we have the following result.
\begin{proposition}[Separation between $\pi^\chi_\beta$ vs $\piEM$]
\label{prop:chi2-vs-cov-arbitrary-c}
For every $c > 1$, there exist policies $\piref, \pistar$, reward functions $\rhat, \rstar$, and $\epspw \ll 1$ such that $\epspw(\rhat) \leq \epspw$ and it holds that
\begin{align}
    \inf_{\beta > 0} R^{\rstar}(\pistar) - R^{\rstar}(\pi^\chi_\beta) \geq c \cdot \left(\inf_{M > 0} R^{\rstar}(\pistar) - R^{\rstar}(\piEM)\right).
\end{align}
\end{proposition}
This result, whose proof is deferred to \Cref{sec:separation}, demonstrates that the $\chi^2$-regularized Best-of-$N$ algorithm can perform much worse than our proposed algorithm.  In addition to the poor performance demonstrated in \Cref{prop:chi2-vs-cov-arbitrary-c}, we emphasize that one of the chief advantages of the algorithm proposed in this section over that of $\pi^\chi_\beta$ and alternatives such as soft-BoN \citep{khalaf2025inference,verdun2025soft}, running BoN with pessimistic reward estimates, and alternatives \citep{jinnai2025regularized} is in its simplicity: the optimal policy takes the form of a simple top-quantile selector, which is easy to implement and sample from in practice and requires no online estimation or additional computational overhead.

\section{Proof Techniques}\label{sec:proof_techniques}
Next, we provide proof sketches of our main results on the performance of BoN under win-rate regret (\Cref{thm:bon-win-rate-upper-bound,thm:skyline-lb}) and an upper bound for the $\cE_M$-regularized algorithm (\Cref{thm:topk-winrate-upper-bound}).
\vspace{-1em}
\subsection{Proof Sketch of \Cref{thm:bon-win-rate-upper-bound}}
\label{sec:pf-sketch-bon}

In order to compare the performance of $\pihatbon$ to that of an arbitrary comparator policy $\pistar$ under the true reward function $\rstar$, we need to overcome two main challenges: (1) if $N$ is too small and $\pistar$ is far from $\piref$, then we could pay in regret for not being able to find high-reward outputs comparable to those of $\pistar$; and (2) if $N$ is too large, then we may overoptimize $\rhat$ and select outputs on which $\rhat$ and $\rstar$ disagree.  This tension is captured by the two terms in the upper bound \eqref{eq:bon_ub}, which arise out of the following regret decomposition holding for any $M > 0$: the regret $\R^{\rstar}(\pistar)-\R^{\rstar}(\pihatbon)$ is equal to the sum of three terms,
\begin{align}\label{eq:bon-regret-decomp-body}
 \underbrace{\R^{\rstar}(\pistar)-\R^{\rhat}(\pistar_M)}_{\text{(i)}}  + \underbrace{\R^{\rhat}(\pistar_{M})-\R^{\rhat}(\pihatbon)}_{\text{(ii)}} + \underbrace{\R^{\rhat}(\pihatbon)-\R^{\rstar}(\pihatbon)}_{\text{(iii)}},
\end{align}
where $\pistar_M$ is an intermediate comparator policy defined to be the policy that minimizes total variation distance to $\pistar$ subject to the constraint that its density ratio relative to $\piref$ is upper bounded by $M$ almost surely.  While a similar decomposition was introduced in \citet{huang2025best}, the key challenge in our setting is that we are considering win-rate regret rather than expected-reward regret, which requires different techniques to control each of the terms in \eqref{eq:bon-regret-decomp-body}. 

\paragraph{Term (i).}
We further decompose 
\iftoggle{colt}{
$\R^{\rstar}(\pistar)-\R^{\rhat}(\pistar_M)
= \R^{\rstar}(\pistar)-\R^{\rstar}(\pistar_M)
+ \R^{\rstar}(\pistar_M)-\R^{\rhat}(\pistar_M)$
}{
\begin{align}
\R^{\rstar}(\pistar)-\R^{\rhat}(\pistar_M)
&= \R^{\rstar}(\pistar)-\R^{\rstar}(\pistar_M) + \R^{\rstar}(\pistar_M)-\R^{\rhat}(\pistar_M)
\end{align}
} and control each of the two resulting differences separately.  For the first difference, we observe that win-rate is bounded and thus $1$-Lipschitz in total variation distance, so we can control the loss from $\pistar$ to $\pistar_M$ in terms of $\cE_M(\pistar \| \piref)$ by \citet{block2023sample}.  For the second difference, we use the fact that the likelihood ratio of $\pistar_M$ to $\piref$ is upper bounded by $M$ almost surely to control the pairwise win-rate error under $\pistar_M$ between $\rhat$ and $\rstar$ by $M \cdot \epw(\rhat)$ using change-of-measure.

\paragraph{Term (ii).}
Here, we again adopt the formalism of approximate rejection sampling from \citet{block2023sample} to define a \emph{selection rule} to be a policy $\pistar_{M,\mathrm{rej}}$ that approximates $\pistar_M$ using only $N$ i.i.d.\ draws from $\piref$.  We again further decompose term (ii) as
\iftoggle{colt}{
$\R^{\rhat}(\pistar_{M})-\R^{\rhat}(\pihatbon)=\R^{\rhat}(\pistar_{M})-\R^{\rhat}(\pistar_{M,\mathrm{rej}}) + \R^{\rhat}(\pistar_{M,\mathrm{rej}})-\R^{\rhat}(\pihatbon)$
}{
\begin{align}
\R^{\rhat}(\pistar_{M})-\R^{\rhat}(\pihatbon)=\R^{\rhat}(\pistar_{M})-\R^{\rhat}(\pistar_{M,\mathrm{rej}}) + \R^{\rhat}(\pistar_{M,\mathrm{rej}})-\R^{\rhat}(\pihatbon)
\end{align}
}
 and control each of the two resulting terms separately.  For the first difference, we use the approximate rejection sampling analysis of \citet{block2023sample} to show that $\pistar_{M,\mathrm{rej}}$ approximates $\pistar_M$ in total variation distance up to an exponentially small term in $N/M$, which again transfers to win-rate by boundedness.  For the second term, we observe that conditional on any fixed batch of $N$ samples from $\piref$, BoN simply selects the output with the highest $\rhat$ value, and thus dominates any other selection rule on that same batch under $\rhat$, including $\pistar_{M,\mathrm{rej}}$; while this fact is intuitively true, we prove it rigorously in \Cref{app:bon_proofs}.

\paragraph{Term (iii).}  Finally, we control term (iii) using a similar change-of-measure argument as for term (i), observing that BoN has density ratio at most $N$ relative to $\piref$ almost surely and thus the pairwise win rate error under $\pihatbon$ between $\rhat$ and $\rstar$ is at most $N \cdot \epw(\rhat)$.

\paragraph{Conclusion.}  Combining the bounds for terms (i), (ii), and (iii) yields, for any $M > 0$,
\begin{align}
\R^{\rstar}(\pistar)-\R^{\rstar}(\pihatbon) \lesssim (M + N) \cdot \epw(\rhat) + \cE_M(\pistar \| \piref) + \exp(-N/M).
\end{align}
Setting $M = N / \log(1/\epspw(\rhat))$ then gives the desired upper bound.

\subsection{Proof Sketch of \Cref{thm:skyline-lb}}

Prior lower bounds for expected-reward regret \citep{huang2025best} do not directly apply to win-rate regret, as they rely on constructing reward functions that differ on a small set of outputs with low probability under $\piref$ but high expected reward, which does not affect win rate significantly.  Indeed, it is precisely because such constructions do not seriously affect win rate that BoN is able to achieve optimal performance for win-rate regret, in contrast to expected-reward regret.  For fixed $\pihat$, $\piref$, and $\rhat$, we construct a family of $\rstar_A$ parameterized by measurable subsets $A \subseteq \cY$ such that $\rstar_A$ differs from $\rhat$ only on $A$, i.e., $\rstar(y) = \rhat(y) + \bbI{y \in A}$. A computation shows that 
\iftoggle{colt}{
$\epspw(\rhat) \leq \piref(A) \cdot (1 - \piref(A)) \leq \piref(A),$
}{
\begin{align}
\epspw(\rhat) \leq \piref(A) \cdot (1 - \piref(A)) \leq \piref(A),
\end{align}}
so by choosing $A$ to have $\piref$-measure at most $\epsilon$, we ensure that the pairwise win rate error is small. A calculation tells us that with an appropriate choice of $\rhat$, we can lower bound 
\iftoggle{colt}{
$R^{\rstar}(\pistar) - R^{\rstar}(\pihat) \gtrsim \pistar(A) - \pihat(A)$
}{
\begin{align}
R^{\rstar}(\pistar) - R^{\rstar}(\pihat) \gtrsim \pistar(A) - \pihat(A)
\end{align}}
for any $\rstar_A$ in the family.  Thus, the regret is large if, subject to $\piref(A)$ being small, we can ensure that $\pistar(A)$ is large while $\pihat(A)$ is small.  We apply Markov's inequality to show that
\iftoggle{colt}{
$\pihat \left( \left\{ \nicefrac{d \pihat}{d \piref} \leq 2 \right\} \right) \geq \nicefrac 12$
}{
\begin{align}
\pihat \left( \frac{d \pihat}{d \piref} \leq 2  \right) \geq \frac 12
\end{align}
} and then use the non-atomicity of $\piref$ to demonstrate the existence of an $A$ such that $\piref(A) \leq \epsilon$ and $\pihat(A) \leq 2\epsilon$.  We then choose $\pistar$ such that $\nicefrac{d \pistar}{d \piref} = M$ on $A$ and some small constant elsewhere to ensure that $\pistar$ is a valid probability measure.  For an appropriate choice of $M$, we can use this instance to establish the lower bound.

\subsection{Proof Sketch of \Cref{thm:topk-winrate-upper-bound}}
\label{sec:pf-sketch-topk}

We begin in a similar way as the proof of \Cref{thm:bon-win-rate-upper-bound} by considering a regret decomposition, i.e., the regret $R^{\rstar}(\pistar)-R^{\rstar}(\piEMhat)$ is equal to
\begin{align}
\label{eq:topk-regret-decomp-sketch}
\R^{\rstar}(\pistar)-\R^{\rstar}(\piEMhat)
&=
\underbrace{\big(\R^{\rstar}(\pistar)-\R^{\rhat}(\pistar_M)\big)}_{\text{(i)}}
+
\underbrace{\big(\R^{\rhat}(\pistar_M)-\R^{\rhat}(\piEMhat)\big)}_{\text{(ii)}}
+
\underbrace{\big(\R^{\rhat}(\piEMhat)-\R^{\rstar}(\piEMhat)\big)}_{\text{(iii)}},
\end{align}
where $\pistar_M$ is the same as that defined in \Cref{sec:pf-sketch-bon}.  As such, term (i) is handled exactly as in the BoN proof, contributing $\cE_M(\pistar \| \piref) + M \cdot \epspw(\rhat)$.  Similarly, term (iii) is also handled in the same way as in the BoN proof, except we have that $\piEMhat$ has density ratio at most $M$ relative to $\piref$ almost surely (not increasing with $N$, in contradistinction to BoN).  It remains to control term (ii). We now decompose 
\iftoggle{colt}{
$\R^{\rhat}(\pistar_M)-\R^{\rhat}(\piEMhat) = \R^{\rhat}(\pistar_M)-\R^{\rhat}(\pi_M) + \R^{\rhat}(\pi_M)-\R^{\rhat}(\piEMhat)$
}{
\begin{align}
\R^{\rhat}(\pistar_M)-\R^{\rhat}(\piEMhat) = \R^{\rhat}(\pistar_M)-\R^{\rhat}(\pi_M) + \R^{\rhat}(\pi_M)-\R^{\rhat}(\piEMhat),
\end{align}
}
where $\pi_M$ samples uniformly from the top $\nicefrac 1M$ quantile of $\rhat(y)$ for $y \sim \piref$.  The first difference is non-positive by a similar argument as in the BoN analysis (cf. \Cref{app:topk-wr-pf}). To bound the second difference, we compute the win-rate of the empirical quantile policy $\piEMhat$ and the population quantile policy $\pi_M$ explicitly: after tie-breaking, the $\rhat$-ranks of the $N$ samples are i.i.d. $\mathrm{Unif}[0,1]$, so the win-rate of $\piEMhat$ admits an explicit order statistic formula, yielding $R^{\rhat}(\piEMhat) \geq 1 - \nicefrac{1}{2M} - \nicefrac{1}{N}$, while $R^{\rhat}(\pi_M) = 1 - \nicefrac{1}{2M}$, so the difference is at most $\nicefrac{1}{N}$. The result follows.

\section{Performance Under Win-rate Against an Arbitrary Policy}\label{sec:win_rate_arbit}
In this section, we extend our performance guarantees for BoN and $\cE_M$-regularized BoN to the setting where win-rate is evaluated against an \emph{arbitrary} comparison policy $q$ (rather than the baseline $q=\piref$ used in \Cref{eq:win-rate-definition}) as discussed in \Cref{sec:inference_time_alignment_framework}. This generality aligns with common evaluation pipelines—for example, in AlpacaEval-style benchmarks~\citep{dubois2023alpacafarm,dubois2024length} one often measures win-rate against a strong comparator like GPT-4. Formally, for any comparison measure $q$ on $\cY$ and any reward function $r$, we define the win-rate against $q$ by
\begin{align}
\R^{r}_{q}(\pi)
\defeq
\ee_{y\sim \pi,\;y'\sim q}\!\left[\phi_{r}(y,y')\right]
\quad \text{and} \quad
\phi_{r}(y,y')
\defeq
\indic\{r(y)>r(y')\}+\tfrac12\,\indic\{r(y)=r(y')\}.
\end{align}
This definition generalizes the win-rate in \Cref{eq:win-rate-definition} (which corresponds to $q=\piref$) by drawing $y'\sim q$ as the baseline comparator. Throughout this section we assume that $q$ is dominated by $\piref$ ($q \ll \piref$) with a pointwise density-ratio bound 
\begin{align} \label{eq:q-domination-body} w_q(y)\defeq \frac{dq}{d\piref}(y)\le L \quad \piref\text{-a.s.}, \end{align} 
for some constant $L\geq1$. This condition lets us relate win-rate against $q$ to quantities (such as pairwise error and excess mass) that are defined with respect to $\piref$ via the same change-of-measure steps as in \Cref{sec:pf-sketch-bon,sec:pf-sketch-topk}, up to a multiplicative factor $L$. One could instead work with a weaker divergence control, e.g.\ $\cE_L(q\|\piref)$, but we use \eqref{eq:q-domination-body} to keep the bookkeeping minimal. We continue to measure pairwise reward-model error with respect to $\piref$ via \begin{align} \epw(\rhat) \;=\; \ee_{y,y'\sim\piref}\!\left[\big|\phi_{\rhat}(y,y')-\phi_{r^\star}(y,y')\big|\right]. \end{align} The following result extends our BoN guarantee from the baseline comparison $q=\piref$ to an arbitrary comparison measure $q$ satisfying \eqref{eq:q-domination-body}.
\begin{theorem}
\label{thm:bon-win-rate-upper-bound-q-body}
Suppose that $\piref$ and $q$ satisfy \eqref{eq:q-domination-body} and let $\pistar \ll \piref$.  Then for any reward models $\rstar$ and  $\rhat$ and any $N\geq1$, it holds that
\begin{align}
\R_q^{\rstar}(\pistar)-\R_q^{\rstar}(\pihatbon)
\lesssim
L \cdot N \cdot \epspw(\rhat) \cdot \log\left( \nicefrac{1}{\epspw(\rhat)} \right) + \cE_{N / \log(1/\epspw(\rhat))}\left( \pistar \| \piref \right).
\end{align}
\end{theorem}
\noindent \Cref{thm:bon-win-rate-upper-bound-q-body} recovers \Cref{thm:bon-win-rate-upper-bound} in the case that $q=\piref$; more generally, we pay in regret for the mismatch between reference model $\piref$ and comparator $q$ through the uniform density ratio bound.  This factor arises because the regret decomposition and subsequent change-of-measure steps are identical to the $q=\piref$ case (cf. \Cref{sec:pf-sketch-bon}), except that expectations against $y'\sim q$ are controlled by expectations against $y'\sim\piref$ using \eqref{eq:q-domination-body}. The remaining terms are bounded by the same approximate rejection sampling arguments from \citet{block2023sample}, so BoN remains statistically and computationally optimal under win-rate regret, while still being susceptible to reward hacking through $\epw$. As discussed in \Cref{sec:monotone_algos}, $\cE_M$-regularized BoN overcomes this issue and decouples the scaling parameter $N$ from the regularization parameter $M$ and yields a monotone-in-$N$ guarantee for $q=\piref$; in the next theorem we extend an analogous guarantee for $\piEMhat$ to general $q$ under \eqref{eq:q-domination-body}.
\begin{theorem}
\label{thm:topk-winrate-upper-bound-q-body}
For any $M, N \geq 1$, reference model $\piref$, comparator $\pistar\ll\piref$, ground truth reward $\rstar$, and learned reward model $\rhat$, it holds that
\begin{align}\label{eq:EM-reg-winrate-ub-q}
    R^{\rstar}_q(\pistar) - R^{\rstar}_q(\piEMhat) \lesssim L\cdot\cE_M(\pistar\|\piref)+L\cdot M\cdot\epw(\rhat)+\sqrt{\frac{M}{N}}.
\end{align}
\end{theorem}
Once again, we recover the case where $q = \piref$ from \Cref{thm:topk-winrate-upper-bound} by setting $L = 1$, except we now have introduced an additional term that decays as $\sqrt{M/N}$.  We emphasize that this is purely a \emph{computationaly} phenomenon as, taking $N \uparrow \infty$, we recover the same statistical rate as in Best-of-N above.
In the special case when $q=\piref$, the win-rate of a top-$1/M$ quantile policy under $\rhat$ admits an explicit order-statistics calculation, yielding a fast $1/N$ dependence. For a general $q$, the same top-quantile selection step must be analyzed through how $q$ places mass across $\rhat$-level sets near the cutoff; without additional structure, $q$ may concentrate its probability mass in a narrow band of scores around the quantile threshold, making the win-rate highly sensitive to small fluctuations in the empirical cutoff induced by a finite batch of $N$ samples. Thus, we control this step using concentration of empirical quantiles, which leads to the slower rate in $N$. Finally, we note that the density-ratio bound in~\eqref{eq:q-domination-body} is only one convenient way to relate win-rate against $q$ to quantities defined with respect to $\piref$.  We leave to future work the interesting question of the extent to which such a term is necessary in the general-$q$ setting, and whether one can recover a faster rate under additional regularity assumptions on $q$.

\section{Discussion and Related Work}\label{sec:discussion}

In this work, we studied the problem of inference-time alignment within the sample-and-evaluate framework, focusing on the widely used Best-of-$N$ (BoN) algorithm.  In contradistinction to prior work that focused on expected true reward as the objective of interest, we argued that win rate is often a more natural objective in practical settings where reward models are trained and evaluated using pairwise comparisons.  We demonstrated that BoN is both statistically and computationally optimal within the sample-and-evaluate framework for achieving good win rate, partially explaining its widespread practical success despite its simplicity.  Because BoN remains susceptible to reward-hacking even under the win-rate objective, we proposed a simple variant that provably eliminates reward-hacking while maintaining optimal statistical performance.  Our work suggests several interesting directions for future research.  First, we may ask if there is a unifying paradigm that captures both expected-reward and win-rate objectives as special cases, allowing us to better understand under what conditions (optimally tuned) BoN is optimal and when more sophisticated algorithms are necessary.  Second, while we have focused on the sample-and-evaluate framework as a natural abstraction of inference-time alignment, it would be interesting to understand if similar statistical and computational guarantees can be achieved in more general frameworks that allow for adaptive sampling or other forms of interaction with the reference model.  Finally, while our proposed variant of BoN eliminates reward-hacking in theory, it may still be too conservative in practice due to its reliance on information-theoretic guarantees; moving beyond such guarantees and leveraging the rich structure that learned representations provide in order to improve inference-time alignment remains an interesting question.  We now discuss several related works in more detail.

\paragraph{Inference-time Alignment.}  Inference time alignment is a popular framework for aligning large language models (LLMs) with human preferences without requiring expensive retraining \citep{stiennon2020learning,nakano2021webgpt,wang2023pandalm,ouyang2022training,huang2025self,gui2024bonbon}.  Both inference and training time alignment typically targets skills in reasoning, knowledge, or value alignment \citep{jaech2024openai,guo2025deepseek,ouyang2022training,brown2024large}, and has been successfully applied to a variety of tasks.  The former domains often involve binary rewards, where an answer is either true or not, while the latter often involves more subjective human preferences that generally are measured through \emph{win rate} \citep{alpaca_eval,dubois2023alpacafarm,dubois2024length,ouyang2022training}.  Our work focuses on the latter setting, where win rate is a more natural objective than expected reward.

\paragraph{Best-of-$N$.}  When the learner has access to a groundtruth verifier, \citet{brown2024large} demonstrated that simply running a Best-of-$N$ (BoN) procedure with samples from the base model and selecting the output with the highest true reward is effective at producing high-quality outputs in a range of tasks. In practice, however, the groundtruth reward is not accessible, and instead a learned reward model is used to select among samples from the base model \citep{stiennon2020learning,ouyang2022training,bai2022training}.  While BoN with a learned reward model has seen widespread practical use, it is vulnerable to \emph{reward-hacking}, which has been extensively explored \citep{gao2023scaling,stroebl2024inference}.  A number of ad hoc empirical mitigation attempts have been proposed, including regularization in representation space \citep{jinnai2025regularized}, a pessimistic BoN using empirical uncertainty quantification, and a KL-regularized variant of BoN \citep{verdun2025soft,khalaf2025inference}.  None of these approaches have theoretical guarantees that improve substantially over BoN.  On the theoretical side, the most relevant work is that of \citet{huang2025best}, which analyzed BoN under the expected-reward objective and demonstrated that it is statistically suboptimal in general due to its vulnerability to reward overoptimization.  They then proposed a more sophisticated algorithm that achieves optimal statistical performance by leveraging ideas from approximate rejection sampling \citep{block2023sample}.  In contradistinction to their work, we demonstrated that BoN is in fact statistically and computationally optimal under the win-rate objective, which we believe is often more natural in practice.  Thus, while conceptually similar, our results paint a more favorable picture of BoN than that of \citet{huang2025best} and our techniques (in particular our lower bounds) differ significantly.

\bibliographystyle{plainnat}
\bibliography{refs}

@inproceedings{huang2025best,
  title={Is Best-of-N the Best of Them? Coverage, Scaling, and Optimality in Inference-Time Alignment},
  author={Huang, Audrey and Block, Adam and Liu, Qinghua and Jiang, Nan and Krishnamurthy, Akshay and Foster, Dylan J},
  booktitle={Forty-second International Conference on Machine Learning},
  year={2025}
}

@inproceedings{block2023sample,
  title={The sample complexity of approximate rejection sampling with applications to smoothed online learning},
  author={Block, Adam and Polyanskiy, Yury},
  booktitle={The Thirty Sixth Annual Conference on Learning Theory},
  pages={228--273},
  year={2023},
  organization={PMLR}
}

@article{chen2025coverage,
  title={The Coverage Principle: How Pre-Training Enables Post-Training},
  author={Chen, Fan and Huang, Audrey and Golowich, Noah and Malladi, Sadhika and Block, Adam and Ash, Jordan T and Krishnamurthy, Akshay and Foster, Dylan J},
  journal={arXiv preprint arXiv:2510.15020},
  year={2025}
}

@inproceedings{huang2025self,
  title={Self-Improvement in Language Models: The Sharpening Mechanism},
  author={Huang, Audrey and Block, Adam and Foster, Dylan J and Rohatgi, Dhruv and Zhang, Cyril and Simchowitz, Max and Ash, Jordan T and Krishnamurthy, Akshay},
  booktitle={The Thirteenth International Conference on Learning Representations},
  year={2025}
}

@article{stiennon2020learning,
  title={Learning to summarize with human feedback},
  author={Stiennon, Nisan and Ouyang, Long and Wu, Jeffrey and Ziegler, Daniel and Lowe, Ryan and Voss, Chelsea and Radford, Alec and Amodei, Dario and Christiano, Paul F},
  journal={Advances in neural information processing systems},
  volume={33},
  pages={3008--3021},
  year={2020}
}

@article{nakano2021webgpt,
  title={Webgpt: Browser-assisted question-answering with human feedback},
  author={Nakano, Reiichiro and Hilton, Jacob and Balaji, Suchir and Wu, Jeff and Ouyang, Long and Kim, Christina and Hesse, Christopher and Jain, Shantanu and Kosaraju, Vineet and Saunders, William and others},
  journal={arXiv preprint arXiv:2112.09332},
  year={2021}
}

@article{wang2023pandalm,
  title={Pandalm: An automatic evaluation benchmark for llm instruction tuning optimization},
  author={Wang, Yidong and Yu, Zhuohao and Zeng, Zhengran and Yang, Linyi and Wang, Cunxiang and Chen, Hao and Jiang, Chaoya and Xie, Rui and Wang, Jindong and Xie, Xing and others},
  journal={arXiv preprint arXiv:2306.05087},
  year={2023}
}

@article{ouyang2022training,
  title={Training language models to follow instructions with human feedback},
  author={Ouyang, Long and Wu, Jeffrey and Jiang, Xu and Almeida, Diogo and Wainwright, Carroll and Mishkin, Pamela and Zhang, Chong and Agarwal, Sandhini and Slama, Katarina and Ray, Alex and others},
  journal={Advances in neural information processing systems},
  volume={35},
  pages={27730--27744},
  year={2022}
}

@article{brown2024large,
  title={Large language monkeys: Scaling inference compute with repeated sampling},
  author={Brown, Bradley and Juravsky, Jordan and Ehrlich, Ryan and Clark, Ronald and Le, Quoc V and R{\'e}, Christopher and Mirhoseini, Azalia},
  journal={arXiv preprint arXiv:2407.21787},
  year={2024}
}

@article{bai2022training,
  title={Training a helpful and harmless assistant with reinforcement learning from human feedback},
  author={Bai, Yuntao and Jones, Andy and Ndousse, Kamal and Askell, Amanda and Chen, Anna and DasSarma, Nova and Drain, Dawn and Fort, Stanislav and Ganguli, Deep and Henighan, Tom and others},
  journal={arXiv preprint arXiv:2204.05862},
  year={2022}
}

@inproceedings{gao2023scaling,
  title={Scaling laws for reward model overoptimization},
  author={Gao, Leo and Schulman, John and Hilton, Jacob},
  booktitle={International Conference on Machine Learning},
  pages={10835--10866},
  year={2023},
  organization={PMLR}
}

@article{stroebl2024inference,
  title={Inference scaling flaws: The limits of llm resampling with imperfect verifiers},
  author={Stroebl, Benedikt and Kapoor, Sayash and Narayanan, Arvind},
  journal={arXiv preprint arXiv:2411.17501},
  year={2024}
}

@misc{alpaca_eval,
  author = {Xuechen Li and Tianyi Zhang and Yann Dubois and Rohan Taori and Ishaan Gulrajani and Carlos Guestrin and Percy Liang and Tatsunori B. Hashimoto },
  title = {AlpacaEval: An Automatic Evaluator of Instruction-following Models},
  year = {2023},
  month = {5},
  publisher = {GitHub},
  journal = {GitHub repository},
  howpublished = {\url{https://github.com/tatsu-lab/alpaca_eval}}
}

@article{dubois2024length,
  title={Length-Controlled AlpacaEval: A Simple Way to Debias Automatic Evaluators},
  author={Dubois, Yann and Galambosi, Bal{\'a}zs and Liang, Percy and Hashimoto, Tatsunori B},
  journal={arXiv preprint arXiv:2404.04475},
  year={2024}
}

@misc{dubois2023alpacafarm,
  title={AlpacaFarm: A Simulation Framework for Methods that Learn from Human Feedback}, 
  author={Yann Dubois and Xuechen Li and Rohan Taori and Tianyi Zhang and Ishaan Gulrajani and Jimmy Ba and Carlos Guestrin and Percy Liang and Tatsunori B. Hashimoto},
  year={2023},
  eprint={2305.14387},
  archivePrefix={arXiv},
  primaryClass={cs.LG}
}

@inproceedings{jinnai2025regularized,
  title={Regularized best-of-n sampling with minimum bayes risk objective for language model alignment},
  author={Jinnai, Yuu and Morimura, Tetsuro and Ariu, Kaito and Abe, Kenshi},
  booktitle={Proceedings of the 2025 Conference of the Nations of the Americas Chapter of the Association for Computational Linguistics: Human Language Technologies (Volume 1: Long Papers)},
  pages={9321--9347},
  year={2025}
}

@article{gui2024bonbon,
  title={Bonbon alignment for large language models and the sweetness of best-of-n sampling},
  author={Gui, Lin and G{\^a}rbacea, Cristina and Veitch, Victor},
  journal={Advances in Neural Information Processing Systems},
  volume={37},
  pages={2851--2885},
  year={2024}
}

@article{verdun2025soft,
  title={Soft Best-of-n Sampling for Model Alignment},
  author={Verdun, Claudio Mayrink and Oesterling, Alex and Lakkaraju, Himabindu and Calmon, Flavio P},
  journal={arXiv preprint arXiv:2505.03156},
  year={2025}
}

@article{khalaf2025inference,
  title={Inference-Time Reward Hacking in Large Language Models},
  author={Khalaf, Hadi and Verdun, Claudio Mayrink and Oesterling, Alex and Lakkaraju, Himabindu and Calmon, Flavio du Pin},
  journal={arXiv preprint arXiv:2506.19248},
  year={2025}
}

@article{jaech2024openai,
  title={Openai o1 system card},
  author={Jaech, Aaron and Kalai, Adam and Lerer, Adam and Richardson, Adam and El-Kishky, Ahmed and Low, Aiden and Helyar, Alec and Madry, Aleksander and Beutel, Alex and Carney, Alex and others},
  journal={arXiv preprint arXiv:2412.16720},
  year={2024}
}

@article{guo2025deepseek,
  title={Deepseek-r1: Incentivizing reasoning capability in llms via reinforcement learning},
  author={Guo, Daya and Yang, Dejian and Zhang, Haowei and Song, Junxiao and Zhang, Ruoyu and Xu, Runxin and Zhu, Qihao and Ma, Shirong and Wang, Peiyi and Bi, Xiao and others},
  journal={arXiv preprint arXiv:2501.12948},
  year={2025}
}

@article{mirtaheri2025let,
  title={Let Me Think! A Long Chain-of-Thought Can Be Worth Exponentially Many Short Ones},
  author={Mirtaheri, Parsa and Edelman, Ezra and Jelassi, Samy and Malach, Eran and Boix-Adsera, Enric},
  journal={arXiv preprint arXiv:2505.21825},
  year={2025}
}

@article{wang2022self,
  title={Self-consistency improves chain of thought reasoning in language models},
  author={Wang, Xuezhi and Wei, Jason and Schuurmans, Dale and Le, Quoc and Chi, Ed and Narang, Sharan and Chowdhery, Aakanksha and Zhou, Denny},
  journal={arXiv preprint arXiv:2203.11171},
  year={2022}
}

@article{li2022competition,
  title={Competition-level code generation with alphacode},
  author={Li, Yujia and Choi, David and Chung, Junyoung and Kushman, Nate and Schrittwieser, Julian and Leblond, R{\'e}mi and Eccles, Tom and Keeling, James and Gimeno, Felix and Dal Lago, Agustin and others},
  journal={Science},
  volume={378},
  number={6624},
  pages={1092--1097},
  year={2022},
  publisher={American Association for the Advancement of Science}
}

@article{goodhart1984monetary,
  title={Monetary theory and practice: the UK experience},
  author={Goodhart, Charles},
  year={1984},
  publisher={Macmillan London}
}

@article{rafailov2023direct,
  title={Direct preference optimization: Your language model is secretly a reward model},
  author={Rafailov, Rafael and Sharma, Archit and Mitchell, Eric and Manning, Christopher D and Ermon, Stefano and Finn, Chelsea},
  journal={Advances in Neural Information Processing Systems},
  volume={36},
  pages={53728--53741},
  year={2023}
}

@article{bradley1952rank,
  title={Rank analysis of incomplete block designs: I. the method of paired comparisons},
  author={Bradley, Ralph Allan and Terry, Milton E},
  journal={Biometrika},
  volume={39},
  number={3/4},
  pages={324--345},
  year={1952},
  publisher={JSTOR}
}

@book{polyanskiy2010channel,
  title={Channel coding: Non-asymptotic fundamental limits},
  author={Polyanskiy, Yury},
  year={2010},
  publisher={Princeton University}
}

@inproceedings{kwon2023efficient,
  title={Efficient Memory Management for Large Language Model Serving with PagedAttention},
  author={Woosuk Kwon and Zhuohan Li and Siyuan Zhuang and Ying Sheng and Lianmin Zheng and Cody Hao Yu and Joseph E. Gonzalez and Hao Zhang and Ion Stoica},
  booktitle={Proceedings of the ACM SIGOPS 29th Symposium on Operating Systems Principles},
  year={2023}
}

@article{rosset2024direct,
  title={Direct nash optimization: Teaching language models to self-improve with general preferences},
  author={Rosset, Corby and Cheng, Ching-An and Mitra, Arindam and Santacroce, Michael and Awadallah, Ahmed and Xie, Tengyang},
  journal={arXiv preprint arXiv:2404.03715},
  year={2024}
}

@article{xie2024exploratory,
  title={Exploratory preference optimization: Harnessing implicit q*-approximation for sample-efficient rlhf},
  author={Xie, Tengyang and Foster, Dylan J and Krishnamurthy, Akshay and Rosset, Corby and Awadallah, Ahmed and Rakhlin, Alexander},
  journal={arXiv preprint arXiv:2405.21046},
  year={2024}
}

@article{lambert2024tulu,
  title={Tulu 3: Pushing frontiers in open language model post-training},
  author={Lambert, Nathan and Morrison, Jacob and Pyatkin, Valentina and Huang, Shengyi and Ivison, Hamish and Brahman, Faeze and Miranda, Lester James V and Liu, Alisa and Dziri, Nouha and Lyu, Shane and others},
  journal={arXiv preprint arXiv:2411.15124},
  year={2024}
}

@book{polyanskiy2025information,
  title={Information theory: From coding to learning},
  author={Polyanskiy, Yury and Wu, Yihong},
  year={2025},
  publisher={Cambridge university press}
}

@article{neumann1963various,
  title={Various techniques used in connection with random digits},
  author={Von Neumann, John},
  journal={Collected Works of von Neumann},
  volume={5},
  pages={768--770},
  year={1963},
  publisher={Pergamon Press}
}

@book{boyd2004convex,
  title={Convex optimization},
  author={Boyd, Stephen and Vandenberghe, Lieven},
  year={2004},
  publisher={Cambridge university press}
}

\tableofcontents

\appendix
\crefalias{section}{appendix} %

\section{Investigation of Different Notions of Error and Performance}\label{sec:appendix_diff_notions_of_coverage}
In this section, we demonstrate that the pairwise error and squared error as defined in \Cref{sec:error-notions} are generally incomparable. In particular, \Cref{lem:pw-small-mse-large} shows that small $\epw$ does not imply small $\ermsq$, while \Cref{lem:mse-does-not-control-pw} shows the converse. In addition, we provide a proof of \Cref{prop:forallA-EM-impossibility}, which shows the impossibility of bounding expected reward regret as a function of the pairwise error and excess mass divergence introduced in \Cref{sec:pw-error} uniformly across all problems. 

\subsection{Proofs for Reward Model Error Comparison}

\begin{lemma}
\label{lem:pw-small-mse-large}
Fix a base policy $\piref$. For all $c>0$ there are unbounded reward models
$\rhat$ and $\rstar$ with
\begin{align}
\epw(\rhat) = 0
\qquad\text{but}\qquad
\ermsq(\rhat, \rstar)\ge c.
\end{align}
\end{lemma}

\begin{proof}
Recall the pairwise outcome
\begin{align}
\phi_{r}(y,y') = \indic\{r(y)>r(y')\}+\frac12\,\indic\{r(y)=r(y')\}.
\end{align}
and the pairwise error
\begin{align}
\epw(\rhat)\defeq \ee_{y,y'\sim\piref}\left[|\phi_{\rhat}(y,y')- \phi_{\rstar}(y,y')|\right].
\end{align}
The key observation is that the pairwise misranking loss is invariant under strictly increasing
transformations: if $\rhat(y)=\psi(\rstar(y))$ for a strictly increasing $\psi$, then for i.i.d.
$y,y'\sim\piref$,
\begin{align}
\phi_{\rhat}(y,y')=\phi_{\rstar}(y,y'),
\end{align}
hence $\epw(\rhat)=0$. Instantiating this with the strictly increasing map $\psi(t)=ct$ (for any $c>0$), i.e.\ $\rhat(y)=c\,\rstar(y)$,
we have $\epw(\rhat)=0\leq\epsilon$ for every $\epsilon>0$. Meanwhile,
\begin{align}
\ermsq(\rhat, \rstar)
&=\ee_{y\sim\piref}\!\big[(\rhat(y)-\rstar(y))^2\big]
=\ee_{y\sim\piref}\!\big[\left((c-1)\rstar(y)\right)^2\big]
=(c-1)^2\,\ee_{y\sim\piref}\!\big[\rstar(y)^2\big].
\end{align}
Choosing any $\rstar$ with nonzero second moment under $\piref$ and letting $c$ to grow arbitrarily large yields the desired result.
\end{proof}

\begin{lemma}
\label{lem:mse-does-not-control-pw}
Let $\Rmax>0$ and pick any $p\in(0,1)$. For every $\varepsilon>0$,
there exist a response space $\cY$, a reference policy $\piref$ on $\cY$,
a true reward $r^\star(\cdot):\cY\to[0,\Rmax]$, and a reward model $\rhat$
such that
\begin{align}
\ermsq(\rhat, \rstar)
=\ee_{y\sim\piref}\!\big[(\rhat(y)-r^\star(y))^2\big]
=\varepsilon^2,
\qquad
\epw(\rhat)\ \ge\ \frac{(1-p)^2}{4},
\end{align}
and the reward has nontrivial scale:
\begin{align}
\Var_{y\sim\piref}\!\big[r^\star(y)\big]
\;\ge\;
p(1-p)\left(\Rmax-\varepsilon/4\right)^2.
\end{align}
In particular, for fixed $p,\Rmax$ and $\varepsilon\to 0$, the MSE vanishes while
$\epw(\rhat)$ stays $\Omega(1)$ and $\Var(r^\star)$ stays $\Omega(\Rmax^2)$.
\end{lemma}

\begin{proof}
Fix $\varepsilon>0$. Define the response space
\begin{align}
\cY \;\defeq\; \{H\}\ \cup\ [0,\varepsilon/2].
\end{align}
Define the reference policy $\piref$ as:
with probability $p$ draw $y=H$; with probability $1-p$ draw $y=u$ where
$u\sim{\rm Unif}[0,\varepsilon/2]$. Define the true reward by
\begin{align}
r^\star(x,H)=\Rmax,
\qquad
r^\star(x,u)=u\quad\text{for }u\in[0,\varepsilon/2].
\end{align}
Let $\gamma\sim{\rm Rad}(\pm1)$ be independent of $y\sim\piref$, and define
\begin{align}
\rhat(y)\;\defeq\;r^\star(y)+\varepsilon\,\gamma,
\end{align}
so that $\rhat$ is randomized conditional on $y$. For pairwise comparisons, use
independent $\gamma,\gamma'\iid{\rm Rad}(\pm1)$ for $(y,y')$.

\paragraph{Mean squared error.}
Conditioned on $y$, we have $\rhat(y)-r^\star(y)=\varepsilon\gamma$, so
\begin{align}
\ee\big[(\rhat(y)-r^\star(y))^2\mid y\big]=\varepsilon^2,
\end{align}
and hence $\ermsq(\rhat,\rstar)=\varepsilon^2$.

\paragraph{Pairwise error lower bound.}
Recall the pairwise outcome
\begin{align}
\phi_{r}(y,y') = \indic\{r(y)>r(y')\}+\frac12\,\indic\{r(y)=r(y')\},
\end{align}
and the pairwise error
\begin{align}
\epw(\rhat)
=
\ee_{y,y'\sim\piref}\Big[\big|\phi_{r^\star}(y,y')-\phi_{\rhat}(y,y')\big|\Big].
\end{align}
Let $y,y'\iid \piref$ and consider the event
\begin{align}
L\defeq\{y\neq H\}\cap\{y'\neq H\},
\end{align}
which has probability $\pp(L)=(1-p)^2$. On $L$ we can write $y=u$, $y'=u'$ with
$u,u'\iid{\rm Unif}[0,\varepsilon/2]$, and take independent $\gamma,\gamma'\iid{\rm Rad}(\pm1)$.
Define the true gap and predicted gap
\begin{align}
\Delta \defeq r^\star(y)-r^\star(y')=u-u',
\qquad
\widehat{\Delta} \defeq \rhat(y)-\rhat(y') = \Delta + \varepsilon(\gamma-\gamma').
\end{align}
On $L$, $\pp(\Delta=0)=0$ by continuity, and $\widehat{\Delta}$ is also continuous so $\pp(\widehat{\Delta}=0)=0$.
Hence, almost surely on $L$,
\begin{align}
\phi_{r^\star}(y,y')=\indic\{\Delta>0\},\qquad \phi_{\rhat}(y,y')=\indic\{\widehat{\Delta}>0\},
\end{align}
and therefore $\big|\phi_{r^\star}(y,y')-\phi_{\rhat}(y,y')\big|=\indic\{\indic\{\Delta>0\}\neq \indic\{\widehat{\Delta}>0\}\}$ on $L$. If $\gamma=\gamma'$, then $\widehat{\Delta}=\Delta$ so the absolute difference is $0$.
If $\gamma\neq \gamma'$ (which occurs w.p.\ $1/2$), then $\varepsilon(\gamma-\gamma')\in\{\pm 2\varepsilon\}$ dominates
$\Delta$ since $|\Delta|\le \varepsilon/2$. Hence on $L\cap\{\gamma\neq\gamma'\}$,
\begin{align}
\sign(\widehat{\Delta})=\sign(\gamma-\gamma').
\end{align}
Moreover, $\sign(\gamma-\gamma')$ is symmetric and independent of $\sign(\Delta)$, hence
\begin{align}
\ee\left[\left|\phi_{r^\star}(y,y')-\phi_{\rhat}(y,y')\right|\ \Big|\ L,\ \gamma\neq\gamma'\right]
=
\pp\left(\indic\{\Delta>0\}\neq \indic\{\widehat{\Delta}>0\}\ \Big|\ L,\ \gamma\neq\gamma'\right)
=\frac12.
\end{align}
Therefore,
\begin{align}
\epw(\rhat)
\ \ge\
\pp(L)\cdot \pp(\gamma\neq \gamma')\cdot \frac12
=
(1-p)^2\cdot \frac12\cdot \frac12
=
\frac{(1-p)^2}{4}.
\end{align}

\paragraph{Reward variance (scale).}
Let $B\defeq\indic\{y=H\}\sim{\rm Bernoulli}(p)$. By the law of total variance,
\begin{align}
\Var(r^\star(y))=\Var(\ee[r^\star(y)\mid B]) + \ee[\Var(r^\star(y)\mid B)].
\end{align}
We have $\ee[r^\star(y)\mid B=1]=\Rmax$ and
$\ee[r^\star(y)\mid B=0]=\ee[u]=\varepsilon/4$, hence
\begin{align}
\Var(\ee[r^\star(y)\mid B])=p(1-p)\left(\Rmax-\varepsilon/4\right)^2.
\end{align}
Also, $\Var(r^\star(y)\mid B=1)=0$ and $\Var(r^\star(y)\mid B=0)=\Var(u)=\varepsilon^2/48$,
so $\ee[\Var(r^\star(y)\mid B)]=(1-p)\varepsilon^2/48\ge 0$. Therefore
\begin{align}
\Var(r^\star(y))\ge p(1-p)\left(\Rmax-\varepsilon/4\right)^2.
\end{align}
\end{proof}

\subsection{Proof of \Cref{prop:forallA-EM-impossibility}}

To construct a lower bound on expected reward regret, we will make 
$\rhat$ assign the same maximal score to a high-scoring region $S$ of $\piref$-mass $1/M$, so that any alignment algorithm $\cA$ has no information to distinguish points inside $S$ and can only pick some distribution over $S$. We then define $\rstar$ to place all its value on the part of $S$ that $\cA$ selects least (while keeping 
$\epw$ and $\cE_M(\pistar\|\piref)$ small), forcing a large true-reward regret. For completeness, we first restate the theorem before providing a proof.

\begin{proposition}
\label{prop:forallA-EM-impossibility}
Fix a threshold $M\ge 2$. For any inference-time algorithm $\cA$ and for every $\epsilon\in(0,1/M)$ there exists an instance with unbounded rewards $(\piref,\rhat,\rstar)$ and a comparator policy $\pistar\ll\piref$ such that
\begin{align}
\epw \le\ \epsilon
\qquad\text{and}\qquad
\cE_M(\pistar\|\piref)=\epsilon M,
\end{align}
while, for every target gap $\gamma>0$,
\begin{align}
\ee_{\pistar}[\rstar]-\ee_{\pi_{\cA}}[\rstar]
\ \ge\
\gamma.
\label{eq:forallA-gap-gamma}
\end{align}
By taking $\epsilon\to 0$, one can make the regret arbitrarily large while simultaneously having
$\epw\to 0$ and $\cE_M(\pistar\|\piref)\to 0$.
\end{proposition}

\begin{proof}
Define $\alpha\defeq 1/M$.
We construct $\piref$ to be non-atomic and supported on two disjoint measurable regions
$B$ and $S$ with $\piref(S)=\alpha$ and $\piref(B)=1-\alpha$. Precisely, let $B$ and $S$ be disjoint measurable sets, each equipped with a non-atomic probability measure. Define the reward model
\begin{align}
\rhat(y)=0 \ \text{for } y\in B,
\qquad
\rhat(y)=R_{\max}\ \text{for } y\in S.
\end{align}
Let $\pi_{\cA}$ denote the (possibly randomized) policy output by $\cA$ when run on input $(\piref,\rhat)$. Fix $\epsilon\in(0,\alpha)$.
Because $\rhat$ is constant on $S$, the algorithm has no score-based signal to distinguish points within $S$ and may concentrate
arbitrarily on some subset of $S$. We therefore choose a region $P\subseteq S$ adversarially as an $\epsilon$-mass subset
that captures as much of $\pi_{\cA}$'s mass as possible:
\begin{align}
P \in \arg\sup\Big\{\pi_{\cA}(A): A\subseteq S,\ \piref(A)=\epsilon\Big\}.
\end{align}
Note that because $\piref$ is non-atomic and $\piref(S)=\alpha$, there exist measurable subsets of $S$
of any prescribed $\piref$-mass in $[0,\alpha]$. 
An exact maximizer need not exist; any $P$ achieving the supremum up to an arbitrarily small slack suffices.
Also, define $C\defeq S\setminus P$ as well, so $\piref(C)=\alpha-\epsilon$. Since $P$ is (approximately) maximal among $\epsilon$-mass subsets of $S$, its $\pi_{\cA}$-mass is at least the average
$\pi_{\cA}$-mass over an $\epsilon/\alpha$ fraction of $S$:
\begin{align}
\pi_{\cA}(P)\ \ge\ \frac{\piref(P)}{\piref(S)}\,\pi_{\cA}(S)
\ =\ (\epsilon M)\,\pi_{\cA}(S).
\end{align}
Now we can let the true reward be maximized on the rest of $S$ by setting
\begin{align}
\rstar(y)\defeq c\,\rhat(y)\,\indic\{y\in C\}
=
\begin{cases}
cR_{\max}, & y\in C,\\
0, & y\in B\cup P,
\end{cases}
\end{align}
for some $c>0$ to be chosen later,
so that $\pistar\defeq \piref(\cdot|\,C)$ and then $\ee_{\pistar}[\rstar]=cR_{\max}$. Next, we compute $\cE_M(\pistar\|\piref)$.
The density ratio $w^\star=\frac{d\pistar}{d\piref}$ equals $\frac{1}{\piref(C)}=\frac{1}{\alpha-\epsilon}$ on $C$ and $0$ elsewhere.
Since $M=1/\alpha$, we have $w^\star>M$ on $C$, and therefore
\begin{align}
\cE_M(\pistar\|\piref)
&=\ee_{\piref}\big[(w^\star-M)_+\big]\\
&=\piref(C)\Big(\frac{1}{\piref(C)}-M\Big) \\
&=\epsilon M.
\end{align}
We must also ensure that the pairwise error is small. Recall the pairwise outcome
\begin{align}
\phi_{r}(y,y')=\indic\{r(y)>r(y')\}+\tfrac12\,\indic\{r(y)=r(y')\} 
\end{align}
and the pairwise error
\begin{align}
\epspw(\rhat)=\ee_{y,y'\sim\piref}\left[\left|\phi_{\rstar}(y,y')-\phi_{\rhat}(y,y')\right|\right].
\end{align}
In our construction, $\rhat$ is constant on $S$ and $\rstar$ is constant on each of $B,P,C$.
Thus $\phi_{\rstar}(y,y')=\phi_{\rhat}(y,y')$ except for when exactly one of $(y,y')$ falls in $P$ and the
other falls in $B\cup C$. Moreover, when $y \sim P$ and $y' \sim C$ (or vice versa), we have
$\phi_{\rhat}(y,y')=\tfrac12$ (a tie under $\rhat$) while $\phi_{\rstar}(y,y')\in\{0,1\}$ (a strict comparison under $\rstar$),
and when $y \sim P$ and $y' \sim B$ (or vice versa) we have $\phi_{\rstar}(y,y')=\tfrac12$ (a tie under $\rstar$) while $\phi_{\rhat}(y,y')\in\{0,1\}$
(a strict comparison under $\rhat$). In all these cases,
$|\phi_{\rstar}(y,y')-\phi_{\rhat}(y,y')|=\tfrac12$. Therefore,
\begin{align}
\epspw(\rhat)
&=\tfrac12\cdot \left(\pp\left(y\in P,\, y' \notin P\right) + \pp\left(y'\in P, \,y \notin P\right)\right)\\
&=\tfrac12\cdot 2\,\piref(P)\left(1-\piref(P)\right)\\
&=\epsilon(1-\epsilon)\le \epsilon.
\end{align}
Above, we chose $P\subseteq S$ with $\piref(P)=\epsilon$ so that
\begin{align}
\pi_{\cA}(P)\ \ge\ \frac{\piref(P)}{\piref(S)}\,\pi_{\cA}(S)
\ =\ (\epsilon M)\,\pi_{\cA}(S),
\end{align}
Then, since $\rstar$ is $cR_{\max}$ on $C=S\setminus P$ and $0$ elsewhere,
\begin{align}
\ee_{\pi_{\cA}}[\rstar]
&=cR_{\max}\,\pi_{\cA}(C)\\
&=cR_{\max}\big(\pi_{\cA}(S)-\pi_{\cA}(P)\big)\\
&\le cR_{\max}\,\pi_{\cA}(S)\,(1-\epsilon M)\\
&\le cR_{\max}(1-\epsilon M).
\end{align}
Therefore,
\begin{align}
\ee_{\pistar}[\rstar]-\ee_{\pi_{\cA}}[\rstar]
\ge cR_{\max}\epsilon M.
\end{align}
Finally, for any target $\gamma>0$, choosing $c\defeq \gamma/(R_{\max}\epsilon M)$ gives \eqref{eq:forallA-gap-gamma}.

\end{proof}

\section{Supporting Lemmas}
\label{sec:supporting-lemmas}
The lemmas in this section provide key auxiliary results used throughout the rest of the proofs. \Cref{lem:win-transfer-em} quantifies the win-rate penalty incurred when transferring between reward models, and provides good intuition for maneuvering the regret decomposition in the upper bounds. It is here where the alignment between the pairwise error and win-rate defintions is useful: we can translate control of pairwise outcome disagreements into control of win-rate differences. \Cref{lem:cdf-identity-pf} demonstrates the relationship between our win-rate metric and the expected quantile under $\rhat$ achieved by a policy under a fixed reward model. 

\begin{lemma}
\label{lem:win-transfer-em}
Fix a reference policy $\piref$, a true reward $\rstar$, and a learned reward model $\rhat$.
For any policy $\pi\ll\piref$ with density ratio $w_\pi \defeq \nicefrac{d\pi}{d\piref}$ and any $M\ge 1$,
\begin{align}
\left|\R^{\rstar}(\pi)-\R^{\rhat}(\pi)\right|
\;\le\;
M\cdot \epw(\rhat)\;+\;\cE_M(\pi\|\piref).
\end{align}
\end{lemma}

\begin{proof}
For any score function $r$, we defined the pairwise outcome
\begin{align}
\phi_{r}(y,y') = \indic\{r(y)>r(y')\}+\frac12\,\indic\{r(y)=r(y')\}\in\left\{0,\frac12,1\right\}.
\end{align}
Then by definition, the win-rate functional can be written as
\begin{align}
\R^{r}(\pi) = \ee_{y\sim\pi,\;y'\sim\piref}\left[\phi_{r}(y,y')\right].
\end{align}
In addition, recall the pairwise error
\begin{align}
\epw(\rhat) = \ee_{y,y'\sim\piref}\left[|\phi_{\rstar}(y,y')-\phi_{\rhat}(y,y')|\right].
\end{align}
In order to bound the win-rate difference, we first observe that
\begin{align}
\left|\R^{\rstar}(\pi)-\R^{\rhat}(\pi)\right|
&= \left|\ee_{y\sim\pi,\;y'\sim\piref}\left[\phi_{\rstar}(y,y')-\phi_{\rhat}(y,y')\right]\right| \\
&\le \ee_{y\sim\pi,\;y'\sim\piref}\left[\left|\phi_{\rstar}(y,y')-\phi_{\rhat}(y,y')\right|\right].
\end{align}
For shorthand, let $\ell(y,y')=\left|\phi_{\rstar}(y,y')-\phi_{\rhat}(y,y')\right|$. We can change measure using $w_\pi$:
\begin{align}
\ee_{y\sim\pi,\;y'\sim\piref}\left[\ell(y,y')\right]
= \ee_{y\sim\piref,\;y'\sim\piref}\left[w_\pi(y)\,\ell(y,y')\right].
\end{align}
Then, we can decompose $w_\pi(y)=(w_\pi(y)\wedge M) + (w_\pi(y)-M)_+$ and use $\ell\in[0,1]$:
\begin{align}
\ee_{y,y'\sim\piref}\left[w_\pi(y)\ell(y,y')\right]
&= \ee_{y,y'\sim\piref}\left[(w_\pi(y)\wedge M)\,\ell(y,y')\right]
  + \ee_{y,y'\sim\piref}\left[(w_\pi(y)-M)_+\,\ell(y,y')\right] \\
&\le M\cdot\ee_{y,y'\sim\piref}\left[\ell(y,y')\right]
  + \ee_{y\sim\piref}\left[(w_\pi(y)-M)_+\right] \\
&= M\cdot \epw(\rhat) + \cE_M(\pi\|\piref).
\end{align}
\end{proof}
\begin{lemma}
\label{lem:cdf-identity-pf}
Fix a reward model $\rhat:\cY\to\rr$. Define the mid-CDF transform
\begin{align}
\widetilde F_{\rhat}(t)\;\defeq\;\frac12\left(F_{\rhat}(t^-)+F_{\rhat}(t)\right)\in[0,1].
\end{align}
where
\begin{align}
F_{\rhat}(t)\;\defeq\;\pp_{y'\sim\piref}\left(\rhat(y')\le t\right),
\qquad
F_{\rhat}(t^-)\;\defeq\;\pp_{y'\sim\piref}\left(\rhat(y')< t\right).
\end{align}
Then for any policy $\pi$,
\begin{align}
\label{eq:midcdf-identity}
\R^{\rhat}(\pi)
\;=\;
\ee_{y\sim\pi}\Big[\widetilde F_{\rhat}\left(\rhat(y)\right)\Big],
\end{align}
Moreover, let $V\sim{\rm Unif}[0,1]$ be independent of everything, and define the
randomized uniformized score
\begin{align}
\label{eq:randomized-pit}
\tilde u(y)
\;\defeq\;
F_{\rhat}(\rhat(y)^-)\;+\;V\cdot\left(F_{\rhat}(\rhat(y))-F_{\rhat}(\rhat(y)^-)\right)\in[0,1].
\end{align}
Then:
\begin{enumerate}
\item For any $\pi$, $\R^{\rhat}(\pi)=\ee_{y\sim\pi,V}\big[\tilde u(y)\big]$.
\item If $y'\sim\piref$, then $\tilde u(y')\sim{\rm Unif}[0,1]$.
\end{enumerate}
\end{lemma}

\begin{proof} By the law of total expectation,
\begin{align}
\R^{\rhat}(\pi)
=
\ee_{y\sim\pi}\!\Big[
\pp_{y'\sim\piref}\left(\rhat(y)>\rhat(y')\mid y\right)
+\frac12\,\pp_{y'\sim\piref}\left(\rhat(y)=\rhat(y')\mid y\right)
\Big].
\end{align}
Fix $y$. Then
\begin{align}
\pp_{y'\sim\piref}\left(\rhat(y)>\rhat(y')\mid y\right)
=
\pp_{y'\sim\piref}\left(\rhat(y')<\rhat(y)\right)
=
F_{\rhat}(\rhat(y)^-),
\end{align}
and
\begin{align}
\pp_{y'\sim\piref}\left(\rhat(y)=\rhat(y')\mid y\right)
&=
\pp_{y'\sim\piref}\left(\rhat(y')\le \rhat(y)\right)-\pp_{y'\sim\piref}\left(\rhat(y')<\rhat(y)\right)\\
&=
F_{\rhat}(\rhat(y))-F_{\rhat}(\rhat(y)^-).   
\end{align}
Plugging these two identities into the previous display yields \eqref{eq:midcdf-identity}. 

\smallskip
\noindent
Next, let $V\sim{\rm Unif}[0,1]$ be independent. By \eqref{eq:randomized-pit} and $\ee[V]=\tfrac12$,
\begin{align}
\ee\big[\tilde u(y)\mid y\big]
&=
F_{\rhat}(\rhat(y)^-)\;+\;\ee[V]\cdot\left(F_{\rhat}(\rhat(y))-F_{\rhat}(\rhat(y)^-)\right) \\
&=
F_{\rhat}(\rhat(y)^-)+\frac12\cdot\left(F_{\rhat}(\rhat(y))-F_{\rhat}(\rhat(y)^-)\right).
\end{align}
Taking $\ee_{y\sim\pi}$ on both sides and comparing to \eqref{eq:midcdf-identity} gives
$\R^{\rhat}(\pi)=\ee_{y\sim\pi,\;V}[\tilde u(y)]$.

\smallskip
\noindent
Finally, we prove $\tilde u(y')\sim{\rm Unif}[0,1]$ when $y'\sim\piref$.
Let $X\defeq \rhat(y')$ and write
\begin{align}
A\defeq F_{\rhat}(X^-),\qquad B\defeq F_{\rhat}(X),
\end{align}
so that conditional on $X$, $\tilde u(y')=A+V(B-A)$ is uniform on the interval $[A,B]$.
Fix any $u\in[0,1]$ and define the generalized inverse
\begin{align}
t_u \;\defeq\; \inf\{t\in\rr:\;F_{\rhat}(t)\ge u\}.
\end{align}
By definition of $t_u$, we have $F_{\rhat}(t_u^-)\le u\le F_{\rhat}(t_u)$ and
\begin{align}
\pp(X<t_u)=F_{\rhat}(t_u^-),\qquad \pp(X=t_u)=F_{\rhat}(t_u)-F_{\rhat}(t_u^-).
\end{align}
Now decompose according to $X$:
\begin{itemize}
\item If $X<t_u$, then $B=F_{\rhat}(X)\le F_{\rhat}(t_u^-) \le u$, hence $\pp(\tilde u\le u\mid X)=1$.
\item If $X>t_u$, then $A=F_{\rhat}(X^-)\ge F_{\rhat}(t_u)\ge u$, hence $\pp(\tilde u\le u\mid X)=0$.
\item If $X=t_u$, then by construction, $\tilde u$ is uniform on $[F_{\rhat}(t_u^-),F_{\rhat}(t_u)]$, so
\begin{align}
\pp(\tilde u\le u\mid X=t_u)=\frac{u-F_{\rhat}(t_u^-)}{F_{\rhat}(t_u)-F_{\rhat}(t_u^-)}.
\end{align}
\end{itemize}
Therefore,
\begin{align}
\pp(\tilde u\le u)
&=\pp(X<t_u)\cdot 1 \;+\;\pp(X=t_u)\cdot \frac{u-F_{\rhat}(t_u^-)}{F_{\rhat}(t_u)-F_{\rhat}(t_u^-)} \\
&=F_{\rhat}(t_u^-)\;+\;\left(F_{\rhat}(t_u)-F_{\rhat}(t_u^-)\right)\cdot
\frac{u-F_{\rhat}(t_u^-)}{F_{\rhat}(t_u)-F_{\rhat}(t_u^-)}
\;=\;u.
\end{align}
Since this holds for all $u\in[0,1]$, we conclude $\tilde u(y')\sim{\rm Unif}[0,1]$.
\end{proof}

\section{Proofs for Best-of-$N$}\label{app:bon_proofs}
This section gives proofs for the guarantees of Best-of-$N$ in \Cref{sec:BoN}. The results in this section build on the guarantees for approximate rejection sampling as mentioned in \Cref{sec:pf-sketch-bon}, and are organized as follows:
\begin{enumerate}
    \item \Cref{sec:bon-upper-bound-pf} provides a general upper bound on the regret of BoN in terms of the $\cE_M$-divergence introduced in \Cref{sec:pw-error}, and in \Cref{sec:lower-bound-pf} we give a general lower bound on regret. 
    \item In \Cref{sec:computational-lower-bound} we provide a lower bound to demonstrate the computational efficiency of BoN. 
\end{enumerate}

\noindent
The proof of \Cref{thm:bon-win-rate-upper-bound} below makes use of the approximate rejection sampling analysis of \citet{block2023sample}. We briefly remind the reader that, given a base distribution $\piref$ and a target policy $\pi\ll\piref$ with density ratio
$w\defeq d\pi/d\piref$, classical rejection sampling draws $y\sim\piref$ and accepts it with probability
$\min\{w(y)/M,1\}$; if $w\le M$ $\piref$-a.s., then the accepted sample is exactly distributed as $\pi$
(\citet{neumann1963various}).
When $w$ is not uniformly bounded, the same accept/reject rule truncates the target: mass in regions where
$w(y)>M$ cannot be fully realized and is effectively “clipped’’ at level $M$.
Running this procedure for $N$ proposals (returning a default proposal if no acceptance occurs) and $M$ sufficiently large induces a policy
$\pi_R$ that approximates $\pi$ in total variation, with error controlled by the clipped mass
$\cE_M(\pi\|\piref)$ plus an exponentially small failure-to-accept term
(\citealp{block2023sample}; see \Cref{tv-to-em-rs}).

\subsection{Proof of \Cref{thm:bon-win-rate-upper-bound}}
\label{sec:bon-upper-bound-pf}
\begin{proof} 
As detailed in \Cref{sec:pf-sketch-bon}, in order to compare $\pihatbon$ to an arbitrary comparator $\pistar$ under the true reward $\rstar$, we need to  balance two competing effects: if $N$ is too small, BoN may fail to find responses comparable to those favored by $\pistar$; if $N$ is too large, BoN may overfit to $\rhat$ and pick outputs where $\rhat$ and $\rstar$ disagree. 
These two effects correspond exactly to the two terms in the upper bound \eqref{eq:bon_ub}, and they appear via the following regret decomposition, which holds for any $M > 0$: the regret $\R^{\rstar}(\pistar)-\R^{\rstar}(\pihatbon)$ is equal to the sum of three terms,
\begin{align}\label{eq:bon-regret-decomp}
 \underbrace{\R^{\rstar}(\pistar)-\R^{\rhat}(\pistar_M)}_{\mathrm{(T1)}}  + \underbrace{\R^{\rhat}(\pistar_{M})-\R^{\rhat}(\pihatbon)}_{\mathrm{(T2)}} + \underbrace{\R^{\rhat}(\pihatbon)-\R^{\rstar}(\pihatbon)}_{\mathrm{(T3)}},
\end{align}
where $\pistar_M$ is an intermediate comparator policy that is a closest $M$-capped approximation to $\pistar$ in total variation:
\begin{align}
\label{eq:tv-proj-def}
\pistar_M
\;\in\;
\argmin_{\pi\ll\piref:\;\frac{d\pi}{d\piref}\le M\ \piref\text{-a.s.}}
\TV(\pistar,\pi).
\end{align}
To bound $\mathrm{(T1)}$, we can further decompose 
\iftoggle{colt}{
$\R^{\rstar}(\pistar)-\R^{\rhat}(\pistar_M)
= \R^{\rstar}(\pistar)-\R^{\rstar}(\pistar_M)
+ \R^{\rstar}(\pistar_M)-\R^{\rhat}(\pistar_M)$
}{
\begin{align}
\R^{\rstar}(\pistar)-\R^{\rhat}(\pistar_M)
&= \R^{\rstar}(\pistar)-\R^{\rstar}(\pistar_M) + \R^{\rstar}(\pistar_M)-\R^{\rhat}(\pistar_M)
\end{align}
} and control each of the two resulting terms separately. For the first difference, we observe that win-rate is bounded and thus $1$-Lipschitz in total variation distance (it can be written as an expectation
of the mid-CDF $\widetilde F_{\rhat}(\rhat(y))\in[0,1]$ as in \Cref{lem:cdf-identity-pf}), so we can control the loss from $\pistar$ to $\pistar_M$ in terms of $\cE_M(\pistar \| \piref)$ by \citet{block2023sample}:
\Cref{lem:em-lower-bound-tv} shows that no policy with bounded density ratio can achieve total variation distance to $\pistar$ less than $\cE_M(\pistar\|\piref)$ itself, and \Cref{lem:em-upper-bound-tv} demonstrates that the total variation distance between $\pistar$ and $\pistar_M$ as defined above is less than $\cE_M(\pistar\|\piref)$, so $\TV(\pistar,\pistar_M)=\cE_M(\pistar\|\piref)$, and hence
\begin{align}
\R^{\rstar}(\pistar)-\R^{\rstar}(\pistar_M)
\;\le\;
\TV(\pistar,\pistar_M)=\cE_M(\pistar\|\piref).
\end{align}
For the second difference, we use the fact that the likelihood ratio of $\pistar_M$ to $\piref$ is upper bounded by $M$ almost surely to control the pairwise win rate error under $\pistar_M$ between $\rhat$ and $\rstar$. Specifically, we have 
$w_{\pistar_M}(y)\le M$ $\piref$-a.s., hence $\cE_M(\pistar_M\|\piref)=0$ and applying \Cref{lem:win-transfer-em} with $\pi=\pistar_M$ and truncation $M$ yields
\begin{align}
\R^{\rstar}(\pistar_M)-\R^{\rhat}(\pistar_M)\le M\epw + \cE_M(\pistar_M\|\piref) = M\epw .
\end{align}
Putting these together, we have
\begin{align}
\mathrm{(T1)}=\R^{\rstar}(\pistar) - \R^{\rhat}(\pistar_M) \leq \cE_M(\pistar\|\piref)+M\epw.
\end{align}
To bound $\mathrm{(T2)}$, we introduce another intermediate comparator in order to utilize the optimality of $\pihatbon$ under $\rhat$ on a fixed batch of $N$ i.i.d. samples from $\piref$. Adopting the formalism of approximate rejection sampling from \citet{block2023sample}, we define a \emph{selection rule} to be a policy $\pistar_{M,\mathrm{rej}}$ that approximates $\pistar_M$ using only $N$ i.i.d.\ draws from $\piref$; it returns the first accepted $y_i$, and if no acceptance occurs it returns $y_1$. We again further decompose $\mathrm{(T2)}$ as 
\iftoggle{colt}{
$\R^{\rhat}(\pistar_{M})-\R^{\rhat}(\pihatbon)=\R^{\rhat}(\pistar_{M})-\R^{\rhat}(\pistar_{M,\mathrm{rej}}) + \R^{\rhat}(\pistar_{M,\mathrm{rej}})-\R^{\rhat}(\pihatbon)$
}{
\begin{align}
\R^{\rhat}(\pistar_{M})-\R^{\rhat}(\pihatbon)=\R^{\rhat}(\pistar_{M})-\R^{\rhat}(\pistar_{M,\mathrm{rej}}) + \R^{\rhat}(\pistar_{M,\mathrm{rej}})-\R^{\rhat}(\pihatbon)
\end{align}
}
and control each of the two resulting terms separately. For the first difference, we use the approximate rejection sampling analysis of \citet{block2023sample} to show that $\pistar_{M,\mathrm{rej}}$ approximates $\pistar_M$ in total variation distance up to an exponentially small term in $N/M$, which again transfers to win-rate by boundedness. Thus
\begin{align}
    \R^{\rhat}(\pistar_{M})-\R^{\rhat}(\pistar_{M,\mathrm{rej}})
    &\leq \TV\left(\pistar_M,\pistar_{M,\mathrm{rej}}\right) \\
    &\leq \cE_M(\pistar_M\|\piref) + \frac{1}{2}\exp\left(-\frac{N}{M}\left(1-\cE_M(\pistar_M\|\piref)\right)\right)\\
    &=\frac12\exp(-N/M),
\end{align}
where the second inequality follows from
\Cref{tv-to-em-rs}, and  the last equality holds because $\pistar_M$ is $M$-capped, so $\cE_M(\pistar_M\|\piref)=0$. To bound the second difference, we observe that conditional on any fixed batch of $N$ samples from $\piref$, BoN simply selects the output with the highest $\rhat$ value, and thus dominates any other selection rule on that same batch under $\rhat$, including $\pistar_{M,\mathrm{rej}}$. To make this concrete, we first apply \Cref{lem:cdf-identity-pf} to rewrite win-rate under $\rhat$ as an expectation of
$\widetilde F_{\rhat}(\rhat(\cdot))$, so that
\begin{align}
\R^{\rhat}(\pistar_{M,\mathrm{rej}})-\R^{\rhat}(\pihatbon)
&=
\ee_{y\sim\pistar_{M,\mathrm{rej}}}\!\left[\widetilde F_{\rhat}\!\left(\rhat(y)\right)\right]
-
\ee_{y\sim\pihatbon}\!\left[\widetilde F_{\rhat}\!\left(\rhat(y)\right)\right].
\end{align}
We can now appeal to the optimality of $\pihatbon$ under $\rhat$. Below, we use $\ee_{\YhatN\sim\piref}[\cdot]$ to refer to expectations over $N$ samples $\YhatN$ drawn from $\piref$, and, given $\YhatN$, we use $\ee_{\pistar_{M,\mathrm{rej}}\mid\YhatN}[\cdot]$ to refer to the expectation over responses induced by running approximate rejection sampling conditioned on the realization of the set $\YhatN$ drawn by the algorithm; we define $\ee_{\pihatbon\mid\YhatN}[\cdot]$ analogously. We can then use the linearity of expectation, so that
\begin{align}
\ee_{\YhatN\sim\piref} \left[\ee_{y\sim\pistar_{M,\mathrm{rej}}\mid\YhatN}\!\left[ \widetilde F_{\rhat}\left(\rhat(y)\right) \right]- \ee_{y\sim\pihatbon\mid\YhatN}\!\left[ \widetilde F_{\rhat}\left(\rhat(y)\right) \right]\right]
\end{align}
couples the set $\YhatN$ drawn by the two algorithms, and compares the performance of $\pihatbon$ and $\pistar_{M,\mathrm{rej}}$ for a fixed set of $N$ responses. Then, because the BoN policy always chooses a response with the largest value under $\rhat$ for any fixed $\YhatN$, and $\widetilde F_{\rhat}$ is nondecreasing, we have $\widetilde F_{\rhat}(\rhat(y))\le \max_{i\in[N]}\widetilde F_{\rhat}(\rhat(y_i))$ for all $y \in \YhatN$. Therefore,
\begin{align}
\R^{\rhat}(\pistar_{M,\mathrm{rej}})-\R^{\rhat}(\pihatbon) \leq 0.
\end{align}
Combining these steps gives
\begin{align}
\mathrm{(T2)}=\R^{\rhat}(\pistar_{M})-\R^{\rhat}(\pihatbon)\leq\frac12\exp(-N/M). 
\end{align}
Finally, we control $\mathrm{(T3)}$ using a similar change-of-measure argument as for $\mathrm{(T1)}$. Before applying \Cref{lem:win-transfer-em} again, we first show that BoN has density ratio at most $N$ relative to $\piref$, and thus the pairwise win-rate error under $\pihatbon$ between $\rhat$ and $\rstar$ is at most $N \cdot\epw(\rhat)$. Let $\YhatN \defeq (y_1,\dots,y_N)\in\cY^N$ denote the sequence of $N$ i.i.d.\ draws from $\piref$,
and write
\begin{align}
\piref(\YhatN)
\;\defeq\;
\prod_{i=1}^N \piref(y_i)
\end{align}
Since conditioned on $\YhatN$, the BoN algorithm outputs a single element of $\YhatN$ deterministically
(via the fixed tie-breaking rule), we can write $\pihatbon$ in closed form as
\begin{align}
\label{eq:bon-closed-form-pf}
\pihatbon(y)
\;=\;
\sum_{\YhatN\in \cY^N}
\piref(\YhatN)\cdot
\indic\{y\in \YhatN\}\cdot
\indic\Big\{\rhat(y)\ge \rhat(y')\ \ \forall\,y'\in \YhatN\Big\}.
\end{align}
Next, express $\piref(y)$ by marginalizing over $\YhatN$ and then sampling an index uniformly from $\{1,\dots,N\}$:
\begin{align}
\label{eq:ref-closed-form-pf}
\piref(y)
\;=\;
\sum_{\YhatN\in \cY^N}
\piref(\YhatN)\cdot
\sum_{y'\in \YhatN}\frac{\indic\{y'=y\}}{N}.
\end{align}
Combining \eqref{eq:bon-closed-form-pf} and \eqref{eq:ref-closed-form-pf}, for any $y$ with $\piref(y)>0$ we have
\begin{align}
\frac{\pihatbon(y)}{\piref(y)}
&=
N\cdot
\frac{
\sum_{\YhatN\in \cY^N}
\piref(\YhatN)\cdot
\indic\{y\in \YhatN\}\cdot
\indic\Big\{\rhat(y)\ge \rhat(y')\ \forall\,y'\in \YhatN\Big\}
}{
\sum_{\YhatN\in \cY^N}
\piref(\YhatN)\cdot
\sum_{y'\in \YhatN}\indic\{y'=y\}
}\\
&\le
N\cdot
\frac{
\sum_{\YhatN\in \cY^N}
\piref(\YhatN)\cdot
\indic\{y\in \YhatN\}
}{
\sum_{\YhatN\in \cY^N}
\piref(\YhatN)\cdot
\indic\{y\in \YhatN\}
}\\
&=N.
\end{align}
Thus $\nicefrac{\pihatbon}{\piref}\le N$ pointwise, and by definition,
\begin{align}
\cE_N(\pihatbon\|\piref)=\ee_{y\sim\piref}\!\left[\Big(\frac{\pihatbon(y)}{\piref(y)}-N\Big)_+\right]=0.
\end{align}
Now we can apply \Cref{lem:win-transfer-em} to $\pi=\pihatbon$ with truncation $M=N$ such that
\begin{align}
\mathrm{(T3)}\le N\epw + \cE_N(\pihatbon\|\piref) = N\epw .
\end{align}
Combining all the preceding bounds, we obtain
\begin{align}
    \R^{\rstar}(\pistar)-\R^{\rstar}(\pihatbon) \lesssim \cE_M(\pistar\|\piref) + (M+N)\epw+ \exp(-N/M).
\end{align}
Setting $M = N / \log\left( \nicefrac{1}{\epw(\rhat)} \right)$ and observing that for such $M$, it holds that $\exp(-N/M) = \epw(\rhat)$, we obtain the desired result.
\end{proof}

\begin{lemma}[Lemma 32 in \citet{block2023sample}]
\label{lem:em-lower-bound-tv}
Fix $N\ge 1$ and assume $\pistar\ll\piref$ with density ratio
$w^\star \defeq \nicefrac{d\pistar}{d\piref}$.
Then for any $\pi\ll\piref$ whose density ratio $w\defeq \nicefrac{d\pi}{d\piref}$ satisfies
$w \le M$ $\piref$-a.s., we have
\begin{align}
\TV(\pi,\pistar)
\;\ge\;
\cE_M(\pistar\|\piref).
\end{align}
\end{lemma}
\begin{lemma}
\label{lem:em-upper-bound-tv}
Fix $N\ge 1$ and assume $\pistar\ll\piref$ with density ratio
$w^\star \defeq \nicefrac{d\pistar}{d\piref}$.
Then there exists a policy $\pistar_M\ll\piref$ such that $\frac{d\pistar_M}{d\piref}(y)\le M\ \ \piref$-a.s., and
\begin{align}
\TV(\pistar_M,\pistar)=\cE_M(\pistar\|\piref).
\end{align}
\end{lemma}

\begin{proof}
Define the clipped density
\begin{align}
w_M(y)\;\defeq\;\min\{w^\star(y),\,M\}.
\end{align}
Using $w^\star = w_M + (w^\star-M)_+$ pointwise and $\ee_{y\sim\piref}[w^\star(y)]=1$, we get
\begin{align}
\ee_{y\sim\piref}[w_M(y)] \;=\; 1-\cE_M(\pistar\|\piref).
\end{align}
Thus $w_M$ has total mass $1-\cE_M(\pistar\|\piref)$ but is not yet a valid probability density when $\cE_M(\pistar\|\piref)>0$, so we must add back the excess mass while ensuring that the density ratio is still bounded. Define the slack under $M$:
\begin{align}
S \;\defeq\; \ee_{y\sim\piref}\big[(M-w^\star(y))_+\big].
\end{align}
Since $(M-w^\star)_+=M-\min\{w^\star,M\}=M-w_M$ pointwise, we have
\begin{align}
S
=
M-\ee_{y\sim\piref}[w_M(y)]
=
M-(1-\cE_M(\pistar\|\piref))
\ge \cE_M(\pistar\|\piref).
\end{align}
If $\cE_M(\pistar\|\piref)=0$, we set $g\equiv 0$ and $\pistar_M=\pistar$ and we are done. Otherwise we define
\begin{align}
g(y)\;\defeq\; \cE_M(\pistar\|\piref)\cdot \frac{(M-w^\star(y))_+}{S}.
\end{align}
Then $0\le g(y)\le (M-w^\star(y))_+$ $\piref$-a.s.\ and $\ee_{y\sim\piref}[g(y)]=\cE_M(\pistar\|\piref)$. Now we define $\pistar_M\ll\piref$ via
\begin{align}
\frac{d\pistar_M}{d\piref}(y)\;\defeq\; w_M(y)+g(y).
\end{align}
We can verify that normalization holds:
\begin{align}
\ee_{y\sim\piref}\!\left[\frac{d\pistar_M}{d\piref}(y)\right]
=
\ee_{y\sim\piref}[w_M(y)] + \ee_{y\sim\piref}[g(y)]
=
(1-\cE_M(\pistar\|\piref))+\cE_M(\pistar\|\piref)
=
1.
\end{align}
In addition, the pointwise density ratio bound still holds:
if $w^\star(y)>M$ then $w_M(y)=M$ and $(M-w^\star(y))_+=0$, so $g(y)=0$ and $w_M(y)+g(y)=M$;
if $w^\star(y)\le M$ then $w_M(y)=w^\star(y)$ and $g(y)\le M-w^\star(y)$, so $w_M(y)+g(y)\le M$.
Hence $\cE_M(\pistar_M\|\piref)=0$. Finally, we compute the total variation distance using densities:
\begin{align}
\TV(\pistar,\pistar_M)
=
\frac12\,\ee_{y\sim\piref}\!\left[\Big|w^\star(y)-\big(w_M(y)+g(y)\big)\Big|\right].
\end{align}
Pointwise,
\begin{align}
\Big|w^\star(y)-\big(w_M(y)+g(y)\big)\Big|
=
(w^\star(y)-M)_+ + g(y),
\end{align}
because on $\{w^\star>M\}$ we have $w_M=M$ and $g=0$, while on $\{w^\star\le M\}$ we have $w_M=w^\star$.
Taking expectations gives
\begin{align}
\ee_{y\sim\piref}\!\left[\Big|w^\star(y)-\big(w_M(y)+g(y)\big)\Big|\right]
&=
\ee_{y\sim\piref}[(w^\star(y)-M)_+] + \ee_{y\sim\piref}[g(y)]\\
&=
2\cE_M(\pistar\|\piref),
\end{align}
hence $\TV(\pistar,\pistar_M)=\cE_M(\pistar\|\piref)$, which implies the stated upper bound.
\end{proof}

The following lemma (which is originally adapted from Theorem 3 in \citet{block2023sample}) contains sample complexity upper bounds for the approximate rejection sampling procedure detailed in \citet{block2023sample}. We refer to \citet{huang2025best} for a detailed discussion of this result.

\begin{lemma}[Lemma D.4 in \citet{huang2025best}]
\label{tv-to-em-rs}
For any valid policy $\pi \in \Delta(\cY)$, and $N, M > 0$, let $\pi_R \in \Delta(\cY)$ be the law of responses induced by running approximate rejection sampling. Then
\begin{align}
    \TV(\pi, \pi_R)\leq \cE_M(\pi\|\piref) + \frac{1}{2}\exp\left(-\frac{N\cdot(1-\cE_M(\pi\|\piref)}{M}\right).
\end{align}
\end{lemma}

\subsection{Proof of \Cref{thm:skyline-lb}}
\label{sec:lower-bound-pf}
We first restate the theorem for the sake of completeness, after which we provide the proof. 
\begin{theorem}
\label{thm:skyline-lb-pf}
For any non-atomic $\piref$, any $\pihat_\cA$ in the sample-and-evaluate framework, and any $\epsilon < \nicefrac 12$,
there exist reward functions $\rhat$, $\rstar$ satisfying $\epspw(\rhat) \leq \epsilon$ and a comparator policy
$\pistar\ll\piref$ for which
\begin{align}
\label{eq:inf-result}
    \R^{\rstar}(\pistar) - \R^{\rstar}(\pihat_\cA)
    \geq k \cdot \inf_{M\ge 1} \left\{ M \cdot \epsilon + \cE_M(\pistar \| \piref) \right\},
\end{align}
for some universal constant $k>0$.
\end{theorem}

\begin{proof}
We build a hard instance in three steps. First, we identify a large-probability region $S$ on which
$\pihat$ cannot overweight $\piref$ by more than a constant factor; inside $S$ we carve out a small set 
$A$ whose $\piref$-mass is on the order of the pairwise-error budget $\epsilon$ where $\pihat(A)\lesssim \epsilon$. Next, we choose a simple reward pair $(\rhat,\rstar)$ so that the algorithm cannot reliably
distinguish $A$ from $A^c$ while still meeting the pairwise-error budget. Then, we let $\pistar$ upweight $\piref$ by a factor of $M^\star$ on $A$, which forces
$\pistar(A)\approx M^\star\epsilon$. Since the regret is proportional to $\pistar(A)-\pihat(A)$ in our instance, we will relate this regret to the skyline objective.

\medskip
\noindent
First, we construct a small set $A$ on which $\pihat$ has small mass. Fix any sample-and-evaluate algorithm $\pihat$. Let
$\hat w(y)\defeq \nicefrac{d\pihat}{d\piref}(y)$ so that $\ee_{\piref}[\hat w]=1$. By Markov's inequality,
\begin{align}
\pp_{y\sim\piref}\left(\hat w(y)>2\right)\le \frac{1}{2},
\end{align}
so $S\defeq\{\hat w\le 2\}$ satisfies $\piref(S)\ge \frac12$.
Set
\begin{align}
\epsilon' \defeq \min\{\epsilon,\,1/M^\star\}\in(0,\tfrac12].
\end{align}
Since $\piref$ is non-atomic and $\epsilon'\le \piref(S)$, choose a measurable $A\subseteq S$ with $\piref(A)=\epsilon'$.
Then
\begin{align}
\label{eq:pihat-A}
\pihat(A)=\int_A \hat w\,d\piref \le 2\,\piref(A)=2\epsilon'.
\end{align}
We set $\epsilon'$ in this way for two reasons: first, the comparator we construct below will place density $M^\star$ on $A$, so feasibility requires $\pistar(A)=M^\star\,\piref(A)\le 1$, i.e.\ $\piref(A)\le 1/M^\star$. In addition, we will enforce the pairwise-error budget by choosing $\piref(A)$ no larger than $\epsilon$ under the pairwise error, as we will see next. Define
\begin{align}
\rstar(y)\defeq \indic\{y\in A\},\qquad \rhat(y)\equiv 0.
\end{align}
Intuitively, $\rstar$ makes $A$ the uniquely desirable region, while a constant $\rhat$ provides no
signal about membership in $A$, so any reliable preference for $A$ must come from chance rather than information. For any reward function $r$, recall the pairwise outcome
\begin{align}
\phi_r(y,y') = \indic\{r(y)>r(y')\}+\tfrac12\,\indic\{r(y)=r(y')\},
\end{align}
and the pairwise error
\begin{align}
\epspw(\rhat)= \ee_{y,y'\sim\piref}\left[\,\left|\phi_{\rstar}(y,y')-\phi_{\rhat}(y,y')\right|\,\right].
\end{align}
Since $\rhat$ is constant, $\phi_{\rhat}(y,y')\equiv \tfrac12$. Under $\rstar=\indic_A$, we have
$\phi_{\rstar}(y,y')\in\{0,\tfrac12,1\}$ and
$\left|\phi_{\rstar}(y,y')-\tfrac12\right|=\tfrac12$ iff exactly one of $(y,y')$ lies in $A$.
Thus
\begin{align}
\epspw(\rhat)
=\frac12\cdot \pp_{y,y'\sim\piref}\left(\indic_A(y)\neq \indic_A(y')\right)
=\frac12\cdot 2\epsilon'(1-\epsilon')
\le \epsilon'
\le \epsilon.
\end{align}
Next, we construct $\pistar$ and evaluate the skyline objective. Define $\pistar\ll\piref$ via
\begin{align}
w^\star(y)\defeq \frac{d\pistar}{d\piref}(y)
=\begin{cases}
M^\star, & y\in A,\\
c, & y\notin A,
\end{cases}
\qquad
c\defeq \frac{1-M^\star\epsilon'}{1-\epsilon'} \ge 0
\quad(\text{since }\epsilon'\le 1/M^\star).
\end{align}
Note that since $M^\star\ge 1$ and $\epsilon'\in(0,1)$, we have $1-M^\star\epsilon'\le 1-\epsilon'$, hence $c\le 1$. Then $\ee_{\piref}[w^\star]=M^\star\epsilon'+c(1-\epsilon')=1$, so $\pistar$ is a valid policy and
\begin{align}
\label{eq:pistar-A}
\pistar(A)=\int_A w^\star\,d\piref = M^\star\epsilon'.
\end{align}
Moreover, for any $M\in[c,M^\star]$,
\begin{align}
\cE_M(\pistar\|\piref)=\ee_{\piref}\left[(w^\star-M)_+\right]=(M^\star-M)\piref(A)=(M^\star-M)\epsilon'.
\end{align}
Hence for such $M$,
\begin{align}
\label{eq:skyline-affine}
M\epsilon+\cE_M(\pistar\|\piref)
= M\epsilon + (M^\star-M)\epsilon'
= M^\star\epsilon' + M(\epsilon-\epsilon').
\end{align}
Next, we can explicitly compute the win-rate and regret. Fix any policy $\pi$ and write $\piref(A)=\epsilon'$.
Using the definition of win-rate under $\rstar=\indic_A$:
\begin{align}
\R^{\rstar}(\pi)
&=
\pp(\rstar(y)>\rstar(y'))+\tfrac12\,\pp(\rstar(y)=\rstar(y'))\\
&= \pi(A)(1-\piref(A)) + \frac{1}{2}\left(\pi(A)\,\piref(A)+(1-\pi(A))(1-\piref(A))\right)\\
&=\frac12(1-\piref(A))+\frac12\,\pi(A).
\end{align}
Subtracting for $\pi=\pistar$ and $\pi=\pihat$ cancels the tie-mass term, yielding
\begin{align}
\label{eq:R-diff}
\R^{\rstar}(\pistar)-\R^{\rstar}(\pihat)
=\frac12\left(\pistar(A)-\pihat(A)\right).
\end{align}
Combining \eqref{eq:R-diff} with \eqref{eq:pistar-A} and \eqref{eq:pihat-A} gives
\begin{align}
\label{eq:win-rate}
\R^{\rstar}(\pistar)-\R^{\rstar}(\pihat)
\ge \frac12\left(M^\star\epsilon'-2\epsilon'\right)
=\frac{M^\star-2}{2}\,\epsilon'.
\end{align}
Finally, we relate the regret to the skyline objective. We consider two regimes.

\smallskip
\noindent
\emph{Case 1: $M^\star\epsilon\le 1$.}
Then $\epsilon'=\epsilon$, and \eqref{eq:skyline-affine} gives, for all $M\in[c,M^\star]$,
\begin{align}
M\epsilon+\cE_M(\pistar\|\piref)=M^\star\epsilon.
\end{align}
In particular,
\begin{align}
\inf_{M\ge 1}\{M\epsilon+\cE_M(\pistar\|\piref)\}=M^\star\epsilon.
\end{align}
Indeed, for $M>M^\star$, $\cE_M(\pistar\|\piref)=0$ so $M\epsilon\ge M^\star\epsilon$; thus the infimum over $M\ge 1$ equals $M^\star\epsilon$. Combining with \eqref{eq:win-rate}, for $M^\star\ge 4$,
\begin{align}
\R^{\rstar}(\pistar)-\R^{\rstar}(\pihat)\ge \frac{M^\star-2}{2}\,\epsilon
\ge \frac14\,M^\star\epsilon
=\frac14\inf_{M\ge 1}\{M\epsilon+\cE_M(\pistar\|\piref)\}.
\end{align}

\smallskip
\noindent
\emph{Case 2: $M^\star\epsilon>1$.}
Then $\epsilon'=1/M^\star$, so \eqref{eq:win-rate} yields
\begin{align}
\R^{\rstar}(\pistar)-\R^{\rstar}(\pihat)\ge \frac{M^\star-2}{2M^\star}\ge \frac14,
\end{align}
with $M^\star\ge 4$. On the other hand, for any $M\ge 1$, we have $\cE_M(\pistar\|\piref)=\ee_{\piref}[(w^\star-M)_+]\le \ee_{\piref}[w^\star]=1$, so using $\epsilon\le \tfrac12$,
\begin{align}
\inf_{M\ge 1}\{M\epsilon+\cE_M(\pistar\|\piref)\}\le \epsilon+1\le \frac32.
\end{align}
Thus $\frac14\ge \frac16 \inf_{M\ge 1}\{M\epsilon+\cE_M(\pistar\|\piref)\}$. Taking $k=\tfrac16$ uniformly completes the proof.
\end{proof}

\subsection{Proof of \Cref{prop:query-lb-hit-topset}}
\label{sec:computational-lower-bound}
\begin{proof}
Fix $M\ge2$. We construct a hard instance in which the optimal policy must place all its weight on a subset of $\piref$, and show that any algorithm in the sample-and-evaluate framework requires on the order of $M$ samples to sample from this subset and avoid suffering constant regret. Let $\cY=[0,1]$ and let $\piref$ be $\mathrm{Unif}[0,1]$. Define
\[
I := \left[1-\frac{1}{M},\,1\right],
\qquad \piref(I)=\frac{1}{M}.
\]
Set $\hat r(y)=r^\star(y):=\indic\{y\in I\}$, so $\epspw(\hat r,r^\star)=0$.
Let $\pi^\star:=\piref(\cdot\mid I)$ so that $\frac{d\pi^\star}{d\piref}(y)=M\indic\{y\in I\}$ and hence
$\cE_M(\pi^\star\|\piref)=0$. Moreover, $\rstar(y)=1$ a.s. under $\pistar$, and $\rstar(y')=1$ w.p. $1/M$ under $\piref$, so
\begin{align}
\label{eq:pi-star-winrate-clean}
\R^{r^\star}(\pi^\star)
&=\pp_{y\sim\pistar, y'\sim\piref}\!\left( \rstar(y) > \rstar(y') \right) + \frac12\cdot\pp_{y\sim\pistar, y'\sim\piref}\!\left( \rstar(y) = \rstar(y') \right)\\
&=\left(1-\frac{1}{M}\right) + \frac12\cdot\frac1M\\
&=1-\frac{1}{2M}.
\end{align}
Let $y_1,\dots,y_N\iid\piref$ be the samples and let $\hat y$ be $\cA$'s output. By the sample-and-evaluate framework, $\hat y\in\{y_1,\dots,y_N\}$ almost surely. In addition, define the event
$E:=\{\exists i\in[N]: y_i\in I\}$. On $E^c$ we have $y_i\notin I$ for all $i$, hence $r^\star(\hat y)=0$
deterministically and therefore
\[
\R^{r^\star}(\pi_\cA\mid E^c)
=\frac12\cdot\pp_{y'\sim\piref}\left( \rstar(y') = 0\right)
=\frac12\left(1-\frac1M\right)=\frac12-\frac{1}{2M}.
\]
Using \eqref{eq:pi-star-winrate-clean}, on $E^c$ we obtain
\[
\R^{r^\star}(\pi^\star)-\R^{r^\star}(\pi_\cA\mid E^c)=\frac12.
\]
Since regret is nonnegative, taking expectations yields
\[
\ee\left[\R^{r^\star}(\pi^\star)-\R^{r^\star}(\pi_\cA)\right]
\ge \pp(E^c)\cdot\frac12
=\frac12\left(1-\frac1M\right)^N.
\]
If the expected regret is at most $\delta$, then the above implies
\begin{align}
\left(1-\frac{1}{M}\right)^N \le 2\delta.
\end{align}
When $\delta<\frac{1}{2}$, the RHS is $<1$, so taking logs on both sides gives
\begin{align}
N\log\left(1-\frac{1}{M}\right)\le \log(2\delta)
\quad\Longrightarrow\quad
N \ge \frac{\log(\tfrac{1}{2\delta})}{\log(\tfrac{M}{M-1})}.
\end{align}
For $M\ge 2$, we can apply $x\le -\log(1-x)\le \frac{x}{1-x}$ with $x=1/M$, yielding
\begin{align}
    N \gtrsim M\log\left(\tfrac{1}{\delta}\right).
\end{align}
\end{proof}

\section{Proofs for Monotone Algorithms}\label{app:monotone_algos}
This section gives proofs for the guarantees in \Cref{sec:monotone_algos}, and are organized as follows:
\begin{enumerate}
    \item \Cref{app:kkt-derivation} provides a derivation of the optimal policy to the $\cE_M$-regularized variation problem, which is introduced in \Cref{sec:monotone_algos}.
    \item \Cref{app:topk-wr-pf} provides a general upper bound on the regret of monotone algorithm in terms of the $\cE_M$-divergence introduced in \Cref{sec:pw-error}, matching the skyline lower bound on regret. 
    \item In \Cref{sec:separation} we provide a separation result between $\cE_M$-regularized BoN and the only other provably monotone algorithm: the solution to the $\chi^2$-regularized variational problem as analyzed in \citet{huang2025best}. Importantly, this highlights the weaknesses of previous approaches under a win-rate regret goal. 
\end{enumerate}

\subsection{Derivation of $\cE_M$-Regularized BoN Policy}
\label{app:kkt-derivation}
Throughout this section, we fix a measurable score function $s:\cY\to[0,1]$ on a finite $\cY$; similar arguments apply to any bounded score range and infinite response space. Consider the $\cE_M$-regularized variational problem parameterized by $\beta > 0$:
\begin{align}
\max_{\pi\ll\piref}\Big\{\ee_{\pi}[s(y)]-\beta\,\cE_M(\pi\|\piref)\Big\}
\end{align}
which is equivalently the functional optimization over $w$:
\begin{align}
\label{eq:EM-penalized}
\max_{w\ge0,\ \ee_{y\sim\piref}[w(y)]=1}\
\Big\{\ee_{\piref}\big[w(y)s(y)\big]-\beta\,\ee_{\piref}\big[(w(y)-M)_+\big]\Big\}.
\end{align}

\begin{lemma}
\label{lem:EM-reduces-to-capped-finite}
Fix $M\ge2$ and $\beta\ge 1$. Then there exists an optimizer $w^\star$ to \eqref{eq:EM-penalized} satisfying
$0\le w^\star(y)\le M$ $\piref$-a.s., and consequently $\cE_M(\pi^\star\|\piref)=0$ at the optimum (where $\pi^\star \defeq w^\star\,\piref$).
In particular, \eqref{eq:EM-penalized} is equivalent to
\begin{align}
\label{eq:capped-only}
\max_{w:\ 0\le w\le M,\ \ee_{y\sim\piref}[w(y)]=1}\
\ee_{\piref}[w(y)s(y)],
\end{align}
or, returning to $\pi$,
\begin{align}
\max_{\pi\ll\piref:\ d\pi/d\piref\le M}\ \ee_\pi[s].
\end{align}
\end{lemma}

\begin{proof}
Define the primal objective
\begin{align}
J(w)\;\defeq\;\ee_{\piref}\big[w(y)s(y)\big]-\beta\,\ee_{\piref}\big[(w(y)-M)_+\big],
\qquad
\text{s.t. } w\ge 0,\ \ee_{\piref}[w]=1.
\end{align}
Since $\cY$ is finite, this is a finite-dimensional convex program:
the penalty term is concave, so $J$ is concave in $w$, and the constraints are affine.
Slater's condition holds because $w\equiv 1$ is strictly feasible ($w(y)>0$ for all $y$ and $\ee_{\piref}[w]=1$),
so strong duality applies. Next, we introduce a scalar Lagrange multiplier $\alpha\in\rr$ for the constraint $\ee_{\piref}[w]=1$.
The Lagrangian is
\begin{align}
\cL(w,\alpha)
\;&=\;
\ee_{\piref}\!\big[w(y)s(y)-\beta (w(y)-M)_+ \big]
-\alpha\big(\ee_{\piref}[w]-1\big)\\
\;&=\;
\alpha+\ee_{\piref}\!\big[w(y)(s(y)-\alpha)-\beta (w(y)-M)_+ \big].
\end{align}
For fixed $\alpha$, the maximization over $w$ separates over $y$:
for each $y$ we maximize, over $u\ge 0$,
\begin{align}
\phi_y(\alpha)\;\defeq\;\max_{u\ge 0}\ \Big\{u\,(s(y)-\alpha)-\beta (u-M)_+\Big\}.
\end{align}
This is a 1D piecewise-linear problem. Writing $s=s(y)$,
\begin{align}
u(s-\alpha)-\beta(u-M)_+
=
\begin{cases}
u(s-\alpha), & 0\le u\le M,\\
u(s-\alpha-\beta)+\beta M, & u\ge M.
\end{cases}
\end{align}
Hence:
\begin{itemize}
\item If $s-\alpha-\beta>0$ for some $y$, then $\phi_y(\alpha)=+\infty$ (the $u\ge M$ branch has positive slope),
so the dual function is infinite. Therefore any dual-optimal $\alpha^\star$ must satisfy
\begin{align}
s(y)-\alpha^\star-\beta\le 0 \quad\text{for all }y\in\cY.
\end{align}
\item Under this condition, the $u\ge M$ branch has nonpositive slope, so its maximizers can be chosen at $u=M$.
Thus, for every $y$, there exists a maximizer of $\phi_y(\alpha^\star)$ lying in $[0,M]$.
In other words: for any $\alpha$ with finite dual value, the inner maximization can be restricted to $0\le u\le M$.
\end{itemize}
By strong duality, there exists a primal optimizer $w^\star$ that attains the (finite) dual optimum at some
$\alpha^\star$, and by the pointwise conclusion above we may choose $w^\star(y)\in[0,M]$ for every $y$.
For such a $w^\star$, we have $(w^\star(y)-M)_+=0$ for all $y$, hence $\cE_M(\pi^\star\|\piref)=0$ and
\begin{align}
J(w^\star)=\ee_{\piref}[w^\star s].
\end{align}
Therefore, the optimal value of \eqref{eq:EM-penalized} equals the optimal value of \eqref{eq:capped-only}, i.e., we may restrict attention (without loss) to policies whose density ratios satisfy $0 \leq w \leq M$ $\piref$-a.s.; moreover, an optimizer exists within this restricted class.
\end{proof}
With this simplification in hand, we may work directly with the equivalent density-ratio program \eqref{eq:capped-only}.
Since the problem is convex and strong duality holds, we can characterize the optimizer by the KKT conditions,
which are necessary and sufficient for optimality \citep{boyd2004convex}. We now apply these conditions to derive the optimal policy. In the following proposition, we assume for the sake of simplicity that there exists a threshold $\alpha^\star$ so that $\pp_{y\sim\piref}(s(y)\ge \alpha^\star)=1/M$ exactly; otherwise one can randomize inclusion on the boundary set $\{y:s(y)=\alpha^\star\}$ to make the selected $\piref$-mass exactly $1/M$, without affecting optimality. 

\begin{proposition}
\label{prop:KKT-top-quantile}
Consider \eqref{eq:capped-only} and assume the existence of an optimizer via \Cref{lem:EM-reduces-to-capped-finite}.
Fix $M\ge 2$. Then there exists a threshold $\alpha^*\in\rr$ such that an optimal density ratio satisfies
\begin{align}
w^\star(y)=M\,\indic\{s(y)\ge \alpha^*\}
\qquad(\piref\text{-a.s.}),
\end{align}
where $\alpha^*$ can be chosen so that $\pp_{y\sim\piref}(s(y)\ge \alpha^\star)=1/M$.
Equivalently, the optimal policy is
\begin{align}
\piEM(\cdot)=\piref(\cdot\mid s(\cdot)\ge \alpha^*),
\end{align}
i.e., the top-$1/M$ score-threshold policy.
\end{proposition}

\begin{proof}
We solve \eqref{eq:capped-only} via the Lagrangian/KKT conditions. Introduce a scalar multiplier $\alpha\in\rr$
for the equality constraint $\ee_{y\sim\piref}[w(y)]=1$, and pointwise multipliers $\mu(y)\ge 0$ and $\nu(y)\ge 0$
for the inequality constraints $w(y)\ge 0$ and $w(y)\le M$, respectively. The Lagrangian is
\begin{align}
\cL(w;\alpha,\mu,\nu)
&\defeq
\ee_{y\sim\piref}[w(y)\,s(y)]
-\alpha\big(\ee_{y\sim\piref}[w(y)]-1\big)\\
&\quad +\ee_{y\sim\piref}[\mu(y)\,w(y)]
+\ee_{y\sim\piref}[\nu(y)\,(M-w(y))].
\end{align}
The KKT conditions at an optimum $(w^\star,\alpha^\star,\mu^\star,\nu^\star)$ are:

\emph{(i) Primal feasibility:} $0\le w^\star(y) \le M$ and $\ee_{y\sim\piref}[w^\star(y)]=1$.

\emph{(ii) Dual feasibility:} $\mu^\star(y)\ge 0$ and $\nu^\star(y)\ge 0$ $\piref$-a.s.

\emph{(iii) Stationarity (pointwise):}
for $\piref$-a.s.\ $y$,
\begin{align}
\label{eq:stationarity-pf}
s(y)-\alpha^\star+\mu^\star(y)-\nu^\star(y)=0.
\end{align}

\emph{(iv) Complementary slackness:}
\begin{align}
\mu^\star(y)\,w^\star(y)=0,
\qquad
\nu^\star(y)\,(M-w^\star(y))=0
\qquad(\piref\text{-a.s.}).
\end{align}

\noindent We now deduce the threshold structure.
Fix $y$ and consider three cases.

\smallskip
\noindent\textbf{Case 1:} $0<w^\star(y)<M$.
Then complementary slackness forces $\mu^\star(y)=0$ and $\nu^\star(y)=0$,
and \eqref{eq:stationarity-pf} gives $s(y)=\alpha^\star$.

\smallskip
\noindent\textbf{Case 2:} $w^\star(y)=0$.
Then $\mu^\star(y)\ge 0$ may be nonzero, while $\nu^\star(y)=0$ (since $M-w^\star(y)=M>0$ implies $\nu^\star(y)=0$ by slackness).
Thus \eqref{eq:stationarity-pf} becomes $s(y)-\alpha^\star+\mu^\star(y)=0$, so $s(y)\le \alpha^\star$.

\smallskip
\noindent\textbf{Case 3:} $w^\star(y)=M$.
Then $\nu^\star(y)\ge 0$ may be nonzero, while $\mu^\star(y)=0$ (since $w^\star(y)>0$ implies $\mu^\star(y)=0$).
Thus \eqref{eq:stationarity-pf} becomes $s(y)-\alpha^\star-\nu^\star(y)=0$, so $s(y)\ge \alpha^\star$.

\smallskip
\noindent Putting the three cases together yields the pointwise characterization:
\begin{align}
w^\star(y)=
\begin{cases}
M, & s(y)>\alpha^\star,\\
0, & s(y)<\alpha^\star,\\
\text{any value in }[0,M], & s(y)=\alpha^\star.
\end{cases}
\end{align}
Then the density ratio of the optimizer satisfies
\begin{align}
w^\star(y)=M\,\indic\{s(y)\ge \alpha^\star\}\qquad(\piref\text{-a.s.}).
\end{align}
Finally, enforce the normalization $\ee_{y\sim\piref}[w^\star(y)]=1$:
\begin{align}
1=\ee_{y \sim \piref}[w^\star(y)]
=M\,\pp_{y\sim\piref}\big(s(y)\ge \alpha^\star\big),
\end{align}
so $\pp_{y \sim \piref}(s(y)\ge \alpha^\star)=1/M$. Then the induced policy is precisely
$\piEM=\piref(\cdot\mid s\ge \alpha^*)$.
\end{proof}

\subsection{Proof of \Cref{thm:topk-winrate-upper-bound}}
\label{app:topk-wr-pf}

\begin{proof}
Fix $M\ge 1$ and let $k\defeq\lceil N/M\rceil$.
Recall that $\piEMhat$ is the (unconditional) law of the random output $y_{\mathrm{out}}$ produced by: sample $y_1,\dots,y_N\iid\piref$ and i.i.d.\ $V_1,\dots,V_N\sim{\rm Unif}[0,1]$ independent;
let $I$ index the top-$k$ samples under the lexicographic order $(\rhat(y_i),V_i)$;
output $y_J$ with $J$ uniform on $I$ conditional on $(\YhatN,V_{1:N})$. To compare the performance of $\piEMhat$ to that of an arbitrary comparator $\pistar$, we proceed in a similar fashion to the proof of \Cref{thm:bon-win-rate-upper-bound} in \Cref{sec:bon-upper-bound-pf} by considering the following regret decomposition:
\begin{align}
\R^{\rstar}(\pistar)-\R^{\rstar}(\piEMhat)
&=
\underbrace{\left(\R^{\rstar}(\pistar)-\R^{\rhat}(\pistar_M)\right)}_{\mathrm{(T1)}}
+
\underbrace{\left(\R^{\rhat}(\pistar_M)-\R^{\rhat}(\piEMhat)\right)}_{\mathrm{(T2)}}
+
\underbrace{\left(\R^{\rhat}(\piEMhat)-\R^{\rstar}(\piEMhat)\right)}_{\mathrm{(T3)}}.
\end{align}
As in the proof of \Cref{thm:bon-win-rate-upper-bound}, $\pistar_M$ is an intermediate comparator policy defined to be the policy that minimizes total variation distance to $\pistar$ subject to the constraint that its density ratio relative to $\piref$ is upper bounded by $M$ almost surely. To bound $\mathrm{(T1)}$, we can further decompose into
\iftoggle{colt}{
$\R^{\rstar}(\pistar)-\R^{\rhat}(\pistar_M)=\R^{\rstar}(\pistar)-\R^{\rhat}(\pistar) + \R^{\rhat}(\pistar)-\R^{\rhat}(\pistar_M)$
}{
\begin{align}
\R^{\rstar}(\pistar)-\R^{\rhat}(\pistar_M)=\R^{\rstar}(\pistar)-\R^{\rhat}(\pistar) + \R^{\rhat}(\pistar)-\R^{\rhat}(\pistar_M)
\end{align}
}
and control each of the two terms separately. For the first difference, we can simply invoke \Cref{lem:win-transfer-em}, which states that for any policy $\pi$,
\begin{align}
\big|\R^{\rstar}(\pi)-\R^{\rhat}(\pi)\big|
\;\le\;
M\,\epw+\cE_M(\pi\|\piref).
\end{align}
Applying this with $\pi=\pistar$ yields
\begin{align}
\R^{\rstar}(\pistar)-\R^{\rhat}(\pistar)\le M\,\epw+\cE_M(\pistar\|\piref).
\end{align}
For the second difference, as in the proof of \Cref{thm:bon-win-rate-upper-bound}, we observe that win-rate is bounded and thus $1$-Lipschitz in total variation distance (it can be written as an expectation
of $\widetilde F_{\rhat}(\rhat(y))\in[0,1]$ by \Cref{lem:cdf-identity-pf}), so we can control the loss from $\pistar$ to $\pistar_M$ in terms of $\cE_M(\pistar \| \piref)$ by \citet{block2023sample}: 
\Cref{lem:em-lower-bound-tv} shows that no policy with bounded density ratio can achieve total variation distance to $\pistar$ less than $\cE_M(\pistar\|\piref)$ itself, and \Cref{lem:em-upper-bound-tv} demonstrates that the total variation distance between $\pistar$ and $\pistar_M$ as defined above is less than $\cE_M(\pistar\|\piref)$, so $\TV(\pistar,\pistar_M)=\cE_M(\pistar\|\piref)$, and hence
\begin{align}
\R^{\rhat}(\pistar)-\R^{\rhat}(\pistar_M)\le \TV(\pistar,\pistar_M)=\cE_M(\pistar\|\piref).
\end{align}
Putting these together, we have 
\begin{align}
\mathrm{(T1)} = \R^{\rstar}(\pistar)-\R^{\rhat}(\pistar_M) \lesssim \cE_M(\pistar\|\piref)+M\epw.
\end{align}
Next, to bound $\mathrm{(T2)}$, we again decompose further as 
\iftoggle{colt}{
$\R^{\rhat}(\pistar_M)-\R^{\rhat}(\piEMhat) = \R^{\rhat}(\pistar_M)-\R^{\rhat}(\pi_M) + \R^{\rhat}(\pi_M)-\R^{\rhat}(\piEMhat)$.
}{
\begin{align}
\R^{\rhat}(\pistar_M)-\R^{\rhat}(\piEMhat) = \R^{\rhat}(\pistar_M)-\R^{\rhat}(\pi_M) + \R^{\rhat}(\pi_M)-\R^{\rhat}(\piEMhat).
\end{align}
} Here, we introduce another intermediate comparator $\pi_M$, the solution to the $\cE_M$-regularized variational problem \eqref{eq:cov-reg-var-prob} and the ``idealized'' version of $\piEMhat$ which selects uniformly from a top-$1/M$ tail under $\rhat$. To make this concrete, we derive the solution to \eqref{eq:cov-reg-var-prob} in \Cref{prop:KKT-top-quantile} with $s=\rhat$, writing the resulting threshold as $\lambda_M$ so that $\pp_{y \sim \piref}(s(y)\ge \alpha^\star)=1/M$ exactly. Thus, the optimizer has bounded density ratio of the form
\begin{align}
w_M^\star(y)\;=\;M\,\indic\{\rhat(y)\ge \lambda_M\},
\end{align}
$\piref$-a.s.\ and $\lambda_M$ is the $(1-\tfrac1M)$-quantile of $\rhat$ under $\piref$, so the optimal policy is the top-$1/M$ threshold policy
\begin{align}
\piEM(\cdot)\;=\;\piref\left(\cdot \mid \rhat(\cdot)\ge \lambda_M\right).
\end{align}
Now, for the first difference, recall that by \Cref{lem:cdf-identity-pf} we can rewrite $\R^{\rhat}(\pi)=\ee_{y\sim\pi}[\widetilde F_{\rhat}(\rhat(y))]$. Since $\widetilde F_{\rhat}(\cdot)$ is non-decreasing, $\pi_M$ (output by \Cref{prop:KKT-top-quantile} with $s=\rhat$) maximizes $\R^{\rhat}(\pi)$ over the class of policies with density ratio bounded by $M$. Recalling that $\pistar_M$ is also subject to this constraint, it follows that
$\R^{\rhat}(\pistar_M)\le \R^{\rhat}(\pi_M)$, and hence
\begin{align}
\R^{\rhat}(\pistar_M)- \R^{\rhat}(\pi_M)\leq0.
\end{align} 
We now bound the second difference by directly comparing the win-rate of the ideal top-$1/M$ threshold policy $\pi_M$
to that of our implementation $\piEMhat$.
A convenient way to analyze $\R^{\rhat}(\piEMhat)$ when $\rhat$ may have ties is to “forget the scale’’ of $\rhat$
and work with a \emph{randomized quantile} that induces a strict total order under $\piref$. Letting $F_{\rhat}$ and $F_{\rhat}(\cdot^-)$ be as in \Cref{lem:cdf-identity-pf}, and given an auxiliary seed $V\sim{\rm Unif}[0,1]$ independent of $y$, we define
\begin{align}
\label{eq:rand-quantile}
\tilde u(y;V)
\;\defeq\;
F_{\rhat}(\rhat(y)^-)\;+\;V\cdot\left(F_{\rhat}(\rhat(y)) - F_{\rhat}(\rhat(y)^-)\right)\in[0,1].
\end{align}
Indeed, $\tilde u(y;V)$ is exactly the randomized CDF value: it is strictly increasing in $\rhat(y)$,
and on a tie $\rhat(y)=t$ it reduces to an independent linear interpolation over the CDF jump at $t$,
so ordering by $(\rhat,V)$ is the same as ordering by $\tilde u$. Then for $y'\sim\piref$ and $V'\sim{\rm Unif}[0,1]$ independent, \Cref{lem:cdf-identity-pf} shows that $\tilde u(y';V')\sim{\rm Unif}[0,1]$,
and moreover the win-rate under $\rhat$ can be realized as a strict comparison of these randomized quantiles:
for any (possibly randomized) output $y$,
\begin{align}
\pp\left(\rhat(y)>\rhat(y')\right)\;+\;\frac12\,\pp\left(\rhat(y)=\rhat(y')\right)
\;=\;
\pp\left(\tilde u(y;V)>\tilde u(y';V')\right).
\end{align}
(Indeed, if $\rhat(y)\neq \rhat(y')$ the inequality agrees with $\tilde u$ by monotonicity, while if $\rhat(y)=\rhat(y')$
then $\tilde u(y;V)>\tilde u(y';V')$ iff $V>V'$, which occurs with probability $1/2$.) Recalling that $\piEMhat$ uses uniform tie-breaking as described in the beginning of the proof, under \eqref{eq:rand-quantile}, ranking by $(\rhat(y_i),V_i)$ is equivalent to ranking by $\tilde u_i\defeq \tilde u(y_i;V_i)$,
and $\tilde u_1,\dots,\tilde u_N \iid {\rm Unif}[0,1]$. Thus, we compute $\R^{\rhat}(\piEMhat)$ via uniform order statistics in \Cref{lem:topk-exact-winrate-hatr}, which implies the lower bound
\begin{align}
\label{eq:piEMhat-winrate-lb}
\R^{\rhat}(\piEMhat)\ge 1-\frac1{2M}-\frac{1}{N+1}.
\end{align}
Next, we calculate the win-rate of the ideal top-$1/M$ quantile policy $\pi_M$. This should be intuitive, as under $\pi_M$, the draw $y\sim\pi_M$ lies in the $\rhat$-upper tail of $\piref$ of mass $1/M$,
so against an independent $y'\sim\piref$ it wins whenever $\rhat(y)>\rhat(y')$ (which happens with probability $1-1/M$),
and ties occur exactly when $y'$ also falls in the same tail (probability $1/M$), contributing a factor $\tfrac12$ under our win-rate convention. \Cref{lem:exact-conditional-winrate} formalizes this intuition and gives 
\begin{align}
\label{eq:piM-winrate}
\R^{\rhat}(\pi_M)=1-\frac{1}{2M},
\end{align}
assuming feasibility of $\pp_{\piref}(\rhat\ge \lambda_M)=1/M$ as we did in \Cref{prop:KKT-top-quantile}. Combining \eqref{eq:piEMhat-winrate-lb} and \eqref{eq:piM-winrate} yields
\begin{align}
\R^{\rhat}(\pi_M)-\R^{\rhat}(\piEMhat)
\lesssim \frac{1}{N}.
\end{align}
Thus, we have
\begin{align}
\mathrm{(T2)} = \R^{\rhat}(\pistar_M)-\R^{\rhat}(\piEMhat) \lesssim \frac{1}{N}.
\end{align}
Finally, to bound $\mathrm{(T3)}$, we will first show that
$\cE_M(\piEMhat\|\piref)=0$ before applying \Cref{lem:win-transfer-em} to $\piEMhat$. Let $A\in\mathcal F$ be measurable. Conditional on $(\YhatN,V_{1:N})$,
the algorithm outputs $y_{\rm out}=y_J$ with $J\sim \mathrm{Unif}(I(\YhatN,V_{1:N}))$, hence
\begin{align}
\piEMhat(A)
&=
\ee_{\YhatN\sim\piref,\;V_{1:N}\sim{\rm Unif}[0,1]^N}
\!\left[
\pp\left(y_{\rm out}\in A \,\middle|\, \YhatN,V_{1:N}\right)
\right] \\
&=
\ee_{\YhatN,V_{1:N}}
\!\left[
\frac{1}{k}\sum_{i\in I(\YhatN,V_{1:N})}\indic\{y_i\in A\}
\right] \\
&\le
\ee_{\YhatN,V_{1:N}}
\!\left[
\frac{1}{k}\sum_{i=1}^{N}\indic\{y_i\in A\}
\right]\\
&=
\frac{N}{k}\,\piref(A).
\end{align}
Therefore $\piEMhat\ll\piref$ and its density ratio satisfies
\begin{align}
w_{\piEMhat}(y)\defeq\frac{d\piEMhat}{d\piref}(y)\le \frac{N}{k}
\quad\piref\text{-a.s.}
\end{align}
Since $k=\lceil N/M\rceil$, we have $k\ge N/M$ and hence $\tfrac{N}{k}\le M$. Thus $w_{\piEMhat}\le M$ $\piref$-a.s.,
so
\begin{align}
\cE_M(\piEMhat\|\piref)=\ee_{y\sim\piref}\big[(w_{\piEMhat}(y)-M)_+\big]=0.
\end{align}
Applying \Cref{lem:win-transfer-em} to $\pi=\piEMhat$ yields
\begin{align}
\mathrm{(T3)} \le M\,\epw + \cE_M(\piEMhat\|\piref)=M\,\epw.
\end{align}

\medskip
\noindent
Combining all the preceding bounds yields
\begin{align}
\R^{\rstar}(\pistar)-\R^{\rstar}(\piEMhat)
\;\lesssim\;
\cE_M(\pistar\|\piref)
\;+\;
M\,\epw
\;+\;
\frac{1}{N}.
\end{align}

\end{proof}

\begin{lemma}
\label{lem:exact-conditional-winrate}
Let $\lambda\in\rr$ and $A_\lambda\defeq\{y:\rhat(y)\ge \lambda\}$ with
$p_\lambda\defeq\pp_{y\sim\piref}(A_\lambda)$.
Let $\tilde\pi_\lambda\defeq\piref(\cdot\mid A_\lambda)$.
Then
\begin{align}
\R^{\rhat}(\tilde\pi_\lambda)=1-\frac{p_\lambda}{2}.
\end{align}
\end{lemma}

\begin{proof}
Fix $\lambda$ and let $A_\lambda=\{y:\rhat(y)\ge \lambda\}$, $p_\lambda=\piref(A_\lambda)$, and
$\tilde\pi_\lambda=\piref(\cdot\mid A_\lambda)$. Then, for $\tilde\pi_\lambda$, $\rhat(y)\ge \lambda$ always. To compute the win-rate of $\tilde\pi_\lambda$, we can split on the event that $\rhat(y')\ge \lambda$:
\begin{align}
\R^{\rhat}(\tilde\pi_\lambda)
&=
\pp_{y\sim\tilde\pi_\lambda,y'\sim\piref}\left(\rhat(y)>\rhat(y')\right)+\frac12\,\pp_{y\sim\tilde\pi_\lambda,y'\sim\piref}\left(\rhat(y)=\rhat(y')\right)\\
&=
\pp\left(\rhat(y')<\lambda\right)
+\pp\left(\rhat(y)>\rhat(y'),\ \rhat(y')\ge \lambda\right)
+\frac12\,\pp\left(\rhat(y)=\rhat(y'),\ \rhat(y')\ge \lambda\right)\\
&=
(1-p_\lambda)
+\pp\left(\rhat(y)>\rhat(y'),\ \rhat(y')\ge \lambda\right)
+\frac12\,\pp\left(\rhat(y)=\rhat(y'),\ \rhat(y')\ge \lambda\right).
\end{align}
Next, we condition on the event $\{\rhat(y)\ge \lambda,\ \rhat(y')\ge \lambda\}$, which is implied by the events in the last two terms in the equation above. Let
\begin{align}
q \;\defeq\; \pp\left(\rhat(y)>\rhat(y') \,\big|\, \rhat(y)\ge \lambda,\ \rhat(y')\ge \lambda\right),
\qquad
t \;\defeq\; \pp\left(\rhat(y)=\rhat(y') \,\big|\, \rhat(y)\ge \lambda,\ \rhat(y')\ge \lambda\right).
\end{align}
By exchangeability, the events $\{\rhat(y)>\rhat(y')\}$ and $\{\rhat(y')>\rhat(y)\}$ have the same conditional
probability, and together with ties they partition the conditional sample space, hence $2q+t=1$, implying
$q+\frac12 t=\frac12$. Therefore, by Bayes' Theorem,
\begin{align}
\pp\left(\rhat(y)>\rhat(y'),\ \rhat(y')\ge \lambda\right)
+\frac12\,\pp\left(\rhat(y)=\rhat(y'),\ \rhat(y')\ge \lambda\right)
&=
\pp\left(\rhat(y)\ge \lambda,\ \rhat(y')\ge \lambda\right)\cdot\left(q+\frac12 t\right)\\
&=\frac12\,\pp\left(\rhat(y)\ge \lambda,\ \rhat(y')\ge \lambda\right).
\end{align}
Finally, note that
\begin{align}
\pp\left(\rhat(y)\ge \lambda,\ \rhat(y')\ge \lambda\right)
&=
\pp_{y\sim\tilde\pi_\lambda,y'\sim\piref}\left(y\in A_\lambda,\ y'\in A_\lambda\right)\\
&=
\frac{1}{p_\lambda}\pp_{y\sim\piref,y'\sim\piref}\left(y\in A_\lambda,\ y'\in A_\lambda\right)\\
&= p_\lambda.
\end{align}
Plugging in gives
\begin{align}
\R^{\rhat}(\tilde\pi_\lambda) = 1-\frac{p_\lambda}{2}.
\end{align}
\end{proof}

\begin{lemma}
\label{lem:topk-exact-winrate-hatr}
Fix $M\ge 1$ and integer $N\ge 1$, and let $k\defeq\lceil N/M\rceil$. Then the win-rate of $\piEMhat$ against $\piref$ under $\rhat$ is
\begin{align}
\R^{\rhat}(\piEMhat)
\;=\;
1-\frac{k+1}{2(N+1)}.
\end{align}
\end{lemma}

\begin{proof}
Let $F_{\rhat}$ and $F_{\rhat}(\cdot^-)$ be as in \Cref{lem:cdf-identity-pf}, and define the randomized quantile
\begin{align}
\tilde u(y;V)
\;\defeq\;
F_{\rhat}(\rhat(y)^-)+V\left(F_{\rhat}(\rhat(y))-F_{\rhat}(\rhat(y)^-)\right).
\end{align}
Set $\tilde u_i\defeq \tilde u(y_i;V_i)$. The randomized-PIT argument provided in \Cref{lem:cdf-identity-pf} gives
\begin{align}
\tilde u_1,\dots,\tilde u_N \iid {\rm Unif}[0,1].
\end{align}
Moreover, by construction, ordering the indices by $(\rhat(y_i),V_i)$ is equivalent to ordering them by $\tilde u_i$
(since ties in $\rhat$ are broken by $V_i$, which exactly corresponds to the interpolation inside each CDF jump).

Let $y_{\rm out}$ denote the algorithm output and define $\tilde u_{\rm out}\defeq \tilde u(y_{\rm out};V_J)$.
Introduce an independent “opponent’’ draw $y'\sim\piref$ and $V'\sim{\rm Unif}[0,1]$ independent of everything,
and write $\tilde u'\defeq \tilde u(y';V')\sim{\rm Unif}[0,1]$. Then
\begin{align}
\R^{\rhat}(\piEMhat)
&=
\pp\left(\rhat(y_{\rm out})>\rhat(y')\right)+\frac12\,\pp\left(\rhat(y_{\rm out})=\rhat(y')\right)\\
&=
\pp\left(\tilde u_{\rm out}>\tilde u'\right)
=
\ee[\tilde u_{\rm out}],
\end{align}
where the last step uses $\tilde u'\sim{\rm Unif}[0,1]$ independent of $\tilde u_{\rm out}$.

Now condition on the realized batch $(\tilde u_1,\dots,\tilde u_N)$ and $\tilde u_{(1)}\le\cdots\le \tilde u_{(N)}$ be the order statistics.
Since $J$ is uniform over the top-$k$ indices (under the $\tilde u$-ranking),
\begin{align}
\ee\!\big[\tilde u_{\rm out}\mid \tilde u_1,\dots,\tilde u_N\big]
=
\frac{1}{k}\sum_{j=N-k+1}^N \tilde u_{(j)}.
\end{align}
Taking expectations and using $\ee[\tilde u_{(j)}]=\frac{j}{N+1}$ for ${\rm Unif}[0,1]$ order statistics gives
\begin{align}
\R^{\rhat}(\piEMhat)
=
\ee[\tilde u_{\rm out}]
=
\frac{1}{k}\sum_{j=N-k+1}^N \ee[\tilde u_{(j)}]
=
\frac{1}{k}\sum_{j=N-k+1}^N \frac{j}{N+1}
=
1-\frac{k+1}{2(N+1)}.
\end{align}
Since $\big|\frac{k}{N+1}-\frac{1}{M}\big|\le \frac{2}{N+1}$, when $k=\lceil N/M\rceil$, we in fact have
\begin{align}
1-\frac{1}{2M}-\frac{1}{N+1}
\;\le\;
\R^{\rhat}(\piEMhat)
\;\le\;
1-\frac{1}{2M}+\frac{1}{N+1}.
\end{align}
\end{proof}

\subsection{Proof of \Cref{prop:chi2-vs-cov-arbitrary-c}}
\label{sec:separation}

We first describe the hard instance construction before stating the formal separation result and proof.  
\paragraph{Hard instance setup.}
Fix a target constant $c>1$ and choose parameters
$k \ge\sqrt{2c}$ and $\delta\in(0,1)$ with $\delta\ge1-\tfrac{1}{2c}$. 
For any $\epsilon\in(0,\tfrac{1}{4k^2}]$, define
$\alpha \defeq k\sqrt{\epsilon}$ and $M \defeq 1/\alpha$, and let the reference distribution $\piref$ be supported on three outcomes $\{B,C,P\}$ with masses
\begin{align}
\label{eq:separation-setup-piref}
\piref(B)=1-\alpha,\qquad
\piref(C)=\alpha-\epsilon,\qquad
\piref(P)=\epsilon,
\end{align}
so $\epsilon<\alpha\le1/2$ for all $\epsilon\le \tfrac{1}{4k^2}$). Define the score function $\hat r:\{B,C,P\}\to[0,1]$ by 
\begin{align}
\label{eq:separation-setup-rhat}
\hat r(B)=0,\qquad \hat r(C)=1-\delta,\qquad \hat r(P)=1,
\end{align}
and the true reward
$r^\star:\{B,C,P\}\to[0,1]$ by
\begin{align}
\label{eq:separation-setup-rstar}
r^\star(B)=\tfrac12,\qquad r^\star(C)=1,\qquad r^\star(P)=0.
\end{align}
Recall for reward function $r$ the pairwise outcome
\begin{align}
\phi_{r}(y,y')=\indic\{r(y)>r(y')\}+\frac12\,\indic\{r(y)=r(y')\},
\end{align}
and the pairwise reward-model error
\begin{align}
\epw=\ee_{y,y'\sim\piref}\left[|\phi_{\rhat}(y,y')- \phi_{\rstar}(y,y')|\right].
\end{align}
In our construction, $\hat r$ and $r^\star$ induce no ties except when $y=y'$, in which case
$\phi_{\hat r}(y,y')=\phi_{r^\star}(y,y')=\tfrac12$; hence
$|\phi_{\hat r}(y,y')-\phi_{r^\star}(y,y')|=\indic\{\phi_{\hat r}(y,y')\neq \phi_{r^\star}(y,y')\}$.
Since $\rhat(B)<\rhat(C)<\rhat(P)$ while $\rstar(P)<\rstar(B)<\rstar(C)$, the only strict-order disagreements occur exactly when one draw is $P$ and the other is in $\{B,C\}$.
Therefore,
\begin{align}
\label{eq:epw-to-eps}
\epw=\pp(y=P,y'\neq P)+\pp(y\neq P,y'=P)=2\,\piref(P)\,(1-\piref(P))=2\epsilon(1-\epsilon).
\end{align}
Since the above expression is increasing on $(0, \frac{1}{2})$ and our setup assumes $\epsilon \le \frac{1}{4k^2}$, it follows that
\begin{align}
0<\epw \le\frac{1}{2k^2}\left(1-\frac{1}{4k^2}\right).
\end{align}
Finally, the two policies to be compared are the top-$\alpha$ quantile policy under $\rhat$ and the family of solutions (parametrized by $\beta$) to the $\chi^2$-regularized variational problem. 

\paragraph{Proof sketch.} Before stating the formal result, we first provide intuition for why this example creates a separation. The construction creates a spurious top-score spike $P$ of tiny reference mass $\epsilon$ that looks best under $\hat r$ but is worst under $r^\star$, alongside a near-top outcome $C$ of mass $\alpha-\epsilon$ that is truly optimal. The $\piEM$ must spread its mass over $P\cup C$, so it only pays a regret proportional to the unavoidable contamination fraction $\epsilon/\alpha$. In contrast, the $\chi^2$-regularized solution reweights proportionally to the score gap $(\hat r-\lambda)_+$, so it overconcentrates on the spike $P$ whenever $P$ has even a slight score advantage over $C$ (controlled by $\delta$), inducing an additional regret term that cannot be tuned away uniformly over $\beta$. Thus, in order to increase the constant factor win-rate regret gap between the two policies prescribed by a fixed $c$, we would like to increase the score gap by decreasing $\delta$ and increasing the total-mass of the top-quantile $\alpha$, both as a function of $c$. 

\begin{proposition}
\label{prop:chi2-vs-cov-arbitrary-c-pf}
For every target constant $c>1$, there exist $(k, \delta)$ with $k \ge\sqrt{2c}$ and $\delta \ge 1-\tfrac{1}{2c}$, and an  $\epwnaught=\epwnaught(k)$ defined by $\epwnaught\defeq\tfrac{1}{2k^2}(1-\tfrac{1}{4k^2})$, such that for every $\epw\in(0,\epwnaught]$, there exists a hard instance of the form
in the above setup for which
\begin{align}   
\inf_{\beta > 0} R^{\rstar}(\pistar) - R^{\rstar}(\pi^\chi_\beta) \geq c \cdot \left(\inf_{M > 0} R^{\rstar}(\pistar) - R^{\rstar}(\piEM)\right).
\end{align}
\end{proposition}
The range for $\epw$ above is implied by the constraint $\epsilon \le \tfrac{1}{4k^2}$ based on our construction. Moreover, for any target $\epw\in(0,\epwnaught]$ we may choose the unique $\epsilon$ obtained by inverting~\eqref{eq:epw-to-eps}. Since $\epw\asymp \epsilon$, achieving larger $c$ requires choosing smaller $\epsilon$ (hence smaller $\epw$) to keep $\alpha = k\sqrt{\epsilon}\le\tfrac{1}{2}$. In the proof below, we first solve for the explicit form of each policy before comparing performance and concluding a separation result.
\smallskip
\begin{proof} Fix $\beta>0$. Recall from \citet{huang2025best} that the $\chi^2$-regularized solution admits the pointwise form
\begin{align}
\pi^\chi_\beta(y)\;=\;\piref(y)\,w_\beta(y),
\qquad
w_\beta(y)\;:=\;\mathrm{relu}\!\left(\beta^{-1}(\rhat(y)-\lambda)\right)
\;=\;\frac{(\rhat(y)-\lambda)_+}{\beta},
\end{align}
where $\lambda$ is chosen so that $\sum_y \piref(y)w_\beta(y)=1$, i.e.
\begin{align}
\label{eq:chi2-normalization-eq}
\ee_{y\sim\piref}\big[(\rhat(y)-\lambda)_+\big]\;=\;\beta.
\end{align}

\paragraph{Enforcing normalization.}
For our instance, plugging in the values of $\rhat$ from \eqref{eq:separation-setup-rhat} and $\piref$ weights from \eqref{eq:separation-setup-piref}, the normalization condition \eqref{eq:chi2-normalization-eq} becomes
\begin{align}
\label{eq:chi2-normalization-discrete}
(1-\alpha)\,(0-\lambda)_+\;+\;(\alpha-\epsilon)\,(1-\delta-\lambda)_+\;+\;\epsilon\,(1-\lambda)_+\;=\;\beta.
\end{align}
We will define for convenience the reference mean of $\rhat$ on this instance:
\begin{align}
\label{eq:def-mu}
\mu\;:=\;\ee_{\piref}[\rhat]\;=\;(\alpha-\epsilon)(1-\delta)\;+\;\epsilon
\;=\;\alpha(1-\delta)+\epsilon\delta.
\end{align}

\paragraph{Solving for $\lambda(\beta)$.} Before calculating an explicit form of $\pi^\chi_\beta$, we identify three $\lambda$ regimes for the policy based on the values of $\beta$. We solve \eqref{eq:chi2-normalization-discrete} by cases, according to where $\lambda$ falls relative to
$0$, $1-\delta$, and $1$.

\medskip
\noindent
\textbf{Case I: $\lambda\in[1-\delta,\,1)$.}
Then $(1-\delta-\lambda)_+=0$ and $(0-\lambda)_+=0$, while $(1-\lambda)_+=1-\lambda$, so
\begin{align}
\beta=\epsilon(1-\lambda)
\quad\Longrightarrow\quad
\lambda \;=\;1-\beta/\epsilon.
\end{align}
This case is self-consistent iff $\lambda\ge 1-\delta$, i.e. $1-\beta/\epsilon\ge 1-\delta$
equivalently
\begin{align}
\label{eq:case1-range}
\beta \le \epsilon\delta.
\end{align}

\medskip
\noindent
\textbf{Case II: $\lambda\in[0,\,1-\delta)$.}
Then $(0-\lambda)_+=0$, while both $(1-\delta-\lambda)_+=1-\delta-\lambda$ and $(1-\lambda)_+=1-\lambda$, so
\begin{align}
\beta=(\alpha-\epsilon)(1-\delta-\lambda)+\epsilon(1-\lambda) =\mu-\alpha\lambda.
\end{align}
Hence
\begin{align}
\label{eq:lambda-case2}
\lambda\;=\;\frac{\mu-\beta}{\alpha}.
\end{align}
Self-consistency requires $\lambda<1-\delta$ and $\lambda\ge 0$.
The first is
\begin{align}
\frac{\mu-\beta}{\alpha}<1-\delta
\;\Longleftrightarrow\;
\mu-\beta<\alpha(1-\delta)
\;\Longleftrightarrow\;
\beta>\mu-\alpha(1-\delta)=\epsilon\delta,
\end{align}
and the second is $\mu-\beta\ge 0\iff \beta\le \mu$. So Case II holds for
\begin{align}
\label{eq:case2-range}
\epsilon\delta \;<\;\beta\;\le\;\mu.
\end{align}

\medskip
\noindent
\textbf{Case III: $\lambda<0$.}
Then all three ReLU terms are active:
$(0-\lambda)_+=-\lambda$, $(1-\delta-\lambda)_+=1-\delta-\lambda$, $(1-\lambda)_+=1-\lambda$, so
\begin{align}
\label{eq:lambda-case3}
\beta=(1-\alpha)(-\lambda)+(\alpha-\epsilon)(1-\delta-\lambda)+\epsilon(1-\lambda) =\mu-\lambda.
\end{align}
Self-consistency $\lambda<0$ is exactly
\begin{align}
\label{eq:case3-range}
\beta>\mu.
\end{align}

\paragraph{Finding the explicit form of $\pi^\chi_\beta$.}
Recall $\pi^\chi_\beta=\piref\cdot w_\beta$ with
\begin{align}
w_\beta(y)=\frac{(\rhat(y)-\lambda)_+}{\beta},
\end{align}
and $\lambda=\lambda(\beta)$ given by the previous case analysis. We now record the resulting
$\pi^\chi_\beta$ in each $\beta$-regime.

\medskip
\noindent
\textbf{Regime I ($\beta\le \epsilon\delta$):} here $\lambda=1-\beta/\epsilon\in[1-\delta,1)$, so only $P$ is active:
\begin{align}
w_\beta(P)=\frac{1-\lambda}{\beta}=\frac{\beta/\epsilon}{\beta}=\frac1\epsilon,
\end{align}
while
$w_\beta(C)=w_\beta(B)=0$. Thus
\begin{align}
\label{eq:pi-regime1}
\pi^\chi_\beta(P)=1,\qquad \pi^\chi_\beta(C)=0,\qquad \pi^\chi_\beta(B)=0.
\end{align}

\medskip
\noindent
\textbf{Regime II ($\epsilon\delta<\beta\le\mu$):} here $\lambda=(\mu-\beta)/\alpha\in[0,1-\delta)$, so $B$ is inactive and $C,P$ are active:
\begin{align}
w_\beta(C)=\frac{1-\delta-\lambda}{\beta}
=\frac{\beta-\epsilon\delta}{\alpha\beta},
\qquad
w_\beta(P)=\frac{1-\lambda}{\beta}
=\frac{\alpha-\mu+\beta}{\alpha\beta}.
\end{align}
Therefore
\begin{align}
\label{eq:pi-regime2}
\pi^\chi_\beta(B)=0,\qquad
\pi^\chi_\beta(C)=(\alpha-\epsilon)\left(\frac{\beta-\epsilon\delta}{\alpha\beta}\right),\qquad
\pi^\chi_\beta(P)
=\epsilon\left(\frac{\alpha-\mu+\beta}{\alpha\beta}\right).
\end{align}

\medskip
\noindent
\textbf{Regime III ($\beta>\mu$):} here $\lambda=\mu-\beta<0$, so all three actions are active. For $B$ and $C$ we have
\begin{align}
w_\beta(B)=\frac{-\lambda}{\beta}=1-\frac{\mu}{\beta},
\qquad
w_\beta(C)=\frac{1-\delta-\lambda}{\beta}=1+\frac{(1-\delta)-\mu}{\beta},
\end{align}
and for $P$ we have
\begin{align}
w_\beta(P)=\frac{1-\lambda}{\beta}=1+\frac{1-\mu}{\beta}.
\end{align}
Therefore, for $B$ and $C$,
\begin{align}
\label{eq:pi-regime3}
\pi^\chi_\beta(B)=(1-\alpha)\left(1-\frac{\mu}{\beta}\right),\quad
\pi^\chi_\beta(C)=(\alpha-\epsilon)\left(1+\frac{(1-\delta)-\mu}{\beta}\right),
\end{align}
and for $P$,
\begin{align}
\pi^\chi_\beta(P)=\epsilon\left(1+\frac{1-\mu}{\beta}\right).
\end{align}
Using this explicit formulation of $\pi^\chi_\beta$ in these three regimes, we next compute the win-rate exactly. 

\paragraph{Computing $\R^{\rstar}(\pi^\chi_\beta)$.}
For each region $t\in\{B,C,P\}$ recall the definition of win-rate as:
\begin{align}
v(t)\;\defeq\;\pp_{y'\sim\piref}\!\left(\rstar(t)>\rstar(y')\right)
\;+\;\frac12\,\pp_{y'\sim\piref}\!\left(\rstar(t)=\rstar(y')\right).
\end{align}
By construction $\rstar(C)=1$, $\rstar(B)=\tfrac12$, $\rstar(P)=0$, so ties occur only when $y'=t$.
Hence,
\begin{align}
v(C)
&=\pp(y'\in\{B,P\})+\frac12\,\pp(y'=C)
=(1-\alpha+\epsilon)+\frac12(\alpha-\epsilon)
=1-\frac{\alpha-\epsilon}{2},\\
v(B)
&=\pp(y'=P)+\frac12\,\pp(y'=B)
=\epsilon+\frac{1-\alpha}{2},\\
v(P)
&=\frac12\,\pp(y'=P)
=\frac{\epsilon}{2}.
\end{align}
Therefore, for any policy $\pi$ supported on $\{B,C,P\}$,
\begin{align}
\label{eq:WR-v-decomposition}
\R^{\rstar}(\pi)
&=\pp_{y\sim\pi,\,y'\sim\piref}\!\left(\rstar(y)>\rstar(y')\right)
+\frac12\,\pp_{y\sim\pi,\,y'\sim\piref}\!\left(\rstar(y)=\rstar(y')\right)\\
&=\sum_{t\in\{B,C,P\}}\pi(t)\,v(t)
=\pi(C)\,v(C)+\pi(B)\,v(B)+\pi(P)\,v(P).
\end{align}
We now plug in the mass of $\pi^\chi_\beta$ in each regime given in \eqref{eq:pi-regime1}--\eqref{eq:pi-regime3}.

\medskip
\noindent
\textbf{Regime I ($\beta\le \epsilon\delta$).}
Using \eqref{eq:pi-regime1} and \eqref{eq:WR-v-decomposition}, we have $\pi^\chi_\beta(P)=1$, hence
\begin{align}
\label{eq:WR-regime1}
\R^{\rstar}(\pi^\chi_\beta)=v(P)=\frac{\epsilon}{2}.
\end{align}

\medskip
\noindent
\textbf{Regime II ($\epsilon\delta<\beta\le\mu$).}
Here $\pi^\chi_\beta(B)=0$ by \eqref{eq:pi-regime2}, so by \eqref{eq:WR-v-decomposition},
\begin{align}
\R^{\rstar}(\pi^\chi_\beta)
&=\pi^\chi_\beta(C)\,v(C)+\pi^\chi_\beta(P)\,v(P)\\
&=v(P)+\pi^\chi_\beta(C)\left(v(C)-v(P)\right).
\end{align}
With the values $v(P)=\epsilon/2$ and $v(C)-v(P)=1-\alpha/2$, plugging in the explicit $\pi^\chi_\beta(C)$ from \eqref{eq:pi-regime2} yields
\begin{align}
\label{eq:WR-regime2}
\R^{\rstar}(\pi^\chi_\beta)
=\frac{\epsilon}{2}
+\left(1-\frac{\alpha}{2}\right)
(\alpha-\epsilon)\left(\frac{\beta-\epsilon\delta}{\alpha\beta}\right).
\end{align}

\medskip
\noindent
\textbf{Regime III ($\beta>\mu$).}
Plugging in the explicit $\pi^\chi_\beta$ from \eqref{eq:pi-regime3} into \eqref{eq:WR-v-decomposition} yields
\begin{align}
\label{eq:WR-regime3}
\R^{\rstar}(\pi^\chi_\beta)
&=(\alpha-\epsilon)\left(1+\frac{(1-\delta)-\mu}{\beta}\right)\left(1-\frac{\alpha-\epsilon}{2}\right)\\
&\quad+(1-\alpha)\left(1-\frac{\mu}{\beta}\right)\left(\epsilon+\frac{1-\alpha}{2}\right)
+\epsilon\left(1+\frac{1-\mu}{\beta}\right)\left(\frac{\epsilon}{2}\right).
\end{align}

\paragraph{Computing $\R^{\rstar}(\piEM)$.}
On this instance (with $\alpha>\epsilon$), the top-$\alpha$ set under $\rhat$ is exactly $\{P,C\}$, so
\begin{align}
\piEM=\piref(\cdot\mid \{P,C\}),
\qquad
\piEM(C)=\frac{\alpha-\epsilon}{\alpha}=1-\frac{\epsilon}{\alpha},
\qquad
\piEM(P)=\frac{\epsilon}{\alpha},
\qquad
\piEM(B)=0.
\end{align}
Using the values for $v(P)$ and $v(C)$ and plugging into \eqref{eq:WR-v-decomposition}, we obtain
\begin{align}
\label{eq:WR-piEM-halfties}
\R^{\rstar}(\piEM)
&=\piEM(C)\,v(C)+\piEM(P)\,v(P)\\
&=\left(1-\frac{\epsilon}{\alpha}\right)\left(1-\frac{\alpha-\epsilon}{2}\right)
+\frac{\epsilon}{\alpha}\cdot\frac{\epsilon}{2}.
\end{align}
Equivalently, since $v(C)-v(P)=1-\alpha/2$, this can be written more transparently as
\begin{align}
\label{eq:WR-piEM-halfties-simple}
\R^{\rstar}(\piEM)
=\frac{\epsilon}{2}+\left(1-\frac{\alpha}{2}\right)\piEM(C)
=\frac{\epsilon}{2}+\left(1-\frac{\alpha}{2}\right)\left(1-\frac{\epsilon}{\alpha}\right),
\end{align}
Rewriting in this way will prove to be useful when comparing policies in regret. Now that we have the win-rates for each policy, we can compute the regret against the best policy on this instance. 
\paragraph{Computing win-rate regret.}
Since $C$ is strictly optimal under $r^\star$, we take the benchmark policy to be $\pistar\defeq\delta_C$.
Thus,
\begin{align}
\R^{\rstar}(\pistar)=\R^{\rstar}(\delta_C)=v(C)=1-\frac{\alpha-\epsilon}{2}.
\end{align}
For compactness, we define the win-rate regret by
\begin{align}
\label{eq:Reg-def-halfties}
\Reg(\pi)\;\defeq\;\R^{\rstar}(\delta_C)-\R^{\rstar}(\pi)
=v(C)-\R^{\rstar}(\pi).
\end{align}

\medskip
\noindent
\textbf{Regret of $\piEM$.}
Using \eqref{eq:WR-piEM-halfties-simple} and $v(C)=1-\frac{\alpha-\epsilon}{2}=1-\frac{\alpha}{2}+\frac{\epsilon}{2}$, the regret of the top-quantile selector is
\begin{align}
\label{eq:Reg-piEM-halfties}
\Reg(\piEM)
&=v(C)-\R^{\rstar}(\piEM)\\
&=\left(1-\frac{\alpha}{2}+\frac{\epsilon}{2}\right)
-\left(\frac{\epsilon}{2}+\left(1-\frac{\alpha}{2}\right)\left(1-\frac{\epsilon}{\alpha}\right)\right)\\
&=\left(1-\frac{\alpha}{2}\right)\frac{\epsilon}{\alpha}.
\end{align}

\medskip
\noindent
\textbf{Regret of $\pi^\chi_\beta$: regime-wise.}
We use the regime-wise win-rate formulas \eqref{eq:WR-regime1}--\eqref{eq:WR-regime3} together with
$\Reg(\pi)=v(C)-\R^{\rstar}(\pi)$.

\medskip
\noindent
\textbf{Regime I ($\beta\le \epsilon\delta$).}
From \eqref{eq:WR-regime1}, $\R^{\rstar}(\pi^\chi_\beta)=\epsilon/2$, hence
\begin{align}
\label{eq:Reg-regime1-halfties}
\Reg(\pi^\chi_\beta)=v(C)-\frac{\epsilon}{2}=1-\frac{\alpha}{2}.
\end{align}

\medskip
\noindent
\textbf{Regime II ($\epsilon\delta<\beta\le\mu$).}
From \eqref{eq:WR-regime2}, $\R^{\rstar}(\pi^\chi_\beta)=\frac{\epsilon}{2}+\left(1-\frac{\alpha}{2}\right)\pi^\chi_\beta(C)$,
so
\begin{align}
\label{eq:Reg-regime2-halfties-clean}
\Reg(\pi^\chi_\beta)
&=v(C)-\R^{\rstar}(\pi^\chi_\beta)\\
&=\left(1-\frac{\alpha}{2}+\frac{\epsilon}{2}\right)
-\left(\frac{\epsilon}{2}+\left(1-\frac{\alpha}{2}\right)\pi^\chi_\beta(C)\right)\\
&=\left(1-\frac{\alpha}{2}\right)\left(1-\pi^\chi_\beta(C)\right)\\
&=\left(1-\frac{\alpha}{2}\right)\Bigg[\frac{\epsilon}{\alpha}
+\frac{\epsilon\delta}{\beta}\left(1-\frac{\epsilon}{\alpha}\right)\Bigg],
\end{align}
where in the last line we used $\pi^\chi_\beta(C)=\left(1-\epsilon/\alpha\right)\left(1-\epsilon\delta/\beta\right)$ (from \eqref{eq:pi-regime2}).
In particular, comparing to \eqref{eq:Reg-piEM-halfties},
\begin{align}
\label{eq:Reg-regime2-compare-halfties}
\Reg(\pi^\chi_\beta)
=\underbrace{\left(1-\frac{\alpha}{2}\right)\frac{\epsilon}{\alpha}}_{=\,\Reg(\piEM)}
\;+\;
\left(1-\frac{\alpha}{2}\right)\cdot \frac{\epsilon\delta}{\beta}\left(1-\frac{\epsilon}{\alpha}\right).
\end{align}
The second term above is the \emph{additional penalty} induced by $\chi^2$ tilting:
even though $\piEM$ only needs to include a $(1-\alpha/2)\cdot\epsilon/\alpha$ fraction of the poisoned spike $P$,
the $\chi^2$ solution incurs an extra loss proportional to the score gap $\delta$ and inversely proportional to the regularization parameter $\beta$, matching our intuition in the hard instance setup.

\medskip
\noindent
\textbf{Regime III ($\beta>\mu$).}
Although we already showed how to calculate the win-rate in this regime in \eqref{eq:WR-regime3}, we will now show a simpler method that is more amenable to our casework. In this regime we have $\lambda=\mu-\beta<0$, hence for all $y\in\{B,C,P\}$,
\begin{align}
w_\beta(y)
=\frac{\rhat(y)-\lambda}{\beta}
=\frac{\rhat(y)-(\mu-\beta)}{\beta}
=1+\frac{\rhat(y)-\mu}{\beta},
\end{align}
and therefore
\begin{align}
\pi^\chi_\beta(y)
=\piref(y)\,w_\beta(y)
=\piref(y)\left(1+\frac{\rhat(y)-\mu}{\beta}\right).
\end{align}
Plugging $\pi=\pi^\chi_\beta$ into the win-rate formula in \eqref{eq:WR-v-decomposition} gives
\begin{align}
\R^{\rstar}(\pi^\chi_\beta)
&=\sum_{t}\piref(t)\left(1+\frac{\rhat(t)-\mu}{\beta}\right)v(t)\\
&=\underbrace{\sum_{t}\piref(t)v(t)}_{=\R^{\rstar}(\piref)}
+\frac{1}{\beta}\underbrace{\sum_{t}\piref(t)\left(\rhat(t)-\mu\right)v(t)}_{=:S(\alpha,\epsilon,\delta)}.
\end{align}
where the constant $S(\alpha,\epsilon,\delta)$ depends only on the instance (and not on $\beta$). Indeed, one can easily verify this by rearranging \eqref{eq:WR-regime3}. Thus, for $\beta>\mu$,
\begin{align}
\label{eq:WR-regime3-affine-halfties}
\R^{\rstar}(\pi^\chi_\beta)=\R^{\rstar}(\piref)+\frac{S(\alpha,\epsilon,\delta)}{\beta},
\end{align}
Consequently, using the regret definition \eqref{eq:Reg-def-halfties} with benchmark $v(C)$, for all $\beta>\mu$,
\begin{align}
\label{eq:Reg-regime3-affine-halfties}
\Reg(\pi^\chi_\beta)
=v(C)-\R^{\rstar}(\piref)-\frac{S(\alpha,\epsilon,\delta)}{\beta},
\end{align}
which is affine in $1/\beta$.

\paragraph{Comparing regrets across regimes.}
Recall from \eqref{eq:Reg-piEM-halfties} that
\begin{align}
\label{eq:Reg-EM-recall}
\Reg(\piEM)=\left(1-\frac{\alpha}{2}\right)\frac{\epsilon}{\alpha},
\end{align}
and from \eqref{eq:Reg-def-halfties} that $\Reg(\pi)=v(C)-\R^{\rstar}(\pi)$ with
$v(C)=1-\frac{\alpha-\epsilon}{2}$. We will now lower bound $\inf_{\beta>0}\Reg(\pi^\chi_\beta)$ by analyzing each $\beta$-regime.

\medskip
\noindent
\textbf{Regime I ($\beta\le \epsilon\delta$).}
From \eqref{eq:Reg-regime1-halfties}, $\Reg(\pi^\chi_\beta)=1-\frac{\alpha}{2}$.
Using $\alpha=k\sqrt{\epsilon}$ and $\epsilon\le \frac{1}{4k^2}$, we have
\begin{align}
\frac{\epsilon}{\alpha}=\frac{\sqrt{\epsilon}}{k}\ \le\ \frac{1}{2k^2}.
\end{align}
Plugging into \eqref{eq:Reg-EM-recall} yields
\begin{align}
\Reg(\piEM)=\left(1-\frac{\alpha}{2}\right)\frac{\epsilon}{\alpha}
\ \le\ \left(1-\frac{\alpha}{2}\right)\frac{1}{2k^2}.
\end{align}
Therefore,
\begin{align}
\Reg(\pi^\chi_\beta)=1-\frac{\alpha}{2}
\ \ge\ 2k^2\,\Reg(\piEM)
\ \ge\ c\,\Reg(\piEM),
\end{align}
since $k^2\ge 2c$.

\medskip
\noindent
\textbf{Regime II ($\epsilon\delta<\beta\le\mu$).}
From \eqref{eq:Reg-regime2-halfties-clean},
\begin{align}
\Reg(\pi^\chi_\beta)
=\left(1-\frac{\alpha}{2}\right)\Bigg[\frac{\epsilon}{\alpha}
+\frac{\epsilon\delta}{\beta}\left(1-\frac{\epsilon}{\alpha}\right)\Bigg],
\end{align}
which is decreasing in $\beta$, hence minimized at $\beta=\mu$:
\begin{align}
\inf_{\epsilon\delta<\beta\le\mu}\Reg(\pi^\chi_\beta)=\Reg(\pi^\chi_\mu)
=\left(1-\frac{\alpha}{2}\right)\frac{\epsilon}{\mu}.
\end{align}
Dividing by \eqref{eq:Reg-EM-recall} gives
\begin{align}
\label{eq:ratio-mu-halfties}
\frac{\Reg(\pi^\chi_\mu)}{\Reg(\piEM)}
=\frac{\alpha}{\mu}
=\frac{1}{(1-\delta)+\delta\,\epsilon/\alpha}.
\end{align}
Using $\epsilon/\alpha=\sqrt{\epsilon}/k\le \frac{1}{2k^2}$ (from $\epsilon\le \frac{1}{4k^2}$),
\begin{align}
(1-\delta)+\delta\,\epsilon/\alpha \ \le\ (1-\delta)+\frac{\delta}{2k^2}
\ \le\ \frac{1}{2c}+\frac{1}{4c}
\ =\ \frac{3}{4c}
\ \le\ \frac{1}{c},
\end{align}
where we used $\delta\ge 1-\frac{1}{2c}$ and $k^2\ge 2c$ in the second inequality.
Plugging into \eqref{eq:ratio-mu-halfties} yields $\Reg(\pi^\chi_\mu)\ge c\,\Reg(\piEM)$.
Since $\Reg(\pi^\chi_\beta)$ is decreasing in $\beta$ on Regime II, for all $\beta\in(\epsilon\delta,\mu]$ we have
\begin{align}
\Reg(\pi^\chi_\beta)\ \ge\ \Reg(\pi^\chi_\mu)\ \ge\ c\,\Reg(\piEM).
\end{align}

\medskip
\noindent
\textbf{Regime III ($\beta>\mu$).}
From \eqref{eq:Reg-regime3-affine-halfties},
\begin{align}
\Reg(\pi^\chi_\beta)=v(C)-\R^{\rstar}(\piref)-\frac{S(\alpha,\epsilon,\delta)}{\beta},
\end{align}
so $\Reg(\pi^\chi_\beta)$ is affine in $1/\beta$ over $\beta>\mu$. Hence the minimum over $\beta>\mu$
is attained at an endpoint:
\begin{align}
\label{eq:regime3-endpoints-halfties}
\inf_{\beta>\mu}\Reg(\pi^\chi_\beta)=\min\Big\{\Reg(\pi^\chi_\mu),\ \Reg(\piref)\Big\}.
\end{align}
The first endpoint $\Reg(\pi^\chi_\mu)$ is already lower bounded by $c\,\Reg(\piEM)$ from Regime II.
It remains to lower bound the second endpoint, which is $\Reg(\piref)$. Under the half-ties convention, by symmetry of i.i.d.\ draws,
\begin{align}
\label{eq:R-ref-halfties}
\R^{\rstar}(\piref)
=\pp_{y,y'\sim\piref}\left(\rstar(y)>\rstar(y')\right)
+\frac12\pp_{y,y'\sim\piref}\left(\rstar(y)=\rstar(y')\right)
=\frac12.
\end{align}
Therefore,
\begin{align}
\Reg(\piref)=v(C)-\R^{\rstar}(\piref)
=\left(1-\frac{\alpha-\epsilon}{2}\right)-\frac12
=\frac{1-\alpha+\epsilon}{2} \ge \frac{1}{4},
\end{align}
where the last inequality follows from $\alpha\le \tfrac12$ by construction.
On the other hand, using $\epsilon/\alpha\le \frac{1}{2k^2}$ again,
\begin{align}
\Reg(\piEM)=\left(1-\frac{\alpha}{2}\right)\frac{\epsilon}{\alpha}\le \frac{\epsilon}{\alpha}\le \frac{1}{2k^2}\le \frac{1}{4c},
\end{align}
where the last inequality uses $k^2\ge 2c$. Hence $\Reg(\piref)\ge \frac14 \ge c\,\Reg(\piEM)$.
Plugging into \eqref{eq:regime3-endpoints-halfties} gives $\Reg(\pi^\chi_\beta)\ge c\,\Reg(\piEM)$ for all $\beta>\mu$.

\medskip
\noindent
\textbf{Concluding the separation.}
Combining the three regimes yields
\begin{align}   \inf_{\beta > 0} R^{\rstar}(\pistar) - R^{\rstar}(\pi^\chi_\beta) \geq c \cdot \left(\inf_{M > 0} R^{\rstar}(\pistar) - R^{\rstar}(\piEM)\right),
\end{align}
as claimed.
\end{proof}

\section{Proofs under win-rate against an arbitrary comparison policy}
\label{sec:q-winrate}
In this section, we provide proofs for the upper bounds on the performance of BoN and $\cE_M$-regularized BoN when win-rate is compared against an arbitrary measure $q$. This section is organized as follows:
\begin{enumerate}
    \item \Cref{sec:q-transfer-lemma} provides a modified win-rate transfer lemma introduced in \Cref{sec:supporting-lemmas} under win-rate against $\piref$, under a bounded density ratio assumption between $q$ and $\piref$. 
    \item \Cref{sec:bon-upper-bound-q-pf} provides a general upper bound on the regret of BoN in terms of the $\cE_M$-divergence introduced in \Cref{sec:pw-error}, with a new baseline comparator policy for win-rate.
    \item In \Cref{sec:topk-wr-q-pf}, we give a general upper bound on the regret of our implementation of the $\cE_M$-regularized BoN policy under the win-rate against an arbitrary $q$. 
\end{enumerate}
Before proceeding to the proofs, recall from \Cref{sec:win_rate_arbit} that we defined the \emph{win-rate against $q$} by
\begin{align}
\R^{r}_{q}(\pi)
\;\defeq\;
\ee_{y\sim \pi,\;y'\sim q}\!\left[\phi_{r}(y,y')\right]
\quad \text{and} \quad
\phi_{r}(y,y')
\defeq
\indic\{r(y)>r(y')\}+\tfrac12\,\indic\{r(y)=r(y')\}.
\end{align}
This definition generalizes the win-rate definition against $\piref$ in \Cref{eq:win-rate-definition}, where we now draw $y' \sim q$ to compare against a random sample from our policy. We will assume throughout this section that $q \ll \piref$, and
\begin{align}
\label{eq:q-domination}
w_q(y)\defeq \frac{dq}{d\piref}(y)\le L \qquad \piref\text{-a.s.}
\end{align}
for some constant $L\ge 1$. We continue to measure pairwise reward-model error with respect to $\piref$ via
\begin{align}
\epspw(\rhat)
\;=\;
\ee_{y,y'\sim \piref}\!\left[\left|\phi_{\rhat}(y,y')-\phi_{\rstar}(y,y')\right|\right].
\end{align}

\subsection{Transfer lemma for win-rate against $q$}
\label{sec:q-transfer-lemma}

\begin{lemma}[Win-rate transfer against $q$]
\label{lem:win-transfer-em-q}
Fix $\piref$, $q$ satisfying \eqref{eq:q-domination}, a true reward $\rstar$, and a learned reward model $\rhat$.
For any policy $\pi\ll\piref$ with density ratio $w_\pi\defeq d\pi/d\piref$ and any $M\ge 1$,
\begin{align}
\label{eq:transfer-q}
\left|\R^{\rstar}_{q}(\pi)-\R^{\rhat}_{q}(\pi)\right|
\;\le\;
L\left(M\cdot \epspw(\rhat) + \cE_M(\pi\|\piref)\right),
\end{align}
where $\cE_M(\pi\|\piref)\defeq \ee_{y\sim\piref}[(w_\pi(y)-M)_+]$.
\end{lemma}

\begin{proof}
The proof follows the same strategy as \Cref{lem:win-transfer-em}. We first define $\ell(y,y')\defeq |\phi_{\rstar}(y,y')-\phi_{\rhat}(y,y')|\in[0,1]$.
Then
\begin{align}
\left|\R^{\rstar}_{q}(\pi)-\R^{\rhat}_{q}(\pi)\right|
&=
\left|\ee_{y\sim\pi,\;y'\sim q}\!\left[\phi_{\rstar}(y,y')-\phi_{\rhat}(y,y')\right]\right|\\
&\le
\ee_{y\sim\pi,\;y'\sim q}\!\left[\ell(y,y')\right].
\end{align}
We can then perform a change of measure using $w_\pi=d\pi/d\piref$ and $w_q=dq/d\piref$:
\begin{align}
\ee_{y\sim\pi,\;y'\sim q}\!\left[\ell(y,y')\right]
=
\ee_{y,y'\sim\piref}\!\left[w_\pi(y)\,w_q(y')\,\ell(y,y')\right]
\le
L\cdot \ee_{y,y'\sim\piref}\!\left[w_\pi(y)\,\ell(y,y')\right],
\end{align}
where the inequality uses $w_q\le L$ $\piref$-a.s.
Now decompose $w_\pi(y)=(w_\pi(y)\wedge M) + (w_\pi(y)-M)_+$ and use $\ell\in[0,1]$:
\begin{align}
\ee_{y,y'\sim\piref}\!\left[w_\pi(y)\,\ell(y,y')\right]
&=
\ee_{y,y'\sim\piref}\!\left[(w_\pi(y)\wedge M)\,\ell(y,y')\right]
+
\ee_{y,y'\sim\piref}\!\left[(w_\pi(y)-M)_+\,\ell(y,y')\right]\\
&\le
M\cdot \ee_{y,y'\sim\piref}\!\left[\ell(y,y')\right]
+
\ee_{y\sim\piref}\!\left[(w_\pi(y)-M)_+\right]\\
&=
M\cdot \epspw(\rhat)+\cE_M(\pi\|\piref).
\end{align}
Multiplying by $L$ yields \eqref{eq:transfer-q}.
\end{proof}

\subsection{Proof of \Cref{thm:bon-win-rate-upper-bound-q-body}}
\label{sec:bon-upper-bound-q-pf}

Next, we provide a proof of the upper bound on the regret of the BoN policy when win-rate is comparing against an arbitrary measure $q$. Optimizing the following bound in the same way that we did for the $\piref$-win-rate case yields the same result as \Cref{thm:bon-win-rate-upper-bound} up to constant factors. 

\begin{theorem}
\label{thm:bon-win-rate-upper-bound-q}
Under the pointwise bound on the density ratio of $q$ in \eqref{eq:q-domination}, for any $N\geq1$, reference model $\piref$, comparator $\pistar\ll\piref$, ground truth reward $\rstar$, and learned reward model $\rhat$, it holds that
\begin{align}
\R_q^{\rstar}(\pistar)-\R_q^{\rstar}(\pihatbon)
\lesssim
L \cdot N \cdot \epspw(\rhat) \cdot \log\left( \nicefrac{1}{\epspw(\rhat)} \right) + \cE_{N / \log(1/\epspw(\rhat))}\left( \pistar \| \piref \right).
\end{align}
\end{theorem}

\begin{proof}
We follow the proof of \Cref{thm:bon-win-rate-upper-bound} detailed in \Cref{sec:bon-upper-bound-pf}, replacing $\R^{(\cdot)}$ with $\R^{(\cdot)}_q$.
In particular, for any $M>0$, the same regret decomposition holds:
\begin{align}
\label{eq:bon-regret-decomp-q}
\R_q^{\rstar}(\pistar)-\R_q^{\rstar}(\pihatbon)
=
\underbrace{\R_q^{\rstar}(\pistar)-\R_q^{\rhat}(\pistar_M)}_{\mathrm{(T1)}}
+
\underbrace{\R_q^{\rhat}(\pistar_M)-\R_q^{\rhat}(\pihatbon)}_{\mathrm{(T2)}}
+
\underbrace{\R_q^{\rhat}(\pihatbon)-\R_q^{\rstar}(\pihatbon)}_{\mathrm{(T3)}},
\end{align}
where $\pistar_M$ is the same $M$-capped TV projection defined in \eqref{eq:tv-proj-def}.
To bound $\mathrm{(T1)}$, we decompose as in \Cref{sec:bon-upper-bound-pf}:
\begin{align}
\R_q^{\rstar}(\pistar)-\R_q^{\rhat}(\pistar_M)
=
\R_q^{\rstar}(\pistar)-\R_q^{\rstar}(\pistar_M)
+
\R_q^{\rstar}(\pistar_M)-\R_q^{\rhat}(\pistar_M).
\end{align}
The first difference is controlled exactly as before: for fixed $(r,q)$, the win-rate functional is bounded and therefore $1$-Lipschitz in total variation. Thus,
\begin{align}
\R_q^{\rstar}(\pistar)-\R_q^{\rstar}(\pistar_M)\le \TV(\pistar,\pistar_M)=\cE_M(\pistar\|\piref),
\end{align}
where the equality uses the same arguments from \citet{block2023sample} as in the proof of \Cref{thm:bon-win-rate-upper-bound}.
For the second difference, apply \Cref{lem:win-transfer-em-q} with $\pi=\pistar_M$; since $\nicefrac{d\pistar_M}{d\piref}\le M$,
we have $\cE_M(\pistar_M\|\piref)=0$, and thus
\begin{align}
\R_q^{\rstar}(\pistar_M)-\R_q^{\rhat}(\pistar_M)\le L\cdot M\,\epspw(\rhat).
\end{align}
Therefore,
\begin{align}
\label{eq:T1-q}
\mathrm{(T1)}\le \cE_M(\pistar\|\piref)+L\,M\,\epspw(\rhat).
\end{align}
Next, in order to bound $\mathrm{(T2)}$, we introduce the same approximate rejection sampling selection rule $\pistar_{M,\mathrm{rej}}$ as in \Cref{sec:bon-upper-bound-pf} and decompose
\begin{align}
\R_q^{\rhat}(\pistar_M)-\R_q^{\rhat}(\pihatbon)
=
\R_q^{\rhat}(\pistar_M)-\R_q^{\rhat}(\pistar_{M,\mathrm{rej}})
+
\R_q^{\rhat}(\pistar_{M,\mathrm{rej}})-\R_q^{\rhat}(\pihatbon).
\end{align}
The first difference is controlled exactly as before by the approximate rejection sampling analysis of \citet{block2023sample}, using only boundedness of win-rate:
\begin{align}
\R_q^{\rhat}(\pistar_M)-\R_q^{\rhat}(\pistar_{M,\mathrm{rej}})
\le \TV(\pistar_M,\pistar_{M,\mathrm{rej}})\le \tfrac12 \exp(-N/M),
\end{align}
since $\cE_M(\pistar_M\|\piref)=0$. For the second difference, conditional on any fixed batch $\YhatN=(y_1,\dots,y_N)$, BoN outputs an index attaining
$\max_i \rhat(y_i)$. We define 
\begin{align}
g_{r,q}(r(y))\defeq \pp_{y'\sim q}\left(r(y')<r(y)\right)+\tfrac12\,\pp_{y'\sim q}\left(r(y')=r(y)\right)\in[0,1],
\end{align}
so that for any policy $\pi$, $\R^{\rhat}_{q}(\pi)=\ee_{y\sim\pi}[g_{\rhat, q}(\rhat(y))]$ by conditioning on $y$ exactly as in \Cref{lem:cdf-identity-pf}. From this, it is clear that $g_{\rhat,q}$ is nondecreasing in $\rhat(y)$. Therefore, for any $y\in\YhatN$ we have $g_{\rhat,q}(\rhat(y))\le \max_{i\in[N]} g_{\rhat,q}(\rhat(y_i))$, and the same coupling
argument as in the proof of \Cref{thm:bon-win-rate-upper-bound} gives
\begin{align}
\R^{\rhat}_q(\pistar_{M,\mathrm{rej}})-\R^{\rhat}_q(\pihatbon)\le 0.
\end{align}
Combining the bounds on the two differences ensures
\begin{align}
\label{eq:T2-q}
\mathrm{(T2)}\le \tfrac12 \exp(-N/M).
\end{align}
To bound $\mathrm{(T3)}$, recall that in \Cref{sec:bon-upper-bound-pf} we showed that BoN satisfies $\nicefrac{d\pihatbon}{d\piref}\le N$ pointwise, so $\cE_N(\pihatbon\|\piref)=0$.
Applying \Cref{lem:win-transfer-em-q} with $\pi=\pihatbon$ and truncation $M=N$ yields
\begin{align}
\label{eq:T3-q}
\mathrm{(T3)}
=
\R_q^{\rhat}(\pihatbon)-\R_q^{\rstar}(\pihatbon)
\le L\cdot N\,\epspw(\rhat).
\end{align}
Combining \eqref{eq:T1-q}, \eqref{eq:T2-q}, and \eqref{eq:T3-q} with \eqref{eq:bon-regret-decomp-q} gives
\begin{align}
\R_q^{\rstar}(\pistar)-\R_q^{\rstar}(\pihatbon)
\;\lesssim\;
\cE_M(\pistar\|\piref) + L(M+N)\,\epspw(\rhat) + \exp(-N/M).
\end{align}
Setting $M = N / \log\left( \nicefrac{1}{\epw(\rhat)} \right)$ and observing that for such $M$, it holds that $\exp(-N/M) = \epw(\rhat)$, we obtain the desired result.
\end{proof}

\subsection{Proof of \Cref{thm:topk-winrate-upper-bound-q-body}}
\label{sec:topk-wr-q-pf}

Finally, we provide an upper bound on the regret of the $\cE_M$-regularized BoN policy when win-rate is comparing against an arbitrary measure $q$. We again work under the assumption that the density ratio of $\piEMhat$ is bounded pointwise \eqref{eq:q-domination}. Note that while the following bound is qualitatively similar to the result from \Cref{thm:bon-win-rate-upper-bound}, we obtain a slower rate in $N$ while staying minimax optimal.

\begin{theorem}
\label{thm:topk-winrate-upper-bound-q}
For any $M, N \geq 1$, reference model $\piref$, comparator $\pistar\ll\piref$, ground truth reward $\rstar$, and learned reward model $\rhat$, it holds that
\begin{align}\label{eq:EM-reg-winrate-ub-q}
    R^{\rstar}_q(\pistar) - R^{\rstar}_q(\piEMhat) \lesssim L\cdot\cE_M(\pistar\|\piref)+L\cdot M\cdot\epw(\rhat)+\sqrt{\frac{M}{N}}.
\end{align}
\end{theorem}

\begin{proof}
Fix $M\ge 1$ and let $k\defeq\lceil N/M\rceil$.  Recall that $\piEMhat$ is the (unconditional) law of the random output
$y_{\mathrm{out}}$ produced by: sample $y_1,\dots,y_N\iid\piref$ and i.i.d.\ $V_1,\dots,V_N\sim{\rm Unif}[0,1]$
independent; let $I$ index the top-$k$ samples under the lexicographic order $(\rhat(y_i),V_i)$; output $y_J$
with $J$ uniform on $I$ conditional on $(\YhatN,V_{1:N})$. As in the proof of \Cref{thm:topk-winrate-upper-bound} with $\piref$-win-rate, we consider the same regret decomposition:
\begin{align}
\R^{\rstar}_{q}(\pistar)-\R^{\rstar}_{q}(\piEMhat)
&=
\underbrace{\left(\R^{\rstar}_{q}(\pistar)-\R^{\rhat}_{q}(\pistar_M)\right)}_{\mathrm{(T1)}}
+
\underbrace{\left(\R^{\rhat}_{q}(\pistar_M)-\R^{\rhat}_{q}(\piEMhat)\right)}_{\mathrm{(T2)}}
+
\underbrace{\left(\R^{\rhat}_{q}(\piEMhat)-\R^{\rstar}_{q}(\piEMhat)\right)}_{\mathrm{(T3)}}.
\end{align}
Here, $\pistar_M$ is defined exactly as before: it minimizes total variation distance to $\pistar$ over all $\pi\ll\piref$ such that
$\frac{d\pi}{d\piref}\le M$ $\piref$-a.s. To bound $\mathrm{(T1)}$, we decompose
\begin{align}
\R^{\rstar}_{q}(\pistar)-\R^{\rhat}_{q}(\pistar_M)
=
\R^{\rstar}_{q}(\pistar)-\R^{\rhat}_{q}(\pistar)
+
\R^{\rhat}_{q}(\pistar)-\R^{\rhat}_{q}(\pistar_M).
\end{align}
For the first difference, we apply \Cref{lem:win-transfer-em-q} to obtain
\begin{align}
\R^{\rstar}_{q}(\pistar)-\R^{\rhat}_{q}(\pistar)
\;\le\;
L\left(M\,\epw(\rhat)+\cE_M(\pistar\|\piref)\right).
\end{align}
For the second difference, we observe that for fixed $(q,\rhat)$ the win-rate functional 
is bounded and therefore $1$-Lipschitz in total variation.
Thus,
\begin{align}
\R^{\rhat}_{q}(\pistar)-\R^{\rhat}_{q}(\pistar_M)\le \TV(\pistar,\pistar_M)=\cE_M(\pistar\|\piref),
\end{align}
where the last equality uses the same approximate rejection sampling arguments from \citet{block2023sample} as before.
Combining,
\begin{align}
\label{eq:T1-q-topk}
\mathrm{(T1)}\;\lesssim\;L\,\cE_M(\pistar\|\piref)\;+\;L\,M\,\epw(\rhat).
\end{align}
For $\mathrm{(T2)}$, as before, we introduce the ideal top-$1/M$ threshold policy
\begin{align}
\pi_M(\cdot)\defeq \piref\left(\cdot \mid \rhat(\cdot)\ge \lambda_M\right),
\end{align}
so that $\pp_{y\sim\piref}\left(\rhat(y)\ge \lambda_M\right)=\frac1M$, and we decompose as
\begin{align}
\R^{\rhat}_{q}(\pistar_M)-\R^{\rhat}_{q}(\piEMhat)
=
\R^{\rhat}_{q}(\pistar_M)-\R^{\rhat}_{q}(\pi_M)
+
\R^{\rhat}_{q}(\pi_M)-\R^{\rhat}_{q}(\piEMhat).
\end{align}
For the first difference, as in the proof of \Cref{thm:bon-win-rate-upper-bound-q}, we define 
\begin{align}
g_{r,q}(r(y))\defeq \pp_{y'\sim q}\left(r(y) > r(y')\right)+\tfrac12\,\pp_{y'\sim q}\left(r(y)=r(y')\right)\in[0,1],
\end{align}
so that for any policy $\pi$, $\R^{\rhat}_{q}(\pi)=\ee_{y\sim\pi}[g_{\rhat, q}(\rhat(y))]$ by conditioning on $y$ exactly as in \Cref{lem:cdf-identity-pf}. From this, it is clear that $g_{\rhat,q}$ is nondecreasing in $\rhat(y)$, so $\pi_M$ (output by \Cref{prop:KKT-top-quantile} with $s=\rhat$) maximizes $\R^{\rhat}_{q}(\pi)$ over the class of policies with density ratio bounded by $M$. Recalling that $\pistar_M$ is also feasible subject to this constraint, it follows that
$\R^{\rhat}_{q}(\pistar_M)\le \R^{\rhat}_{q}(\pi_M)$, and hence
\begin{align}
\R^{\rhat}_{q}(\pistar_M)-\R^{\rhat}_{q}(\pi_M)\le 0.
\end{align}
It remains to bound $\R^{\rhat}_{q}(\pi_M)-\R^{\rhat}_{q}(\piEMhat)$.
As in \Cref{app:topk-wr-pf}, we let $F_{\rhat}$ and $F_{\rhat}(\cdot^-)$ denote the cumulative density functions of $\rhat(y)$ with $y \sim \piref$ as in \Cref{lem:cdf-identity-pf}, and given an auxiliary seed $V\sim{\rm Unif}[0,1]$ independent of $y$, we define
\begin{align}
\tilde u(y;V)\defeq F_{\rhat}(\rhat(y)^-)+V\left(F_{\rhat}(\rhat(y))-F_{\rhat}(\rhat(y)^-)\right)\in[0,1].
\end{align}
By the same uniformization argument as in \Cref{lem:cdf-identity-pf}, $\tilde u(y;V)\sim{\rm Unif}[0,1]$, and moreover
\begin{align}
F_{\rhat}^{-1}(\tilde u(y;V))=\rhat(y)
\end{align}
$\piref$-a.s., where $F_{\rhat}^{-1}$ is the generalized inverse. As in the proof of \Cref{thm:topk-winrate-upper-bound}, ranking by
$(\rhat(y_i),V_i)$ is equivalent to ranking by $\tilde u_i\defeq \tilde u(y_i;V_i)$. We let $u_1,\dots,u_N\iid{\rm Unif}[0,1]$ denote these randomized ranks, and write $u_{(1)}\le\cdots\le u_{(N)}$. Now we set $t\defeq 1-\frac1M$ and define the empirical cutoff $\tau\defeq u_{(N-k)}$ (so the selected set corresponds to the top $k$ ranks). We also define  the bounded monotone function
\begin{align}
h(u)\defeq g_{\rhat,q}\!\left(F_{\rhat}^{-1}(u)\right)\in[0,1],
\end{align}
so that $\R^{\rhat}_{q}(\pi)=\ee_{y\sim\pi}[h(\tilde u(y;V))]$.
Under $\pi_M$, we have $\tilde u(y;V)\sim{\rm Unif}[t,1]$, and hence
\begin{align}
\R^{\rhat}_{q}(\pi_M)=\mu(t),
\end{align}
where $\mu(s)\defeq \ee[h(u)\mid u\ge s]$ and $u\sim{\rm Unif}[0,1]$.
Under $\piEMhat$, the output rank $u_{\rm out}$ is uniform over the top $k$ ranks. Following the same logic as \Cref{lem:topk-exact-winrate-hatr}, a standard order-statistic property implies that conditional on $\tau$, the multiset of the $k$ ranks above $\tau$
is i.i.d.\ ${\rm Unif}[\tau,1]$; therefore $u_{\rm out}\mid \tau \sim {\rm Unif}[\tau,1]$ and
\begin{align}
\R^{\rhat}_{q}(\piEMhat)=\ee\left[\mu(\tau)\right],
\end{align}
where the expectation is over the randomness in $\tau$. At this point we have reduced the gap to
\begin{align}
\R^{\rhat}_{q}(\pi_M)-\R^{\rhat}_{q}(\piEMhat)
=
\mu(t)-\ee[\mu(\tau)].
\end{align}
To upper bound this quantity, we note that it is immediate that
\begin{align}
\mu(t)-\ee[\mu(\tau)] \leq \ee\left[|\mu(t)-\mu(\tau)|\right].
\end{align}
Recall $\mu(s)=\ee[h(U)\mid U\ge s]$ with $U\sim{\rm Unif}[0,1]$ and $h\in[0,1]$.
Equivalently, if we write $P_s\defeq{\rm Unif}[s,1]$, then $\mu(s)=\ee_{P_s}[h]$.
Hence for any $s,s'\in[0,1)$,
\begin{align}
|\mu(s)-\mu(s')|
= \left|\ee_{P_s}[h]-\ee_{P_{s'}}[h]\right|
\le \TV(P_s,P_{s'}).
\end{align}
We next show how to bound this total variation distance. Assume wlog $s\le s'$. The corresponding densities are
\begin{align}
f_s(u)=\frac{1}{1-s}\,\indic\{u\in[s,1]\},
\qquad
f_{s'}(u)=\frac{1}{1-s'}\,\indic\{u\in[s',1]\},
\end{align}
so
\begin{align}
\TV(P_s,P_{s'})
&=\frac12\int_0^1 |f_s(u)-f_{s'}(u)|\,du \\
&=\frac12\left(\int_s^{s'}\frac{1}{1-s}\,du
\;+\;\int_{s'}^{1}\left(\frac{1}{1-s'}-\frac{1}{1-s}\right)\,du\right) \\
&=\frac12\left(\frac{s'-s}{1-s}
\;+\;(1-s')\left(\frac{1}{1-s'}-\frac{1}{1-s}\right)\right)
=\frac{s'-s}{1-s}.
\end{align}
Therefore, for all $s,s'\in[0,1)$,
\begin{align}
|\mu(s)-\mu(s')|
\le \frac{|s-s'|}{1-\min\{s,s'\}}.
\end{align}
Applying this with $s=t$ and $s'=\tau$ yields
\begin{align}
|\mu(t)-\mu(\tau)|
\le \frac{|t-\tau|}{1-\min\{t,\tau\}}
\le \frac{|t-\tau|}{1-t}
= M\,|t-\tau|,
\end{align}
since $t=1-\frac1M$ and $\min\{t,\tau\}\le t$.
Hence
\begin{align}
\R^{\rhat}_q(\pi_M)-\R^{\rhat}_q(\hat\pi_M)
=\mu(t)-\ee[\mu(\tau)]
\le \ee\left[|\mu(t)-\mu(\tau)|\right]
\le M\,\ee[|t-\tau|].
\end{align}
Hence
\begin{align}
\label{eq:lipschitz-bound-q}
\R^{\rhat}_{q}(\pi_M)-\R^{\rhat}_{q}(\piEMhat)
\le M\ee\left[|t-\tau|\right],
\end{align}
and it remains to bound $\ee\left[|t-\tau|\right]$. Since $u_1,\dots,u_N\iid{\rm Unif}[0,1]$, the cutoff $\tau=u_{(N-k)}$ has the Beta law
$\tau\sim{\rm Beta}(N-k,\;k+1)$.
In particular,
\begin{align}
\ee[\tau]=\frac{N-k}{N+1}
\qquad\text{and}\qquad
\Var(\tau)=\frac{(N-k)(k+1)}{(N+1)^2(N+2)}\;\le\;\frac{k+1}{(N+1)^2}.
\end{align}
We can decompose using the triangle inequality as
\begin{align}
\ee[|t-\tau|]\;\le\;|t-\ee[\tau]|+\ee[|\tau-\ee[\tau]|],
\end{align}
and using Cauchy-Schwarz on the second term, we obtain
\begin{align}
|t-\ee[\tau]|+\ee[|\tau-\ee[\tau]|]
\;\le\;|t-\ee[\tau]|+\sqrt{\Var(\tau)}.
\end{align}
Since $t=1-\tfrac1M$ and $k=\lceil N/M\rceil$, one can easily verify that $|\ee[\tau]-t|\lesssim \tfrac1N$ (due to rounding) and
$\sqrt{\Var(\tau)}\lesssim \tfrac{1}{\sqrt{MN}}$, yielding
\begin{align}
\ee[|t-\tau|]
\;\lesssim\;\frac1N+\frac{1}{\sqrt{MN}}.
\end{align}
Combining with \eqref{eq:lipschitz-bound-q} gives
\begin{align}
\R^{\rhat}_q(\pi_M)-\R^{\rhat}_q(\hat\pi_M)
\;\lesssim\;
M\left(\frac1N+\frac{1}{\sqrt{MN}}\right)
\;=\;
\frac{M}{N}+\sqrt{\frac{M}{N}},
\end{align}
and hence we can conclude that 
\begin{align}
\label{eq:T2-q-topk}
\mathrm{(T2)}=\R^{\rhat}_{q}(\pistar_M)-\R^{\rhat}_{q}(\piEMhat)
\;\lesssim\;
\frac{M}{N}+\sqrt{\frac{M}{N}},
\end{align}
where the second term dominates for large $N$. For $\mathrm{(T3)}$, as in the proof of \Cref{thm:topk-winrate-upper-bound}, we have $\cE_M(\piEMhat\|\piref)=0$ since $w_{\piEMhat}\le N/k\le M$ $\piref$-a.s.
Applying \Cref{lem:win-transfer-em-q} to $\pi=\piEMhat$ under~\eqref{eq:q-domination} yields
\begin{align}
\label{eq:T3-q-topk}
\mathrm{(T3)}
=\R^{\rhat}_{q}(\piEMhat)-\R^{\rstar}_{q}(\piEMhat)
\;\le\;
L\left(M\,\epw(\rhat)+\cE_M(\piEMhat\|\piref)\right)
=
L\,M\,\epw(\rhat).
\end{align}
Finally, combining \eqref{eq:T1-q-topk}, \eqref{eq:T2-q-topk}, and \eqref{eq:T3-q-topk} gives
\begin{align}
R^{\rstar}_q(\pistar) - R^{\rstar}_q(\piEMhat) \lesssim L\cdot\cE_M(\pistar\|\piref)+L\cdot M\cdot\epw(\rhat)+\sqrt{\frac{M}{N}}.
\end{align}
\end{proof}
While we do not provide a computational lower bound proving that the $\sqrt{M/N}$ estimation term is tight in full generality,  this slower rate can be viewed as an artifact of the general-$q$ setting. The key obstruction is that for general $q$ the win-rate becomes a \emph{generic tail functional} of the randomized $\rhat$-rank rather than an explicitly computable linear statistic as in the proof of \Cref{thm:bon-win-rate-upper-bound} in \Cref{sec:bon-upper-bound-pf}. Concretely, we write
\begin{align}
\R^{\rhat}_q(\pi)=\ee_{y\sim\pi,\,V}\!\big[h(\tilde u(y;V))\big],
\qquad
h(u)\defeq g_{\rhat,q}\!\big(F_{\rhat}^{-1}(u)\big)\in[0,1],
\end{align}
where $h$ is merely bounded (and monotone) under~\eqref{eq:q-domination}. When $q=\piref$, the corresponding transform is linear (indeed $h(u)=u$), which is why the empirical top-$1/M$ policy admits an exact order-statistics calculation and yields a $1/N$ term. For general $q$, however, the shape of $h$ near the fixed cutoff $t=1-\tfrac1M$ is uncontrolled: $h$ can place almost all of its variation in an $O(1/M)$ neighborhood of $t$. In such worst cases, small fluctuations of the empirical cutoff $\tau$ around $t$ can translate into a sensitivity bound of the form
\begin{align}
|\mu(t)-\mu(\tau)| \;\lesssim\; M\,|t-\tau|,
\end{align}
which holds for all bounded $h$ and captures the fact that changing the threshold by $|t-\tau|$ perturbs a tail of mass $1/M$. Since $\tau$ fluctuates around $t$ at scale $\ee|t-\tau|\asymp 1/\sqrt{MN}$ for $k\approx N/M$, this leads to an estimation error on the order of $\sqrt{M/N}$ without additional regularity assumptions on $q$ (equivalently, on $h$) beyond~\eqref{eq:q-domination}.

\end{document}